\documentclass[twoside]{article}

\usepackage[accepted]{aistats2019}

\usepackage[round]{natbib}

\newif\ifbullets
\bulletsfalse

\newif\ifcamera
\cameratrue

\usepackage{url}            %
\usepackage{booktabs}       %
\usepackage{amsfonts}       %
\usepackage{nicefrac}       %
\usepackage{microtype}      %

\usepackage{xcolor}
\usepackage{bbm}
\usepackage{amsmath}
\usepackage{amsthm}
\usepackage{amssymb}
\usepackage[group-separator={,}]{siunitx}
\usepackage{graphicx}
\usepackage{bm}
\usepackage{subfigure}
\usepackage{capt-of}
\usepackage{wrapfig}

\newcommand{\Nystrom}{Nystr\"{o}m }
\newcommand{\NystromNS}{Nystr\"{o}m} %
\newcommand{\NystromCaps}{NYSTR\"{O}M }

\newcommand{\by}{\bar{y}}
\newcommand{\bS}{\bar{S}}
\newcommand{\bT}{\bar{T}}
\newcommand{\hx}{\hat{x}}

\newcommand{\hK}{\hat{K}}

\newcommand{\eps}{\epsilon}

\newcommand{\tK}{\tilde{K}}
\newcommand{\tZ}{\tilde{Z}}

\newcommand{\tz}{\tilde{z}}
\newcommand{\tsigma}{\tilde{\sigma}}

\newcommand{\hcR}{\widehat{\cR}}
\newcommand{\id}{I}
\newcommand{\sq}{\sqrt{2}}

\DeclareMathOperator*{\argmax}{arg\,max}
\DeclareMathOperator*{\rank}{rank}

\newcommand{\eg}{e.g.}

\def\ddefloop#1{\ifx\ddefloop#1\else\ddef{#1}\expandafter\ddefloop\fi}
\def\ddef#1{\expandafter\def\csname bb#1\endcsname{\ensuremath{\mathbb{#1}}}}
\ddefloop ABCDEFGHIJKLMNOPQRSTUVWXYZ\ddefloop
\def\ddef#1{\expandafter\def\csname c#1\endcsname{\ensuremath{\mathcal{#1}}}}
\ddefloop ABCDEFGHIJKLMNOPQRSTUVWXYZ\ddefloop
\newcommand\norm[1]{\|#1\|}

\newcommand\Dotp[1]{\left\langle#1\right\rangle}
\newcommand{\RR}{\ensuremath{\bbR}} %

\newcommand{\E}{\mathbb{E}} %
\newcommand{\Prob}{\mathbb{P}}
\newcommand{\ProbOpr}[1]{\mathbb{#1}}
\newcommand{\expect}[2]{%
\ifthenelse{\equal{#2}{}}{\ProbOpr{E}_{#1}}
{\ifthenelse{\equal{#1}{}}{\ProbOpr{E}\left[#2\right]}{\ProbOpr{E}_{#1}\left[#2\right]}}} %
\newcommand{\var}[2]{%
\ifthenelse{\equal{#2}{}}{\ProbOpr{VAR}_{#1}}
{\ifthenelse{\equal{#1}{}}{\ProbOpr{VAR}\left[#2\right]}{\ProbOpr{VAR}_{#1}\left[#2\right]}}}

\newtheorem{theorem}{Theorem}
\newtheorem{definition}{Definition}
\newtheorem{lemma}[theorem]{Lemma}

\newtheorem{proposition}[theorem]{Proposition}

\newtheorem{corollary}{Corollary}[theorem]

\newtheorem{innercustomthm}{Theorem}
\newenvironment{customthm}[1]
{\renewcommand\theinnercustomthm{#1}\innercustomthm}
{\endinnercustomthm}

\newtheorem{innercustomprop}{Proposition}
\newenvironment{customprop}[1]
{\renewcommand\theinnercustomprop{#1}\innercustomprop}
{\endinnercustomprop}

\newcommand{\defeq}{:=}

\providecommand{\tr}{\mathop{\rm tr}}
\newcommand\numberthis{\addtocounter{equation}{1}\tag{\theequation}}

\ifcamera
\newcommand{\vsp}{\vspace{-0.06in}}
\newcommand{\vfigsp}{\vspace{-0.07in}}
\newcommand{\vflistsp}{}
\else
\newcommand{\vsp}{}
\newcommand{\vfigsp}{}
\newcommand{\vflistsp}{}
\fi

\newcommand\blfootnote[1]{%
	\begingroup
	\renewcommand\thefootnote{}\footnote{#1}%
	\addtocounter{footnote}{-1}%
	\endgroup
}
\begin{document}

\twocolumn[
	\aistatstitle{Low-Precision Random Fourier Features for Memory-Constrained Kernel Approximation}
	\aistatsauthor{Jian Zhang$^*$ \And Avner May$^*$ \And Tri Dao \And Christopher R\'e}
	\aistatsaddress{
		Stanford University
	}
]\blfootnote{$^*$Equal contribution.\vsp}

\vsp\vsp
\begin{abstract}
\vsp\vsp

We investigate how to train kernel approximation methods that generalize well under a memory budget.
Building on recent theoretical work, we define a measure of kernel approximation error which we find to be more predictive of the empirical generalization performance of kernel approximation methods than conventional metrics.
An important consequence of this definition is that a kernel approximation matrix must be high rank to attain close approximation.
Because storing a high-rank approximation is memory intensive,
we propose using a \emph{low-precision} quantization of random Fourier features (LP-RFFs) to build a high-rank approximation under a memory budget.
Theoretically, we show quantization has a negligible effect on generalization performance in important settings.
Empirically, we demonstrate across four benchmark datasets that LP-RFFs can match the performance of full-precision RFFs and the \Nystrom method, with 3x-10x and 50x-460x less memory, respectively.
\vsp

\end{abstract}
\section{\uppercase{Introduction}}
\label{sec:intro}
\vsp

Kernel methods are a powerful family of machine learning methods.
A key technique for scaling kernel methods is to construct feature representations whose inner products approximate the kernel function, and then learn a linear model with these features;
important examples of this technique include the \Nystrom method \citep{nystrom} and random Fourier features (RFFs) \citep{rahimi07random}.
Unfortunately, a large number of features are typically needed for attaining strong generalization performance with these methods on big datasets \citep{rahimi08kitchen,block16,may2017}. 
Thus, the memory required to store these features can become the training bottleneck for kernel approximation models.
In this paper we work to alleviate this memory bottleneck by optimizing the generalization performance for these methods under a fixed memory budget.

To gain insight into how to design more memory-efficient kernel approximation methods, we first investigate the generalization performance vs.\ memory utilization of \Nystrom and RFFs.
While prior work \citep{nysvsrff12} has shown that the \Nystrom method generalizes better than RFFs under the the same number of features, we demonstrate that the opposite is true under a memory budget.
Strikingly, we observe that $\num[group-separator={,}]{50000}$ standard RFFs can achieve the same heldout accuracy as $\num[group-separator={,}]{20000}$ \Nystrom features with 10x less memory on the TIMIT classification task.
Furthermore, this cannot be easily explained by the Frobenius or spectral norms of the kernel approximation error matrices of these methods,
even though these norms are the most common metrics for evaluating kernel approximation methods \citep{gittens16,qmc,sutherland15,yu16,dao17}; 
the above \Nystrom features attain 1.7x smaller Frobenius error and 17x smaller spectral error compared to the RFFs.
This observation suggests the need for a more refined measure of kernel approximation error---one which better aligns with generalization performance, 
and can thus better guide the design of new approximation methods.

\begin{table*}
	\vfigsp
	\vfigsp
	\caption{
		\ifcamera
		Memory utilization for kernel approximation methods. We consider data $x\in \RR^d$, kernel features $z(x)\in \RR^m$, mini-batch size $s$, \# of classes $c$ (for regression/binary classification $c=1$). We assume full-precision numbers are 32 bits. We measure a method's memory utilization as the sum of the three components in this table.
		\else
		The memory utilization for different kernel approximation methods. We consider data $x\in \RR^d$, kernel features $z(x)\in \RR^m$, mini-batch size $s$, number of classes $c$ (except for regression/binary classification, where $c=1$), and $b$ as the precision of $z(x)$. We assume full-precision numbers are stored in 32-bit floating point format. In this paper we measure a method's memory utilization as the sum of the three components in this table.
		\fi
	}
	\label{table:mem-usage}
	\centering
	\begin{tabular}{llll}
		\toprule
		Approximation Method  & Feature generation & Feature mini-batch & Model parameters \\
		\midrule
		\Nystrom                    & $32(md + m^2)$ & $32ms$  & $32mc$ \\
		RFFs                        & $32md$         & $32ms$  & $32mc$ \\
		Circulant RFFs              & $32m$          & $32ms$  & $32mc$ \\
		Low-precision RFFs, $b$ bits  (ours)   & $32m$          & $bms$   & $32mc$ \\
		\bottomrule
	\end{tabular}
\end{table*}

Building on recent theoretical work \citep{avron17}, we define a measure of approximation error which we find to be much more predictive of empirical generalization performance than the conventional metrics.
In particular, we extend \citeauthor{avron17}'s definition of \emph{$\Delta$-spectral approximation} to our definition of 
\emph{$(\Delta_1,\Delta_2)$-spectral approximation} by decoupling the two roles played by $\Delta$ in the original definition.\footnote{The original definition uses the same scalar $\Delta$ to upper and lower bound the approximate kernel matrix in terms of the exact kernel matrix in the semidefinite order.}
This decoupling reveals that $\Delta_1$ and $\Delta_2$ impact generalization differently, and can together much better
explain the relative generalization performance of \Nystrom and RFFs than the original $\Delta$, or the Frobenius or spectral errors.
This $(\Delta_1,\Delta_2)$ definition has an important consequence---in order for an approximate kernel matrix to be close to the exact kernel matrix, it is necessary for it to be \textit{high rank}.

Motivated by the above connection between rank and generalization performance, we propose using \textit{low-precision random Fourier features} (LP-RFFs) to attain a high-rank approximation under a memory budget.
Specifically, we store each random Fourier feature in a low-precision fixed-point representation, thus achieving a higher-rank approximation with more features in the same amount of space.
Theoretically, we show that when the quantization noise is much smaller than the regularization parameter, using low precision has negligible effect on the number of features required for the approximate kernel matrix to be a $(\Delta_1,\Delta_2)$-spectral approximation of the exact kernel matrix.
Empirically, we demonstrate across four benchmark datasets (TIMIT, YearPred, CovType, Census) that in the mini-batch training setting, LP-RFFs can match the performance of full-precision RFFs (FP-RFFs) as well as the \Nystrom method, with 3x-10x and 50x-460x less memory, respectively.
These results suggest that LP-RFFs could be an important tool going forward for scaling kernel methods to larger and more challenging tasks.

The rest of this paper is organized as follows: 
In Section \ref{sec:nys_vs_rff} we compare the performance of the \Nystrom method and RFFs in terms of their training memory footprint.
In Section \ref{sec:refine} we present a more refined measure of kernel approximation error to explain the relative performance of \Nystrom and RFFs.
We introduce the LP-RFF method and corresponding analysis in Section \ref{sec:lprff}, and present LP-RFF experiments in Section \ref{sec:experiments}.
We review related work in Section \ref{sec:relwork}, and conclude in Section \ref{sec:conclusion}.
\vsp

\section{\uppercase{\NystromCaps vs.\ RFFs: An empirical comparison}}
\label{sec:nys_vs_rff}
\vsp

\begin{figure*}[t]
	\centering
	\begin{small}
		\vfigsp
		\begin{tabular}{@{\hskip -0.0in}c@{\hskip -0.0in}c@{\hskip -0.0in}c@{\hskip -0.0in}c@{\hskip -0.0in}}
			\includegraphics[width=0.245\linewidth]{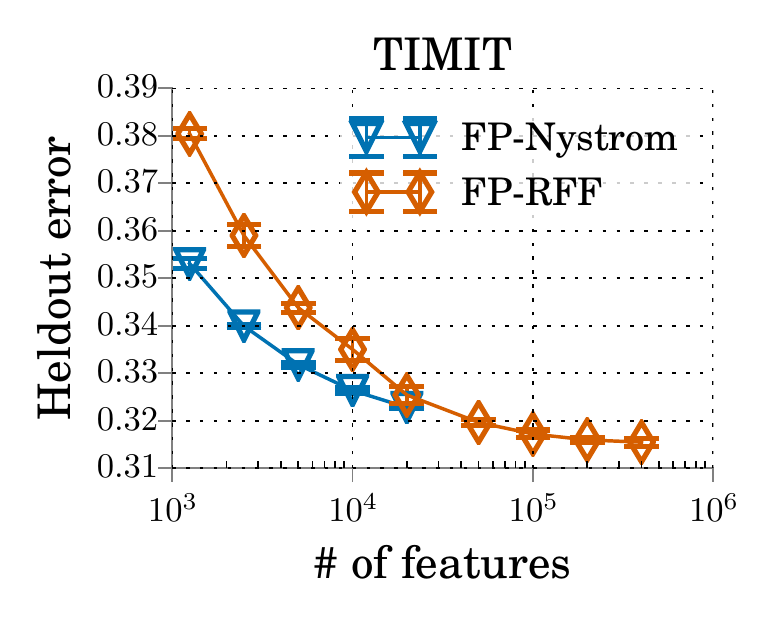} &
			\includegraphics[width=0.245\linewidth]{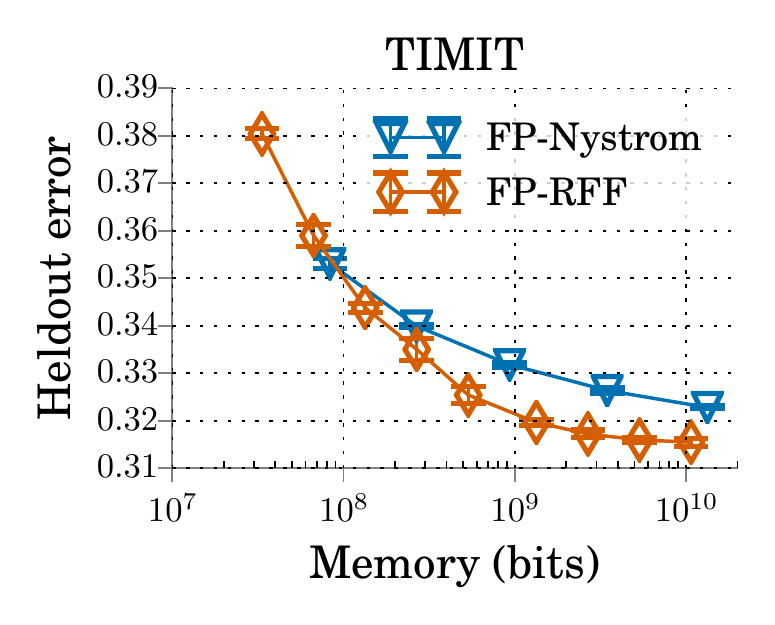} &        
			\includegraphics[width=0.245\linewidth]{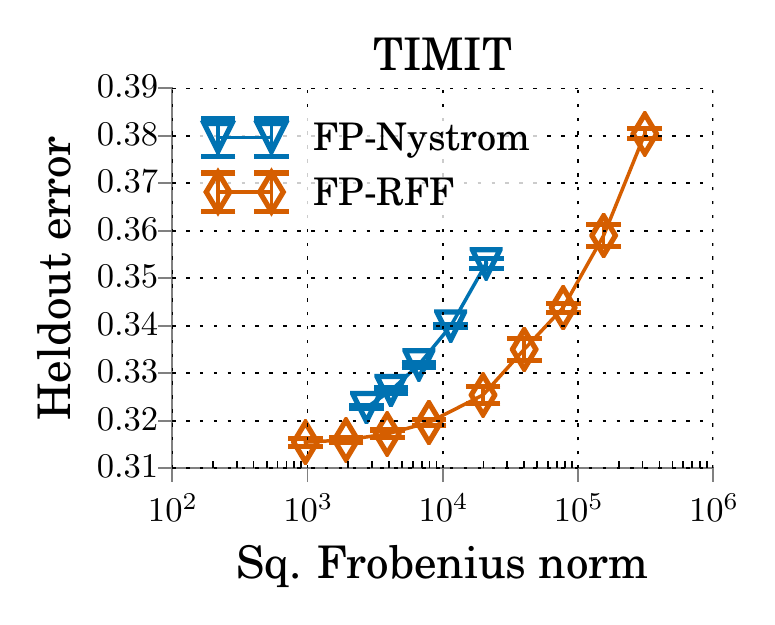} &
			\includegraphics[width=0.245\linewidth]{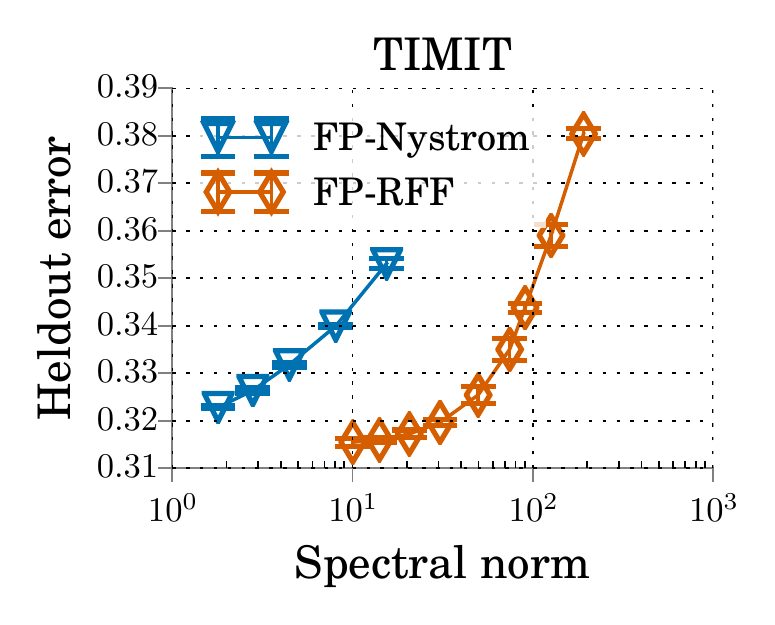}\\
			\;\;\;\;\;(a) & \;\;\;\;\;\;(b) & \;\;\;\;\;\;(c) & \;\;\;\;\;\;(d)
		\end{tabular}
	\end{small}
	\vfigsp
	\caption{
		Generalization performance of full-precision RFFs and \Nystrom with respect to the number of features and training memory footprint on TIMIT (a,b).
		\Nystrom performs better for a fixed number of features, while RFFs perform better under a memory budget.
		We also see that the generalization performance of these methods does not align well with the Frobenius or spectral norms of their respective kernel approximation error matrices (c,d).
		\ifcamera \else
		Specifically, we see many example of \Nystrom models that have lower Frobenius or spectral error, but higher heldout error, than certain RFF models.
		\fi
		For results on YearPred, CovType, and Census, see Appendix~\ref{app:nys_vs_rff_details}.}
	\label{fig:generalization_col}
\end{figure*}

To inform our design of memory-efficient kernel approximation methods, we first perform an empirical study of the generalization performance vs.\ memory utilization of \Nystrom and RFFs.
We begin by reviewing the memory utilization for these kernel approximation methods
in the mini-batch training setting; this is a standard setting for training large-scale
kernel approximation models \citep{huang14kernel,deepfried15,may2017}, and it is the setting we will be using to evaluate the different approximation methods (Sections~\ref{subsec:nys_vs_rff_exp}, \ref{subsec:full_run}).
We then show that RFFs outperform \Nystrom given the same training memory budget,
even though the opposite is true given a budget for the number of features \citep{nysvsrff12}.
Lastly, we demonstrate that the Frobenius and spectral norms of the kernel approximation error matrix align poorly with generalization performance,
suggesting the need for a more refined measure of approximation error for evaluating the quality of a kernel approximation method; we investigate this in Section~\ref{sec:refine}.

For background on RFFs and the \Nystrom method, and for a summary of our notation, see Appendix \ref{sec:background_appendix}.
\vsp

\subsection{Memory Utilization}
\label{subsec:memory_utils}
\vsp
The optimization setting we consider is mini-batch training over kernel approximation features.
To understand the training memory footprint, we present in Table \ref{table:mem-usage} the memory utilization of the different parts of the training pipeline.
The three components are:
\begin{enumerate}
\vflistsp
  \setlength\itemsep{0.25em}
	\item \textit{Feature generation}: Computing $m$ RFFs over data in $\RR^d$ requires a random projection matrix $W \in \RR^{m\times d}$.
	The \Nystrom method stores $m$ ``landmark points'' $\hx_i \in \RR^d$, and a projection matrix in $\RR^{m \times m}$.
	\ifcamera
	\item \textit{Feature mini-batch}: Kernel approximation features $z(x_i)\in\RR^m$ for all $x_i$ in a mini-batch are stored.\footnote{For simplicity,
		we ignore the memory occupied by the mini-batches of $d$-dim.\ inputs and $c$-dim.\ outputs, as generally the number of kernel approx.\ features $m \gg d,c$.
	}
	\else
	\item \textit{Feature mini-batch}: The approximate feature vectors $z(x_i)\in\RR^m$ for all $x_i$ in a mini-batch must be stored.\footnote{For simplicity,
		we ignore the memory occupied by the mini-batches of inputs (in $\RR^{d \times s}$) and outputs (in $\RR^{c\times s}$) of the model, where $d$ is the input dimension, $c$ is the number of classes, and $s$ is the mini-batch size. 
		These input and output vectors take negligible space relative to the approximate feature vectors as long as the number of approximation features $m \gg d,c$.
	\fi
	\item \textit{Model parameters}: For binary classification and regression, the linear model learned on the $z(x)$ features is a vector $\theta \in \RR^m$; for $c$-class classification, it is a matrix $\theta \in \RR^{m\times c}$.	
	\vflistsp
\end{enumerate}
In this work we focus on reducing the memory occupied by the mini-batches of features, which can occupy a significant fraction of the training memory.
Our work is thus orthogonal to existing work which has shown how to reduce the memory utilization of the feature generation \citep{fastfood,yu15} and the model parameters \citep{sainath2013low,structured15,halp18} (\eg, using structured matrices or low precision).
Throughout this paper, we measure the memory utilization of a kernel approximation method as the sum of the above three components.
\vsp
\vsp

\subsection{Empirical Comparison}
\label{subsec:nys_vs_rff_exp}
\vsp
We now compare the generalization performance of RFFs and the \Nystrom method in terms of their training memory footprint.
We demonstrate that RFFs can outperform the \Nystrom method given a memory budget, and show that the difference in performance between these methods cannot be explained by the Frobenius or spectral norms of their kernel approximation error matrices.

In experiments across four datasets (TIMIT, YearPred, CovType, Census \citep{timit,uci}), we use up to 20k \Nystrom features and 400k RFFs to approximate the Gaussian kernel;\footnote{We consider different ranges for the number of \Nystrom vs.\ RFF features because the memory footprint for training with 400k RFFs is similar to 20k \Nystrom features.}
we train the models using mini-batch stochastic gradient descent with early stopping, with a mini-batch size of 250.
We present results averaged from three random seeds, with error bars indicating standard deviations (for further experiment details, see Appendix~\ref{app:nys_vs_rff_details}).
In Figure~\ref{fig:generalization_col}(a) we observe that as a function of the number of kernel approximation features the \Nystrom method generally outperforms RFFs, though the gap narrows as $m$ approaches 20k.
However, we see in Figure~\ref{fig:generalization_col}(b) that RFFs attain better generalization performance as a function of memory.
Interestingly, the relative performance of these methods cannot simply be explained by the Frobenius or spectral norms of the kernel approximation error matrices;\footnote{We 
	consider the Frobenius and spectral norms of $K-\tK$, where $K$ and $\tK$ are the exact and approximate kernel matrices for 20k randomly sampled heldout points.
}
in Figure~\ref{fig:generalization_col}(c,d)
we see that there are many cases in which the RFFs attain better generalization performance,
in spite of having larger Frobenius or spectral approximation error.
This is a phenomenon we observe on other datasets as well (Appendix~\ref{app:nys_vs_rff_details}).
This suggests the need for a more refined measure of the approximation error of a kernel approximation method, which we discuss in the following section.

\vsp

\section{\uppercase{A refined measure of kernel approx. error}}
\label{sec:refine}
\vsp

To explain the important differences in performance between \Nystrom and RFFs, we define a more refined measure of kernel approximation error---\textit{$(\Delta_1,\Delta_2)$-spectral approximation}.
Our definition is an extension of \citeauthor{avron17}'s definition of $\Delta$-spectral approximation, in which we decouple the two roles played by $\Delta$ in the original definition.
This decoupling allows for a more fine-grained understanding of the factors influencing the generalization performance of kernel approximation methods, both theoretically and empirically.
Theoretically, we present a generalization bound for kernel approximation methods in terms of $(\Delta_1,\Delta_2)$ (Sec.~\ref{subsec:delta12}), and show that $\Delta_1$ and $\Delta_2$ influence the bound in different ways (Prop.~\ref{prop:alphabeta}).
Empirically, we show that $\Delta_1$ and $\Delta_2$ are more predictive of the \Nystrom vs.\ RFF performance than the $\Delta$ from the original definition, and the Frobenius and spectral norms of the kernel approximation error matrix (Sec.~\ref{subsec:nys_vs_rff_revisited}, Figure~\ref{fig:metrics_vs_perf}).
An important consequence of the $(\Delta_1,\Delta_2)$ definition is that attaining a small $\Delta_1$ requires a large number of features;
we leverage this insight to motivate our proposed method, low-precision random Fourier features, in Section~\ref{sec:lprff}.

\vsp

\subsection{$(\Delta_1,\Delta_2)$-spectral Approximation}
\label{subsec:delta12}
\vsp
We begin by reviewing what it means for a matrix $A$ to be a $\Delta$-spectral approximation of a matrix $B$ \citep{avron17}.
We then extend this definition to $(\Delta_1, \Delta_2)$-spectral approximation, and bound the generalization performance of kernel approximation methods in terms of $\Delta_1$ and $\Delta_2$ in the context of fixed design kernel ridge regression.
\begin{definition}
	\label{def:specdist_orig}
	For $\Delta \geq 0$, a symmetric matrix $A$ is a \emph{$\Delta$-spectral approximation} of another symmetric matrix $B$ if $(1-\Delta)B \preceq A \preceq (1+\Delta)B$. 
\end{definition}
We extend this definition by allowing for different values of $\Delta$ in the left and right inequalities above:
\begin{definition}
	\label{def:specdist}
  For $\Delta_1, \Delta_2 \geq 0$, a symmetric matrix $A$ is a \emph{$(\Delta_1, \Delta_2)$-spectral approximation} of another symmetric matrix $B$ if $(1-\Delta_1)B \preceq A \preceq (1+\Delta_2)B$. 
\end{definition}
Throughout the text, we will use $\Delta$ to denote the variable in Def.~\ref{def:specdist_orig}, and $(\Delta_1, \Delta_2)$ to denote the variables in Def.~\ref{def:specdist}.
In our discussions and experiments, we always consider the smallest $\Delta$, $\Delta_1$, $\Delta_2$ satisfying the above definitions; thus, $\Delta=\max(\Delta_1,\Delta_2)$.

In the paragraphs that follow we present generalization bounds for kernel approximation models in terms of $\Delta_1$ and $\Delta_2$ in the context of fixed design kernel ridge regression, 
and demonstrate that $\Delta_1$ and $\Delta_2$ influence generalization in different ways (Prop.~\ref{prop:alphabeta}).
We consider the fixed design setting because its expected generalization error has a closed-form expression, which allows us to analyze generalization performance in a fine-grained fashion.  For an overview of fixed design kernel ridge regression, see Appendix~\ref{subsec:app_fix_design}.

\begin{figure*}
	\centering
	\begin{small}
		\vfigsp
		\begin{tabular}{@{\hskip -0.0in}c@{\hskip -0.0in}c@{\hskip -0.0in}c@{\hskip -0.0in}c@{\hskip -0.0in}}
			\includegraphics[width=0.245\linewidth]{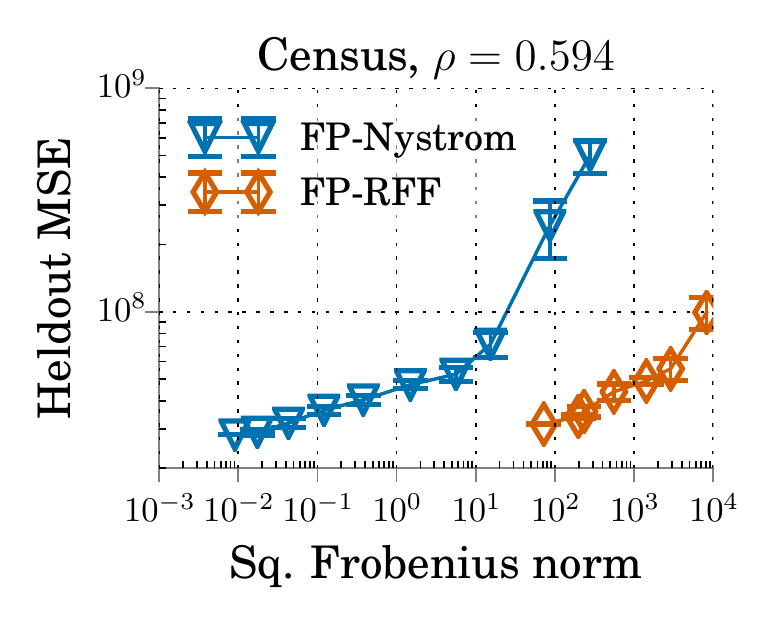} &
			\includegraphics[width=0.245\linewidth]{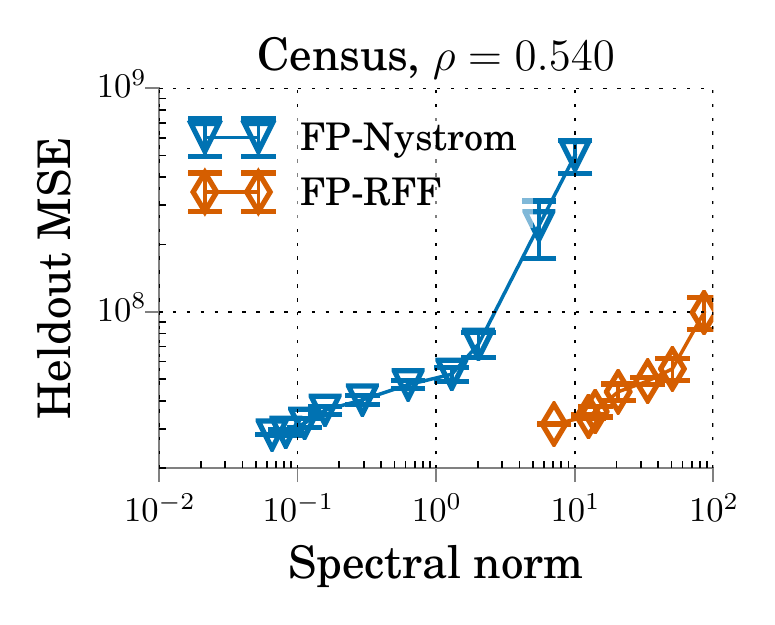} &
			\includegraphics[width=0.245\linewidth]{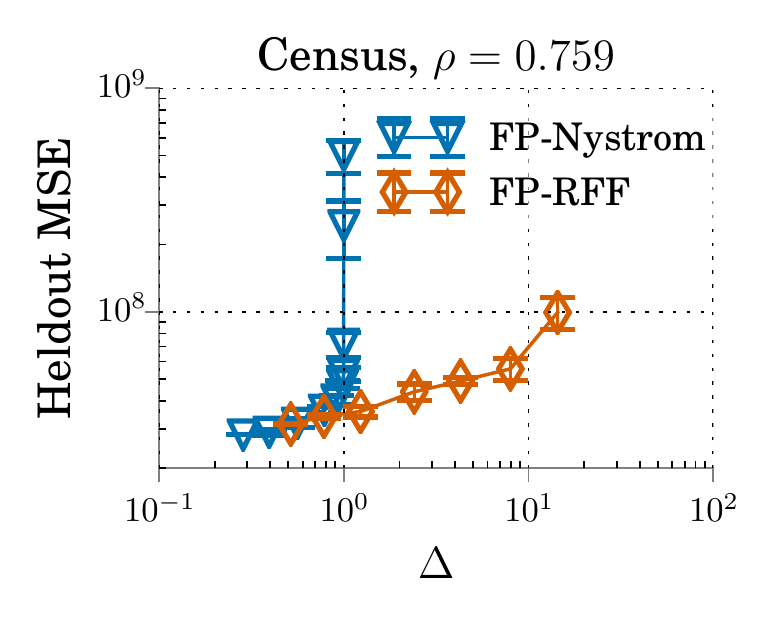} &		
			\includegraphics[width=0.245\linewidth]{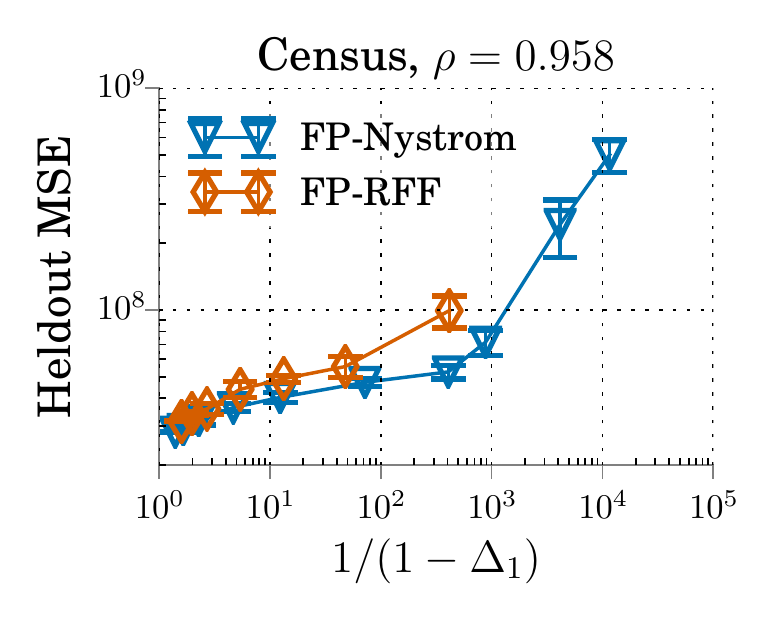} \\
		\end{tabular}
	\end{small}
	\vfigsp
	\vfigsp
	\caption{
		\ifcamera
		The correlation between generalization performance and different measures of kernel approximation error for the full-precision RFF and \Nystrom methods.
		We see that generalization performance aligns well with $1/(1-\Delta_1)$ (Spearman rank correlation coefficient $\rho=0.958$), while aligning poorly with $\Delta$ and the spectral and squared Frobenius norms of the kernel approximation error matrix.
		See Appendix~\ref{app:nys_vs_rff_revisited_details} for results on CovType.
		\else
		The correlation between generalization performance and different measures of kernel approximation error for the full-precision RFF and \Nystrom methods.
		We see that generalization performance aligns well with $1/(1-\Delta_1)$, while aligning poorly with $\Delta$ and the spectral and squared Frobenius norms of the kernel approximation error matrix.
		See Appendix~\ref{app:nys_vs_rff_revisited_details} for analogous results on CovType, which are qualitatively similar.
		\fi
}
	\label{fig:metrics_vs_perf}
\end{figure*}

In the fixed design setting,  given a kernel matrix $K\in \RR^{n\times n}$, a regularization parameter $\lambda \geq 0$,
and a set of labeled points $\{(x_i,y_i)\}_{i=1}^n$ where the observed labels $y_i=\by_i + \eps_i$ are randomly perturbed versions of the true labels $\by_i \in \RR$ ($\eps_i$ independent, $\expect{}{\eps_i} = 0$, $\expect{}{\eps_i^2} = \sigma^2 < \infty$),
it is easy to show \citep{alaoui15} that the optimal kernel regressor\footnote{$f_K(x) = \sum_i \alpha_i k(x,x_i)$ for $\alpha = (K+\lambda I)^{-1}y$.} $f_K$ has expected error
\begin{equation*}
\cR(f_K) = \frac{\lambda^2}{n} \by^T(K+\lambda I)^{-2}\by + \frac{\sigma^2}{n} \tr\Big(K^2(K+\lambda \id)^{-2}\Big),
\end{equation*}
where $\by = (\by_1,\ldots,\by_n)$ is the vector of true labels.

This closed-form expression for generalization error allows us to bound the expected loss $\cR(f_{\tK})$ of a kernel ridge regression model $f_{\tK}$ learned using an approximate kernel matrix $\tK$ in place of the exact kernel matrix $K$.
In particular, if we define
\ifcamera
	\begin{align*}
		\hcR(f_K) & \defeq \frac{\lambda}{n} \by^T(K+\lambda I)^{-1}\by + \frac{\sigma^2}{n}\tr\Big(K(K+\lambda I)^{-1}\Big), %
	\end{align*}
	which is an upper bound on $\cR(f_K)$, 
\else
	\begin{equation}
	\hcR(f_K) \defeq \frac{\lambda}{n} \by^T(K+\lambda I)^{-1}\by + \frac{\sigma^2}{n}\tr\Big(K(K+\lambda I)^{-1}\Big) \geq \cR(f_K),
	\label{eq:avron_rhat} 
	\end{equation}
\fi
\noindent we can bound the expected loss of $f_{\tK}$ as follows:
\begin{proposition}{(Extended from \citep{avron17})}
	Suppose $\tK + \lambda I$ is $(\Delta_1, \Delta_2)$-spectral approximation of $K+\lambda I$, for $\Delta_1 \in [0,1)$ and $\Delta_2 \geq 0$. Let $m$ denote the rank of $\tK$, and let $f_{K}$ and $f_{\tK}$ be the kernel ridge regression estimators learned using these matrices, with regularizing constant $\lambda \geq 0$ and label noise variance $\sigma^2 < \infty$. Then
	\begin{equation}
	\cR(f_{\tK}) \leq \frac{1}{1-\Delta_1} \hcR(f_K) +  \frac{\Delta_2}{1+\Delta_2}\frac{m}{n}\sigma^2.
	\label{eq:risk_bound}
	\end{equation}
	\label{prop:alphabeta}
	\vspace{-0.20in}
\end{proposition}
We include a proof in Appendix~\ref{subsec:generalization_and_rel_spec_dist}.
This result shows that smaller values for $\Delta_1$ and $\Delta_2$ imply tighter bounds on the generalization performance of the model trained with $\tK$.
We can see that as $\Delta_1$ approaches 1 the bound diverges, and as $\Delta_2$ approaches $\infty$ the bound plateaus.
We leverage this generalization bound to understand the difference in performance between \Nystrom and RFFs (Sec.~\ref{subsec:nys_vs_rff_revisited}), and to motivate and analyze our proposed low-precision random Fourier features (Sec.~\ref{sec:lprff}).

\paragraph{Remark}
The generalization bound in Prop.~\ref{prop:alphabeta} assumes the regressor $f_K$ is computed via the closed-form solution for kernel ridge regression.
However, in Sections~\ref{sec:lprff}-\ref{sec:experiments} we focus on stochastic gradient descent (SGD) training for kernel approximation models.
Because SGD can \textit{also} find the model which minimizes the regularized empirical loss \citep{nemirovski09}, the generalization results carry over to our setting.
\vsp

\subsection{Revisiting \Nystrom vs.\ RFF Comparison}
\label{subsec:nys_vs_rff_revisited}
\vsp
In this section we show that the values of $\Delta_1$ and $\Delta_2$ such that the approximate kernel matrix is a $(\Delta_1,\Delta_2)$-spectral approximation of the exact kernel matrix correlate better with generalization performance than the original $\Delta$, and the Frobenius and spectral norms of the kernel approximation error; we measure correlation using Spearman's rank correlation coefficient $\rho$.

To study the correlation of these metrics with generalization performance, we train \Nystrom and RFF models for many feature dimensions on the Census regression task, and on a subsampled version of 20k train and heldout points from the CovType classification task.
We choose these small datasets to be able to compute the various measures of kernel approximation error over the entire heldout set.
We measure the spectral and Frobenius norms of $K-\tK$, and the $\Delta$ and $(\Delta_1,\Delta_2)$ values between $K+\lambda I$ and $\tK+\lambda I$ ($\lambda$ chosen via cross-validation), 
where $K$ and $\tK$ are the exact and approximate kernel matrices for the heldout set.
For more details about these experiments and how we compute $\Delta$ and $(\Delta_1,\Delta_2)$, see Appendix~\ref{app:nys_vs_rff_revisited_details}.

In Figure~\ref{fig:metrics_vs_perf}, we plot the generalization performance on these tasks as a function of these metrics;
while the original $\Delta$ and the Frobenius and spectral norms generally do not align well with generalization performance, we see that $\frac{1}{1-\Delta_1}$ does.
Specifically, $\frac{1}{1-\Delta_1}$ attains a Spearman rank correlation coefficient of $\rho=0.958$, while squared Frobenius norm, spectral norm, and the original $\Delta$ attain values of 0.594, 0.540, and 0.759.\footnote{
One reason $\Delta_1$ correlates better than $\Delta$ is because when $\Delta_2 > \Delta_1$, $\Delta=\max(\Delta_1,\Delta_2)$ hides the value of $\Delta_1$.
This shows why decoupling the two roles of $\Delta$ is important.
}
In Appendix~\ref{app:nys_vs_rff_revisited_details} we show these trends are robust to different kernel approximation methods and datasets.
For example, we show that while other approximation methods (\eg, orthogonal RFFs \citep{yu16}), like \NystromNS, can attain much lower Frobenius and spectral error than standard RFFs, this does not translate to improved $\Delta_1$ or heldout performance.
\ifcamera\else
We also show that the trends in Figure~\ref{fig:metrics_vs_perf} occur on larger datasets, where we measure $(\Delta_1,\Delta_2)$ on a random subset of $20k$ heldout points, and for a range of $\lambda$ values (trend is quite robust to choice of $\lambda$).
\fi
These results mirror the generalization bound in Proposition~\ref{prop:alphabeta}, which grows linearly with $\frac{1}{1-\Delta_1}$.
For simplicity, we ignore the role of $\Delta_2$ here,
as $\Delta_1$ appears to be sufficient for explaining the main differences in performance between these full-precision methods.\footnote{While $1/(1-\Delta_1)$ aligns well with performance, it is not perfect---for a fixed $\Delta_1$, \Nystrom generally performs slightly better than RFFs. In App.~\ref{subsec:app_exp_smallscale} we suggest this is because \Nystrom has $\Delta_2 = 0$
while RFFs has larger $\Delta_2$.
}
In Sections~\ref{subsec:theory} and \ref{subsec:gen_perf_and_delta}, however, we show that $\Delta_2$ has a large influence on generalization performance for low-precision features.

\ifcamera \else
The fact that $1/(1-\Delta_1)$ aligns well with generalization performance, but $\Delta$ does not, shows the importance of decoupling the two roles of $\Delta$.
As we have seen above, the value of $\Delta_1$ is very important for explaining generalization performance; unfortunately, $\Delta=\max(\Delta_1,\Delta_2)$ often hides the value of $\Delta_1$ for RFFs, but not for \NystromNS.
This is because although \Nystrom and RFFs have similar values of $\Delta_1$ for a fixed number of features, RFFs often have $\Delta_2$ values greater than $\Delta_1$, leading to $\Delta=\Delta_2$ (Figure~\ref{fig:theory_supporting}).
\NystromNS, on the other hand, always has $\Delta=\Delta_1$ because $\Delta_2 = 0$.\footnote{This is because $0 \preceq \tK \preceq K$ for \Nystrom kernel matrices $\tK$ \citep{smola02}.}
By separately measuring $\Delta_1$ and $\Delta_2$,
we are able to see that generalization performance aligns well with $1/(1-\Delta_1)$ for these full-precision methods.
\fi

Now that we have seen that $\Delta_1$ has significant theoretical and empirical impact on generalization performance, it is natural to ask how to construct kernel approximation matrices that attain small $\Delta_1$.
An important consequence of the definition of $\Delta_1$ is that for $\tK + \lambda I$ to have small $\Delta_1$ relative to $K+\lambda I$, $\tK$ must be \textit{high-rank};
in particular, a necessary condition is 
$\Delta_1 \geq \frac{\lambda_{m+1}(K)}{\lambda_{m+1}(K)+\lambda}$,
where $m$ is the rank of $\tK$ and $\lambda_i(K)$ is the $i^{th}$ largest eigenvalue of
$K$.\footnote{By definition,
	$(K+\lambda I)(1-\Delta_1) \preceq \tK + \lambda I$.
	By Weyl's inequality this implies $\forall i \; (\lambda_i(K) + \lambda)(1-\Delta_1) \leq \lambda_i(\tK) + \lambda$.
	If $\tK$ is rank $m$, then $\lambda_{m+1}(\tK) = 0$, and the result follows.}
This sets a lower bound on the rank necessary for $\tK$ to attain small $\Delta_1$ which holds regardless of the approximation method used, motivating us to design high-rank kernel approximation methods.

\vsp

\section{\uppercase{Low-Precision Random Fourier Features (LP-RFFs)}}
\label{sec:lprff}
\vsp

\begin{figure*}
	\centering
	\begin{small}
		\vfigsp
		\begin{tabular}{c c c}
			\includegraphics[height=0.22\linewidth]{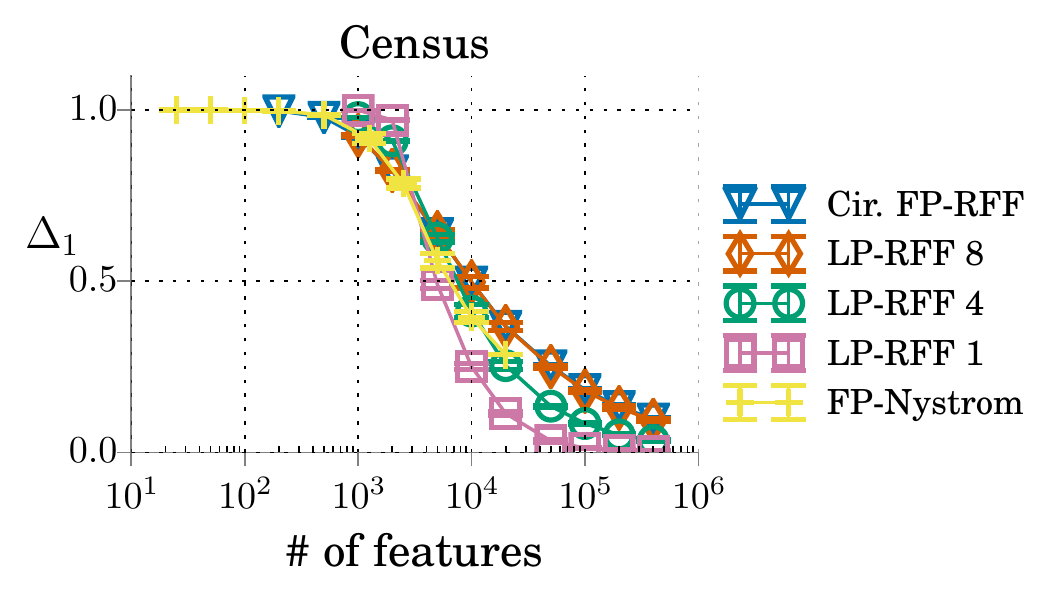} &      
			\hspace{-0.25in}\includegraphics[height=0.22\linewidth]{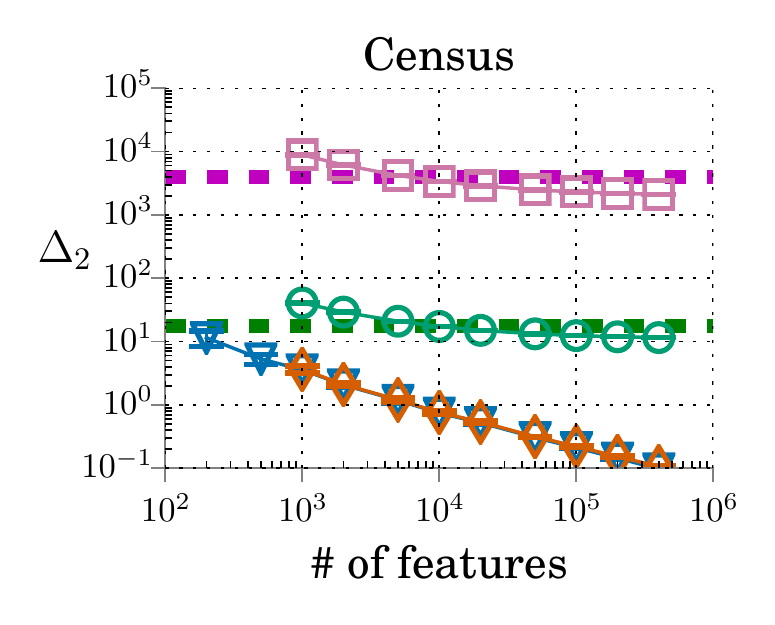} &
			\includegraphics[height=0.22\linewidth]{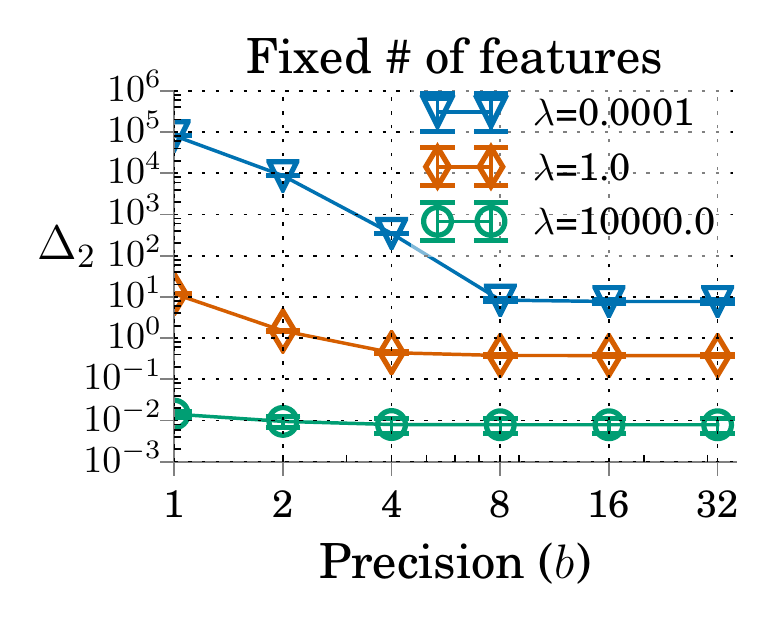} \\
		\end{tabular}
	\end{small}
	\vfigsp
	\vfigsp
	\caption{
		Empirical validation of Theorem~\ref{thm2}. In the left and middle plots (shared legend), we see that as the \# of features grows, LP-RFFs approach $\Delta_1=0$,
		but plateau at larger $\Delta_2$ values (at most $\delta_b^2/\lambda$, marked by dashed lines) for very low precisions.
		In the right plot we see that the larger $\lambda$ is, the lower the precision at which using low precision does not impact $\Delta_2$.
		\ifcamera \else
		For reference, in the left plot we also include the $\Delta_1$ values attained by the \Nystrom method;
		we can see that for the same number of features, the \Nystrom method attains very similar $\Delta_1$ values to RFFs. No \Nystrom line is include in the middle figure because $\Delta_2=0$ for \NystromNS.
		\fi
		For $\Delta_1$ and $\Delta_2$ vs.\ \# features plots on CovType, see Appendix~\ref{app:theory_validation_details}.
	}
	\label{fig:theory_supporting}
\end{figure*}

Taking inspiration from the above-mentioned connection between the rank of the kernel approximation matrix and generalization performance, we propose \textit{low-precision random Fourier features} (LP-RFFs) to create a high-rank approximation matrix under a memory budget.
In particular, we quantize each random Fourier feature to a low-precision fixed-point representation,
thus allowing us to store more features in the same amount of space. %
Theoretically, we show that when the quantization noise is small relative to the regularization parameter,
using low precision has minimal impact on the number of features required for the approximate kernel matrix to be a $(\Delta_1,\Delta_2)$-spectral approximation of the exact kernel matrix; 
by Proposition~\ref{prop:alphabeta}, this implies a bound on the generalization performance of the model trained on the low-precision features. 
At the end of this section (Section~\ref{sec:systems}), we discuss a memory-efficient implementation for training a full-precision model on top of LP-RFFs.
\vsp

\subsection{Method Details}
\label{subsec:method_details}
\vsp
The core idea behind LP-RFFs is to use $b$ bits to store each RFF, instead of $32$ or $64$ bits.
We implement this with a simple stochastic rounding scheme.
We use the parametrization $z_i(x) = \sqrt{2/m}\cos(w_i^T x + a_i) \in [-\sqrt{2/m},\sqrt{2/m}]$ for the RFF vector $z(x)\in\RR^m$ \citep{rahimi07random}, and divide this interval into $2^b - 1$ sub-intervals of equal size $r = \frac{2\sqrt{2/m}}{2^b-1}$.
We then randomly round each feature $z_i(x)$ to either the top or bottom of the sub-interval $[\underline{z},\overline{z}]$ containing it, in such a way that the expected value is equal to $z_i(x)$; 
specifically, 
we round $z_i(x)$ to $\underline{z}$ with probability $\frac{\overline{z}-z}{\overline{z}-\underline{z}}$ and to $\overline{z}$ with probability $\frac{z-\underline{z}}{\overline{z}-\underline{z}}$.
The variance of this stochastic rounding scheme is at most $\delta_b^2/m$, where $\delta_b^2 \defeq 2/(2^b-1)^2$ (Prop.~\ref{prop:var_bound} in App.~\ref{sec:app_kernel_error}).
For each low-precision feature $\tz_i(x)$ we only need to store the integer $j \in [0,2^b-1]$ such that $\tz_i(x) = -\sqrt{2/m} + jr$, which takes $b$ bits.
Letting $\tZ\in\RR^{n\times m}$ denote the matrix of quantized features, 
we call $\tK = \tZ\tZ^T$ an \textit{$m$-feature $b$-bit LP-RFF approximation} of a kernel matrix $K$.

\ifcamera
As a way to further reduce the memory footprint during training, we leverage existing work on using circulant random matrices \citep{yu15} for the RFF random projection matrix to only occupy $32m$ bits.\footnote{Technically, $m$ additional bits are needed to store a vector of Rademacher random variables in $\{-1,1\}^m$.}
\else
As a way to further reduce the memory footprint during training, we leverage existing work on using circulant random matrices \citep{yu15} for the RFF random projection matrix to only occupy $32m$ bits.\footnote{Technically, 
	a vector of Rademacher random variables in $\{-1,1\}^m$ must also be stored, 
	increasing the total space requirement to $33m$.} 
\fi
All our LP-RFF experiments use circulant projections.
\vsp

\subsection{Theoretical Results}
\label{subsec:theory}
\vsp
In this section we show quantization has minimal impact on the number of features required to guarantee strong generalization performance in certain settings.  %
We do this in the following theorem by lower bounding the probability that $\tK + \lambda I$ is a $(\Delta_1,\Delta_2)$-spectral approximation of $K+\lambda I$, for the LP-RFF approximation $\tK$ using $m$ features and $b$ bits per feature.\footnote{This theorem extends directly to the quantization of any kernel approximation feature matrix $Z\in\RR^{n\times m}$ with i.i.d.\ columns and with entries in $[-\sqrt{2/m},\sqrt{2/m}]$.}

\begin{theorem}
	Let $\tK$ be an $m$-feature $b$-bit LP-RFF approximation of a kernel matrix $K$,
	assume $\norm{K} \geq \lambda \geq \delta^2_b \defeq 2/(2^b-1)^2$,
	and define $a\defeq 8 \tr \left( (K + \lambda I_n)^{-1}(K + \delta^2_bI_n)\right)$.
	Then for any $\Delta_1 \geq 0$, $\Delta_2 \geq \delta^2_b/\lambda$,
\ifcamera
	\begin{align*}
		&\Prob\Big[{\textstyle (1-\Delta_1)(K+\lambda I) \preceq \tK+\lambda I \preceq (1+\Delta_2)(K+\lambda I)}\Big] \geq \\
		&{1\!-\!a\!\left(\!\!
		\exp\!\left(\!\frac{-m\Delta_1^2}{\frac{4n}{\lambda}(1 + \frac{2}{3}\Delta_1)} \right) \!+\!
		\exp\!\left(\frac{-m(\Delta_2 - \frac{\delta^2_b}{\lambda})^2}{\frac{4n}{\lambda}(1 + \frac{2}{3}(\Delta_2 - \frac{\delta^2_b}{\lambda}))}\!\right)
		\!\!\right)}.
	\end{align*}
\else
	\begin{align*}
	\Prob\Big[(1-\Delta_1)&(K+\lambda I) \preceq \tK+\lambda I \preceq (1+\Delta_2)(K+\lambda I)\Big] \geq \\
	&1-8 \tr \left( (K + \lambda I_n)^{-1}(K + \delta^2_bI_n)\right)\left(
		\exp\left(\frac{-m\Delta_1^2}{\frac{4n}{\lambda}(1 + \frac{2}{3}\Delta_1)} \right) +
		\exp\left(\frac{-m(\Delta_2 - \frac{\delta^2_b}{\lambda})^2}{\frac{4n}{\lambda}(1 + \frac{2}{3}(\Delta_2 - \frac{\delta^2_b}{\lambda}))}\right)\right).
	\end{align*}
\fi
\label{thm2}
\end{theorem}
The proof of Theorem~\ref{thm2} is in Appendix~\ref{sec:lprff_theory_appendix}.
\ifcamera
To provide more intuition we present the following corollary:
\else
To provide more intuition and simplify the above expression, we present the following corollary:
\fi
\begin{corollary}
  \label{cor:n_features_required}
	Assuming $\Delta_1 \leq 3/2$,
	it follows that $(1-\Delta_1) (K\!+\!\lambda I_n) \preceq \tilde{K} + \lambda I_n$ with probability at least $1-\rho$ if
	$m \geq \frac{8n/\lambda}{\Delta_1^2}\log\Big(\frac{a}{\rho}\Big).$
	Similarly, assuming
	$\Delta_2 \in \big[\frac{\delta_b^2}{\lambda}, \frac{3}{2}\big]$,
	it follows that $\tilde{K}+\lambda I_n \preceq (1+\Delta_2)(K+\lambda I_n)$ with probability at least $ 1-\rho$ if
	$m \geq \frac{8n/\lambda}{(\Delta_2-\delta_b^2/\lambda)^2}\log\Big(\frac{a}{\rho}\Big).$
\end{corollary}
\vsp

The above corollary suggests that using low precision has negligible effect on the number of features necessary to attain a certain value of $\Delta_1$, and also has negligible effect for $\Delta_2$ as long as $\delta_b^2/\lambda \ll \Delta_2$.

\begin{figure*}
	\centering
	\begin{small}
		\vfigsp
		\begin{tabular}{c c c}
			\includegraphics[width=0.27\linewidth]{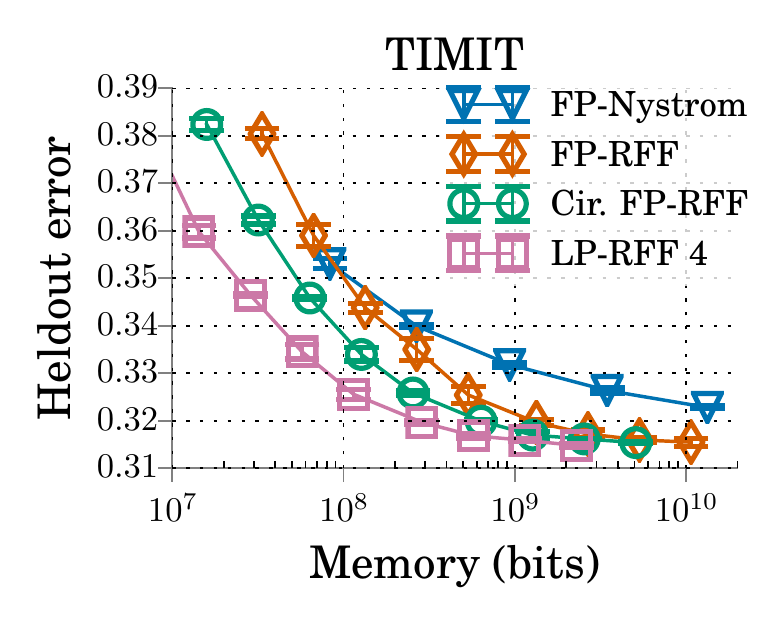} &	
			\includegraphics[width=0.27\linewidth]{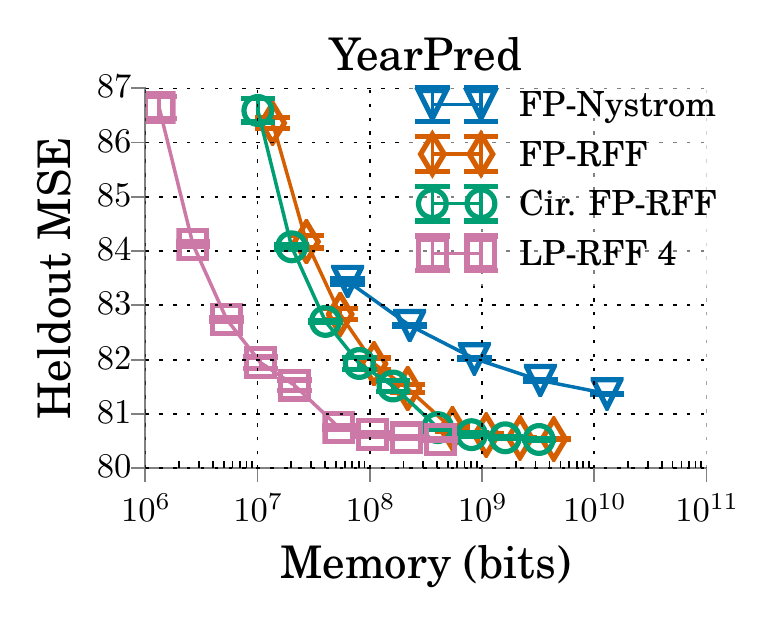} &
			\includegraphics[width=0.27\linewidth]{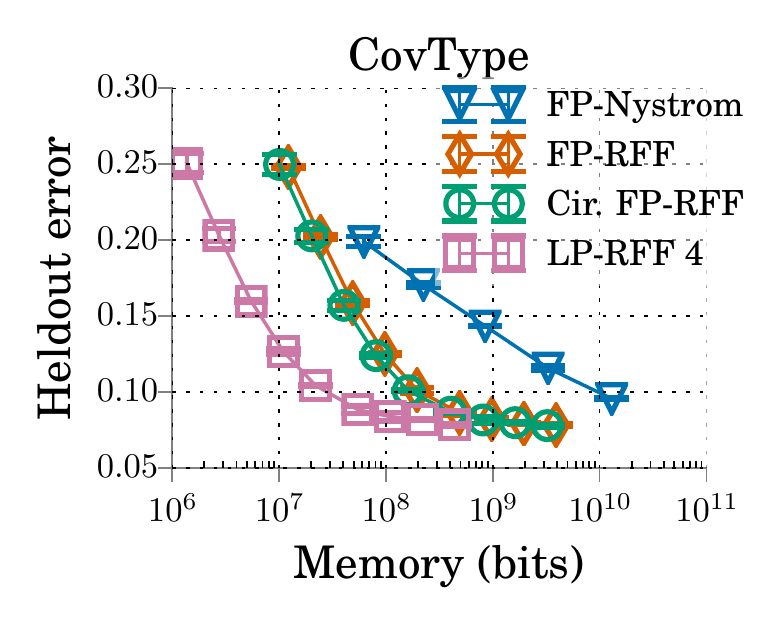}
		\end{tabular}
	\end{small}
	\vfigsp
	\vfigsp
	\caption{
		Generalization performance of FP-\NystromNS, FP-RFFs, circulant FP-RFFs, and LP-RFFs with respect to memory (sum of components in Table~\ref{table:mem-usage}) on TIMIT, YearPred and CovType. 
		LP-RFFs attain the best performance across a wide range of memory budgets.
		The same trend holds for Census in Appendix~\ref{app:lprff_exp_details}.
	}
	\label{fig:generalization_col_full}
\end{figure*}
\vsp \vsp \vsp
\paragraph{Validation of Theory} We now empirically validate the following two predictions made by the above theory:
(1) Using low precision has no effect on the asymptotic behavior of $\Delta_1$ as the number of features $m$ approaches infinity, while having a significant effect on $\Delta_2$ when $\delta_b^2/\lambda$ is large.
Specifically, as $m\rightarrow \infty$, $\Delta_1$ converges to 0 for any precision $b$, while $\Delta_2$ converges to a value upper bounded by $\delta_b^2/\lambda$.\footnote{By
	Lemma~\ref{lem:expectation_CCstar} in Appendix~\ref{sec:lprff_theory_appendix}, we know that $\expect{}{\tZ\tZ^T} = K+D$ for a diagonal matrix $D$ satisfying $0\preceq D \preceq \delta_b^2 I_n$, where $D$ is independent of $m$.
	As $m\rightarrow \infty$, $\Delta_2$ converges to $\|(K+\lambda I)^{-1/2}D(K+\lambda I)^{-1/2}\| \leq \delta_b^2/\lambda$.
}
(2) If $\delta_b^2/\lambda \ll \Delta_2$, using $b$-bit precision will have negligible effect on the number of features required to attain this $\Delta_2$.
Thus, the larger $\lambda$ is, the smaller the impact of using low precision should be on $\Delta_2$.

To validate the first prediction, %
in Figure~\ref{fig:theory_supporting} (left, middle) we plot $\Delta_1$ and $\Delta_2$ as a function of the number of features $m$, for FP-RFFs and LP-RFFs;
we use the same $\lambda$ as in the Section~\ref{sec:nys_vs_rff} Census experiments.
We show that for large $m$, all methods approach $\Delta_1 = 0$; 
in contrast, for precisions $b \leq 4$ the LP-RFFs converge to a $\Delta_2$ value much larger than 0, and slightly less than $\delta_b^2/\lambda$ (marked by dashed lines).

To validate the second prediction, in Figure~\ref{fig:theory_supporting} (right) we plot $\Delta_2$ vs.\ precision for various values of $\lambda$, using $m=2000$ features for all precisions; 
we do this on a random subsample of 8000 Census training points.
We see that for large enough precision $b$, the $\Delta_2$ is very similar to the value from using 32-bit precision.
Furthermore, the larger the value of $\lambda$, the smaller the precision $b$ can be without significantly affecting $\Delta_2$.
\ifcamera\else
This aligns with the theory.
\fi
\vsp

\subsection{Implementation Considerations}
\label{sec:systems}
\vsp
In this paper, we focus on training full-precision models using mini-batch training over low-precision features.
Here we describe how this mixed-precision optimization can be implemented in a memory-efficient manner. %

Naively, to multiply the low-precision features with the full-precision model, one could first cast the features to full-precision, requiring significant intermediate memory.
We can avoid this by \textit{casting in the processor registers}.
Specifically, to perform multiplication with the full-precision model, the features can be streamed to the processor registers in low precision, and then cast to full precision in the registers. 
In this way, only the features in the registers exist in full precision.
A similar technique can be applied to avoid intermediate memory in the low-precision feature computation---after a full-precision feature is computed in the registers, 
it can be directly quantized in-place before it is written back to main memory.
\ifcamera
We leave a more thorough investigation of these systems issues for future work.
\else
This technique can be implemented on both CPUs and GPUs.
Because the focus of this work is analyzing the statistical and empirical properties of LP-RFFs, we leave a more thorough investigation of these systems issues for future work.
\fi

\begin{table}[t]
	\vfigsp
	\centering
	\caption{
		\ifcamera
		The compression ratios achieved by LP-RFFs relative to the best performing full-precision baselines.
		\else
		The compression ratios achieved by LP-RFFs relative to the best performing configurations for each baseline (FP-RFFs, circulant FP-RFFs, \NystromNS).
		\fi
	}
	\begin{tabular}{c c c c}
		\toprule
		& FP-RFFs & Cir. FP-RFFs & \Nystrom \\
		\midrule
		Census & 2.9x & 15.6x & 63.2x \\
		YearPred & 10.3x & 7.6x & 461.6x \\ 
		Covtype & 4.7x & 3.9x & 237.2x \\ 
		TIMIT & 5.1x & 2.4x & 50.9x \\ 
		\bottomrule
	\end{tabular}
	\label{tab:mem_saving}
	\vfigsp
	\vfigsp
\end{table}

\vsp

\section{\uppercase{Experiments}}
\label{sec:experiments}
\vsp

In this section, we empirically demonstrate the performance of LP-RFFs under a memory budget, and show that $(\Delta_1,\Delta_2)$ are predictive of generalization performance.
\ifcamera
We show in Section~\ref{subsec:full_run} that LP-RFFs can attain the same performance as FP-RFFs and \NystromNS, while using 3x-10x and 50x-460x less memory.
In Section~\ref{subsec:gen_perf_and_delta}, we show the strong alignment between $(\Delta_1, \Delta_2)$ and generalization performance, once again validating the importance of this measure.
\else
We show in Section~\ref{subsec:full_run} that LP-RFFs can attain the same performance as full-precision RFFs and \NystromNS, while using 3x-10x and 50x-460x less memory, respectively.
In Section~\ref{subsec:gen_perf_and_delta}, we demonstrate the strong alignment between $(\Delta_1, \Delta_2)$ and generalization performance, once again validating the importance of this measure of kernel approximation error.
\fi
\vsp

\subsection{Empirical Evaluation of LP-RFFs}
\label{subsec:full_run}
\vsp
To empirically demonstrate the generalization performance of LP-RFFs, we compare their performance to FP-RFFs, circulant FP-RFFs, and \Nystrom features for various memory budgets.
We use the same datasets and protocol as the large-scale \Nystrom vs.\ RFF comparisons in  Section~\ref{subsec:nys_vs_rff_exp};
the only significant additions here are that we also evaluate the performance of circulant FP-RFFs, and LP-RFFs for precisions $b \in \{1,2,4,8,16\}$.
Across our experiments, we compute the total memory utilization as the sum of all the components in Table~\ref{table:mem-usage}.
We note that all our low-precision experiments are done in \textit{simulation}, which means we store the quantized values as full-precision floating-point numbers.
We report average results from three random seeds, with error bars showing standard deviations.
For more details about our experiments, see Appendix~\ref{app:lprff_exp_details}.
\ifcamera
We use the above protocol to validate the following claims on the performance of LP-RFFs.\footnote{
	Our code: \url{github.com/HazyResearch/lp_rffs}.
}
\else
We use the above protocol to validate the following claims on the empirical performance of LP-RFFs.\footnote{To replicate our experimental results, please use our GitHub repository (\url{github.com/HazyResearch/lp_rffs}).}
\fi
\vsp

\begin{figure*}
	\vfigsp
	\centering
	\begin{tabular}{@{\hskip -0.0in}c@{\hskip -0.0in}c@{\hskip -0.0in}c@{\hskip -0.0in}c@{\hskip -0.0in}}
		\includegraphics[width=.245\linewidth]{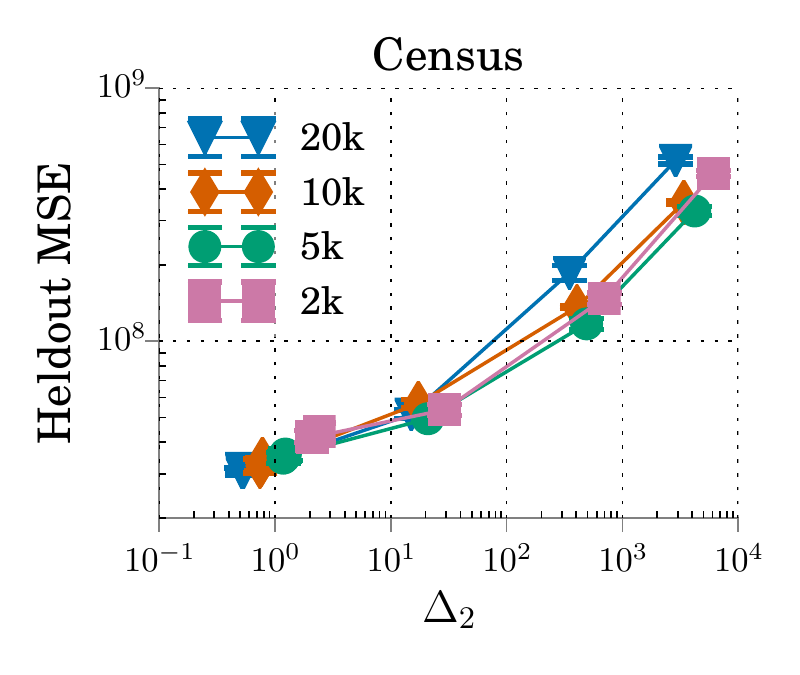} &
		\includegraphics[width=.245\linewidth]{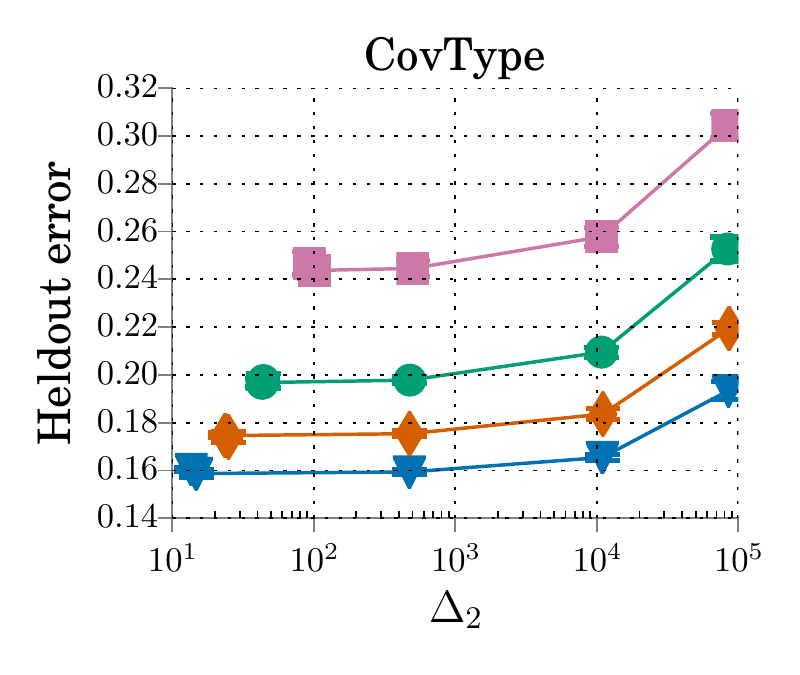} &
		\includegraphics[width=.245\linewidth]{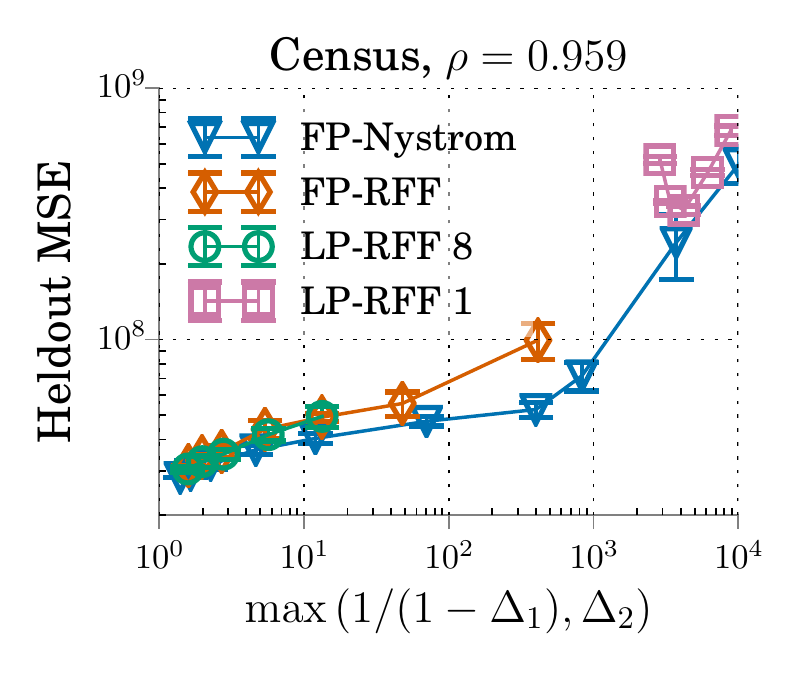} &
		\includegraphics[width=.245\linewidth]{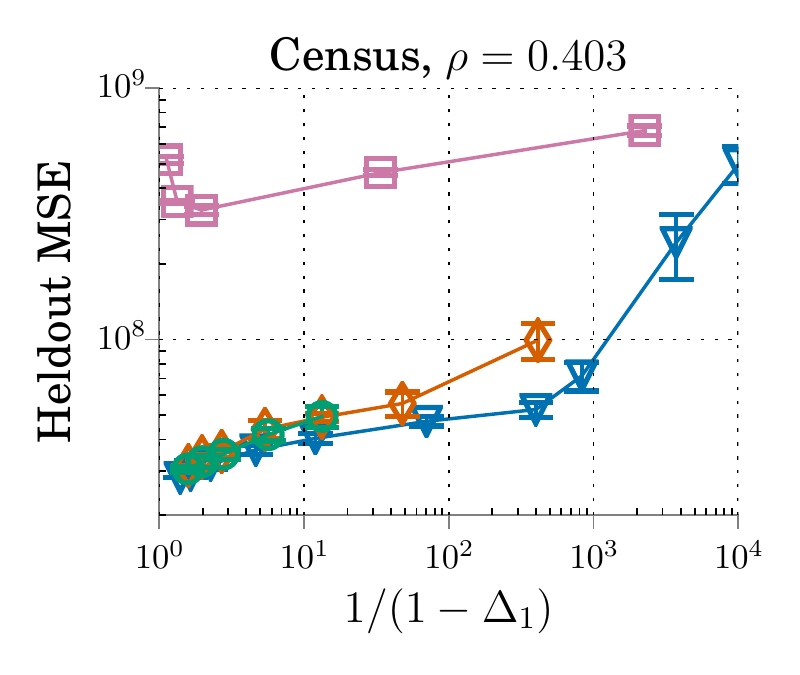} \\
	\end{tabular}
	\vfigsp
	\vfigsp
	\caption{
		Generalization perf.\ vs.\ $\Delta_2$ (left plots, shared legend), and vs.\ $1/(1-\Delta_1)$ and $\max\left( 1/(1-\Delta_1), \Delta_2 \right)$ (right plots, shared legend).
		Left: heldout performance deteriorates as $\Delta_2$ gets larger due to lower precision.
		Right: $\max\left( 1/(1-\Delta_1), \Delta_2 \right)$ aligns well with performance across LP-RFF precisions (Spearman rank correlation coefficient $\rho=0.959$), while $1/(1-\Delta_1)$ aligns poorly ($\rho=0.403$).
		See Appendix~\ref{subsec:gen_perf_and_delta_app} for CovType results.
	}
	\label{fig:gen_delta_correlation}
	\vfigsp
\end{figure*}

\vsp \vsp
\paragraph{LP-RFFs can outperform full-precision features under memory budgets.}
In Figure \ref{fig:generalization_col_full}, we plot the generalization performance for these experiments as a function of the total training memory for TIMIT, YearPred, and CovType.
We observe that LP-RFFs attain better generalization performance than the full-precision baselines under various memory budgets.
\ifcamera \else
It's important to note that LP-RFFs (using circulant projections) outperform circulant FP-RFFs,
showing that the performance improvement under a memory budget comes primarily from using low-precision features.
For clarity, we only plot LP-RFF results for 4 bits. 
\fi
To see results for all precisions, as well as results on additional benchmark datasets (Census, Adult, Cod-RNA, CPU, Forest) from the UCI repository \citep{uci}, see Appendix~\ref{app:lprff_exp_details}.

\vsp
\paragraph{LP-RFFs can match the performance of full-precision features with significantly less memory.}
\ifcamera
In Table~\ref{tab:mem_saving} we present the compression ratios we achieve with LP-RFFs relative to the best performing baseline methods.
For each baseline (FP-RFFs, circulant FP-RFFs, \NystromNS), we find the smallest LP-RFF model, as well as the smallest baseline model, which attain within $10^{-4}$ relative performance of the best-performing baseline model;
we then compute the ratio of the memory used by these two models (baseline/LP-RFF) for three random seeds, and report the average.
We can see that LP-RFFs demonstrate significant memory saving over FP-RFFs, circulant FP-RFFs, and \NystromNS, attaining compression ratios of 2.9x-10.3x, 2.4x-15.6x, and 50.9x-461.6x, respectively.
\else
In Table~\ref{tab:mem_saving} we present the compression ratios we achieve with LP-RFFs relative to the best performing baseline methods, 
while still attaining generalization performance within $10^{-4}$ relative performance of the full-precision baseline.
We measure this as follows: For each baseline (FP-RFFs, circulant FP-RFFs, \NystromNS), we find the best performing model.
We then find the smallest LP-RFF model, as well as the smallest baseline model, which attain within $10^{-4}$ relative performance of the best-performing baseline model.
We report the ratio of the memory used by these two models (baseline/LP-RFF), that are both within $10^{-4}$ relative performance of the best baseline.
We compute this ratio for three random seeds, and report the average.
We can see that LP-RFFs demonstrate significant memory saving over FP-RFFs, circulant FP-RFFs, and \NystromNS, attaining compression ratios of 2.9x-10.3x, 2.4x-15.6x, and 50.9x-461.6x, respectively.
On the TIMIT dataset, there are 147 classes.
Thus the full-precision model parameters occupy a large portion of the total memory for all methods.
Nonetheless, even though these LP-RFF experiments only quantize feature mini-batches,
they still attain 5.1x and 2.4x compression relative to FP-RFFs and circulant FP-RFFs.
As discussed in Section~\ref{subsec:memory_utils}, there are a number of ways to reduce the memory occupied by the model parameters;
in Appendix~\ref{sec:halp}, we present results using a low-precision parameterization of the model \citep{halp18}, and attain similar performance to full-precision training.
\fi

\vsp 
\subsection{Generalization Performance vs.\ $(\Delta_1,\Delta_2)$}
\label{subsec:gen_perf_and_delta}
\vsp

In this section we show that $\Delta_1$ and $\Delta_2$ are together quite predictive of generalization performance across all the kernel approximation methods we have discussed. 
We first show that performance deteriorates for larger $\Delta_2$ values as we vary the precision of the LP-RFFs, when keeping the number of features constant (thereby limiting the influence of $\Delta_1$ on performance).
\ifcamera
We then combine this insight with our previous observation (Section~\ref{subsec:nys_vs_rff_revisited}) that performance scales with $\frac{1}{1-\Delta_1}$ in the full-precision setting by showing that across precisions the performance aligns well with $\max\big(\frac{1}{1-\Delta_1},\Delta_2\big)$.
\else
We then combine this insight with our observation from Section~\ref{subsec:nys_vs_rff_revisited} that performance scales with $1/(1-\Delta_1)$ in the full-precision setting (where the empirical influence of $\Delta_2$ on performance appears to be limited);
we do this by showing that across all our methods performance aligns well with $\max\big(\frac{1}{1-\Delta_1},\Delta_2\big)$.
\fi
For these experiments, we use the same protocol as for the $(\Delta_1,\Delta_2)$ experiments in Section~\ref{subsec:nys_vs_rff_revisited}, but additionally consider LP-RFFs for precisions $b \in \{1,2,4,8,16\}$.

We show in Figure~\ref{fig:gen_delta_correlation} (left plots) that for a fixed number of random Fourier features, performance deteriorates as $\Delta_2$ grows.
As we have shown in Figure~\ref{fig:theory_supporting} (left), $\Delta_1$ is primarily governed by the rank of the approximation matrix, and thus holding the number of features constant serves as a proxy for holding $\Delta_1$ roughly constant.
This allows us to isolate the impact of $\Delta_2$ on performance as we vary the precision.
\ifcamera\else
As we have also shown in Figure~\ref{fig:theory_supporting} (middle), using low-precision can increase the $\Delta_2$ significantly, and thus varying the precision gives us a powerful tool for studying the impact of $\Delta_2$ on generalization.
\fi

\ifcamera
To integrate the influence of $\Delta_1$ and $\Delta_2$ on generalization performance into a single scalar, we consider $\max\big(\frac{1}{1-\Delta_1},\Delta_2\big)$.
In Figure~\ref{fig:gen_delta_correlation} (right plots) we show that when considering both low-precision and full-precision features, 
$\max\big(\frac{1}{1-\Delta_1},\Delta_2\big)$ aligns well with performance ($\rho=0.959$, incorporating \textit{all} precisions),
while $\frac{1}{1-\Delta_1}$ aligns poorly ($\rho=0.403$).
\else
To integrate the influence of $\Delta_1$ and $\Delta_2$ on generalization performance into a single scalar, we consider $\max\big(\frac{1}{1-\Delta_1},\Delta_2\big)$.
In Figure~\ref{fig:gen_delta_correlation} (right plots) we
plot the performance of LP-RFFs, FP-RFFs, and \Nystrom as a function of this metric.
Strikingly, we observe across methods that this quantity is predictive of performance.
\fi

In Appendix~\ref{sec:app_generalization_bound} we argue that performance scales roughly as $\Delta_2$ instead of as $\Delta_2/(1+\Delta_2)$ (as suggested by Prop.~\ref{prop:alphabeta}) due to looseness in the Prop.~\ref{prop:alphabeta} bound.

\ifcamera \else
It is important to note that although the generalization bound from Proposition~\ref{prop:alphabeta} suggests that performance should scale roughly linearly in $\Delta_2/(1+\Delta_2)$, 
we observe in our experiments that it is roughly linear in $\Delta_2$ (Figures~\ref{fig:gen_delta_correlation}(a,b)).
In the appendix, we argue that this is due to looseness in the bound, and provide theoretical (App.~\ref{subsec:d1_d2_generalization}) and empirical (App.~\ref{subsec:delta2_bias_empirical}) support for this claim.
In short, $\frac{1}{1-\Delta_1} \hcR(f_K)$ serves as an upper bound for the bias term $\frac{\lambda^2}{n} \by^T(K+\lambda I)^{-2}\by$ in the generalization bound; 
though this upper bound is independent of $\Delta_2$, we show in experiments that $\Delta_2$ has a large impact on this bias term.
We also provide theoretical justification for the reason $\Delta_2$ influences the bias term.
\fi

\vsp

\section{\uppercase{Related Work}}
\label{sec:relwork}
\vsp
\ifcamera
\paragraph{Low-Memory Kernel Approximation}
For RFFs, there has been work on using structured random projections \citep{fastfood,yu15,yu16}, and feature selection \citep{sparseRKS, may2016} to reduce memory utilization.
Our work is orthogonal, as LP-RFFs can be used with both. For \NystromNS, there has been extensive work on improving the choice of landmark points, and reducing the memory footprint in other ways \citep{ensemble09,fastpred14,meka14,musco17}.
In our work, we focus on the effect of \textit{quantization} on generalization performance per bit, and note that RFFs are much more amenable to quantization.
For our initial experiments quantizing \Nystrom features, see Appendix~\ref{app:other_results}.
\vsp \vsp \vsp \vsp
\paragraph{Low Precision for Machine Learning}
There has been much recent interest in using low precision for accelerating training and inference of machine learning models, as well as for model compression \citep{gupta15,hogwild15,hubara16,halp18,desa17,han15}.
There have been many advances in hardware support for low precision as well \citep{tpu17,brainwave17}.

This work is inspired by the \Nystrom vs.\ RFF experiments in the PhD dissertation of \citet{maythesis}, and provides a principled understanding of the prior results.
For more related work discussion, see Appendix~\ref{sec:relwork_full}.

\else
Our work is not the first to attempt to minimize the memory footprint of kernel approximation methods.
For RFFs, there has been work on using structured random projections \citep{fastfood,yu15,sphereRKS}, and feature selection \citep{sparseRKS, may2016} to reduce memory utilization.
Our work is orthogonal to these, because LP-RFFs can be used in conjunction with both.
For \NystromNS, there has been extensive work on improving the choice of landmark points, and reducing the memory footprint in other ways \citep{kmeans08,ensemble09,fastpred14,meka14,musco17}.
In our work, we study the effect of \textit{quantization} on generalization performance for RFFs under memory constraints.
We focus on RFFs because the feature generation component (landmark points and projection matrix) for \Nystrom is very memory-intensive, and thus limits the relative memory savings attainable via feature quantization.
For our initial experiments quantizing \Nystrom features, see Appendix~\ref{app:other_results}.

From a theoretical perspective, there has been a lot of work analyzing the generalization performance of kernel approximation methods \citep{rahimi08kitchen,cortes10,sutherland15,rudi17,avron17}.
The work most relevant to ours is that of \citet{avron17}, which defines $\Delta$-spectral approximation and bounds the generalization performance of kernel approximation methods in terms of $\Delta$.
This approach differs from works which evaluate kernel approximation methods in terms of the Frobenius or spectral norms of their kernel approximation matrices \citep{cortes10,gittens16,qmc,sutherland15,yu16,dao17}.
Our work shows the promise of \citeauthor{avron17}'s approach, and builds upon it.

On the topic of scaling kernel methods to large datasets, there have been a few notable recent papers.
\citet{block16} propose a distributed block coordinate descent method for solving large-scale least squares problems using the \Nystrom method or RFFs.
The recent work of \citet{may2017} uses a single GPU to train large RFF models on speech recognition datasets, %
showing comparable performance to fully-connected deep neural networks.
That work was limited by the number of features that could fit on a single GPU,
and thus our proposed method could help scale these results.

There has been much recent interest in the topic of low-precision for accelerating training and/or inference of machine learning models, as well as for model compression \citep{gupta15,hogwild15,hubara16,halp18,desa17,han15}. 
There have also been many advances in hardware support for low-precision \citep{tpu17,brainwave17}.
These improvements in hardware could dramatically speed up the training time of our method.

This work is inspired by the experiments comparing \Nystrom and RFFs under a memory budget in the PhD dissertation of one of the first authors \citep{maythesis}.
The current work provides a principled understanding of these prior results by showing that the relative performance of these methods can largely be explained in terms of our $(\Delta_1,\Delta_2)$ measure of kernel approximation error.
Based on this understanding, we propose LP-RFFs as a way of attaining improved generalization performance under a memory budget.
\fi
\vsp
\section{\uppercase{Conclusion}}
\label{sec:conclusion}
\vsp

\ifcamera
We defined a new measure of kernel approximation error and
\else
We defined a measure of kernel approximation error and---$(\Delta_1,\Delta_2)$-spectral approximation---and 
\fi
demonstrated its close connection to the empirical and theoretical generalization performance of kernel approximation methods.
\ifcamera
Inspired by this measure, we proposed LP-RFFs and showed they can attain improved generalization performance under a memory budget in theory and in experiments.
We believe these contributions provide fundamental insights into the generalization performance of kernel approximation methods, and hope to use these insights to scale kernel methods to larger and more challenging tasks.
\else
Inspired by this measure, we proposed low-precision random Fourier features (LP-RFFs) to attain high-rank approximations under a memory budget.
We showed that LP-RFFs can attain improved generalization performance under a memory budget, in theory and in experiments.
We believe these contributions provide fundamental insight into the field of kernel approximation,
and into which types of approximations lead to strong generalization performance.
We hope to use these advances to scale kernel methods to larger and more challenging tasks
across domains such as speech recognition, computer vision, and natural language processing.
\fi

\vsp

\subsubsection*{Acknowledgements}
\label{sec:ack}
\vsp
We thank Michael Collins for his helpful guidance on the \Nystrom vs.\ RFF experiments in Avner May's PhD dissertation \citep{maythesis}, which inspired this work.
We also thank Jared Dunnmon, Albert Gu, Beliz Gunel, Charles Kuang, Megan Leszczynski, Alex Ratner, Nimit Sohoni, Paroma Varma, and Sen Wu for their helpful discussions and feedback on this project.

We gratefully acknowledge the support of DARPA under Nos.\ FA87501720095 (D3M) and FA86501827865 (SDH), NIH under No.\ N000141712266 (Mobilize), NSF under Nos.\ CCF1763315 (Beyond Sparsity) and CCF1563078 (Volume to Velocity), ONR under No.\ N000141712266 (Unifying Weak Supervision), the Moore Foundation, NXP, Xilinx, LETI-CEA, Intel, Google, NEC, Toshiba, TSMC, ARM, Hitachi, BASF, Accenture, Ericsson, Qualcomm, Analog Devices, the Okawa Foundation, and American Family Insurance, and members of the Stanford DAWN project: Intel, Microsoft, Teradata, Facebook, Google, Ant Financial, NEC, SAP, and VMWare. The U.S.\ Government is authorized to reproduce and distribute reprints for Governmental purposes notwithstanding any copyright notation thereon. Any opinions, findings, and conclusions or recommendations expressed in this material are those of the authors and do not necessarily reflect the views, policies, or endorsements, either expressed or implied, of DARPA, NIH, ONR, or the U.S.\ Government.

\bibliography{ref}
\bibliographystyle{plainnat}
\clearpage

\onecolumn

\appendix

\section{\uppercase{Notation and background}}
\label{sec:background_appendix}
In this appendix, we first discuss the notation we use throughout the paper, and then provide an overview of random Fourier features (RFFs) \citep{rahimi07random} and the \Nystrom method \citep{nystrom}. 
After this, we briefly extend our discussion in Section~\ref{sec:refine} on fixed design kernel ridge regression.

\subsection{Notation}
We use $\{(x_i,y_i)\}_{i=1}^n$ to denote a training set, for $x_i \in \RR^d$, and $y_i \in \cY$, where $\cY = \RR$ for regression, and $\cY = \{1,\ldots,c\}$ for classification.
We let $K\in \RR^{n \times n}$ denote the kernel matrix corresponding to a kernel function $k\colon\RR^d\times\RR^d\rightarrow \RR$, where $K_{ij} = k(x_i,x_j)$, and let $\tK$ denote an approximation to $K$.
We let $z\colon\RR^d\rightarrow\RR^m$ denote a feature map for approximating a kernel function, such that $\tK_{ij} = z(x_i)^T z(x_j)$.
We use $s$ to denote the size of the mini-batches during training, and $b$ to denote the precision used for the random features.
We let $\|K\|_2$ and $\|K\|_F$ denote the spectral and Frobenius norms of a matrix $K$, respectively; if the subscript is not specified, $\|K\|$ denotes the spectral norm.
For vectors $x$, $\norm{x}$ will denote the $\ell_2$ norm of $x$, unless specified otherwise.
$I_n$ will denote the $n\times n$ identity matrix.
For symmetric matrices $A$ and $B$, we will say $A \preceq B$ if $B-A$ is positive semidefinite.
We will use $\lambda_i(A)$ to denote the $i^{th}$ largest eigenvalue of $A$, and $\lambda_{\max}(A)$, $\lambda_{\min}(A)$ to denote the largest and smallest eigenvalues of $A$, respectively.

\subsection{Kernel Approximation Background}
The core idea behind kernel approximation is to construct a feature map $z\colon \cX\rightarrow\RR$ such that $z(x)^T z(y) \approx k(x,y)$.
Given such a map, one can then learn a linear model on top of $\{(z(x_i),y_i)\}_{i=1}^n$, and this model will approximate the model trained using the exact kernel function.
We now review RFFs and the \Nystrom method, two of the most widely used and studied methods for kernel approximation.

\paragraph{Random Fourier features (RFFs)}
For shift-invariant kernels ($k(x,x') = \hat{k}(x-x')$), the random Fourier 
feature method \citep{rahimi07random} constructs a random feature representation 
$z(x) \in \RR^m$ such that $\expect{}{z(x)^T z(x')} = k(x,x')$.
This construction  is based on Bochner's Theorem, which states that any positive definite kernel is 
equal to the Fourier transform of a nonnegative measure.
This allows for performing Monte Carlo approximations of this Fourier transform in order to approximate the function.
The resulting features have the following functional form: 
$z_i(x) = \sqrt{2/m}\cos(w_i^Tx + a_i)$, where $w_i$ is drawn from the inverse Fourier
transform of the kernel function $\hat{k}$, and $a_i$ is drawn uniformly from $[0,2\pi]$ (see Appendix A in \citet{may2017} for a derivation).

One way of reducing the memory required for storing $W=[w_1,\ldots, w_m]$, is to replace $W$ by a structured matrix; in this work, we let $W$ be a concatenation of many square circulant random matrices \citep{yu15}.

\paragraph{\Nystrom method}
The \Nystrom method constructs a finite-dimensional feature representation
$z(x) \in \RR^m$ such that $\Dotp{z(x),z(x')} \approx k(x,x')$.  It does this
by picking a set of landmark points $\{\hx_1,\ldots,\hx_m\} \in \cX$,
and taking the SVD $\hK = U\Lambda U^T$ of the $m$ by $m$ 
kernel matrix $\hK$ corresponding to these landmark points 
($\hK_{i,j} = k(\hx_i,\hx_j)$).  The \Nystrom representation for a point $x \in \cX$
is defined as $z(x) = \Lambda^{-1/2} U^T k_x$, where $k_x = [k(x,\hx_1),\ldots,k(x,\hx_m)]^T$.
Letting $K_{m,n} = [k_{x_1},\ldots,k_{x_n}] \in \RR^{m\times n}$, 
the \Nystrom method can be thought of as an efficient low-rank approximation
$K \approx K_{m,n}^T U \Lambda^{-1/2}\Lambda^{-1/2} U^T K_{m,n}$ of the full
$n$ by $n$ kernel matrix $K$ corresponding to the full dataset $\{x_i\}_{i=1}^n$.
One can also consider the lower-dimensional \Nystrom representation
$z_r(x) = \Lambda_r^{-1/2} U_r^T k_x \in \RR^r$, where only the top $r$ eigenvalues and
eigenvectors of $\hK$ are used, instead of all $m$.  In this paper, we will always use $m=r$,
and thus will not specify the subscript $r$.

\subsection{Fixed Design Kernel Ridge Regression}
\label{subsec:app_fix_design}
We consider the problem of fixed design kernel ridge regression, which has a closed-form equation for the generalization error, making it a particularly tractable problem to analyze.
In fixed design regression, one is given a set of labeled points $\{(x_i,y_i)\}_{i=1}^n$, where $x_i \in \RR^d$, $y_i = \by_i + \eps_i \in \RR$, and the $\eps_i$ are zero-mean uncorrelated random variables with shared variance $\sigma^2 > 0$; here, the $\by_i$ represent the ``true labels.''
Given such a sample, the goal is to learn a regressor $f(x)$ such that $\cR(f) = \expect{\eps}{\frac{1}{n}\sum_{i=1}^n (f(x_i) - \by_i)^2}$ is small.
Note that for a fixed learning method, the learned regressor $f$ can be seen as a random function based on the random label noise $\eps_i$.
One approach to solving this problem is kernel ridge regression.
In kernel ridge regression, one chooses a kernel function $k:\RR^d\times\RR^d\rightarrow \RR$, and a regularizing constant $\lambda$, and learns a function of the form $f(x) = \sum_i \alpha_i k(x,x_i)$.
Letting $K\in\RR^{n\times n}$ denote the kernel matrix such that $K_{ij} = k(x_i,x_j)$, and $y = (y_1,\ldots,y_n)$, the closed-form solution for this problem (the one minimizing the regularized empirical loss), is $\alpha = (K+\lambda I)^{-1}y$.
It is then easy to show \citep{alaoui15} that the expected error of this regressor $f_K$ under the fixed design setting is
\begin{equation*}
\cR(f_K) = \frac{1}{n}\lambda^2 \by^T(K+\lambda I)^{-2}\by + \frac{1}{n}\sigma^2 Tr\Big(K^2(K+\lambda \id)^{-2}\Big),
\end{equation*}
where $\by = (\by_1,\ldots,\by_n)$ is the vector of ``true labels.''

\section{\uppercase{Generalization bounds for fixed design regression}}
\label{sec:app_generalization_bound}
\subsection{Generalization Bound in Terms of $(\Delta_1,\Delta_2)$}
\label{subsec:generalization_and_rel_spec_dist}

\begin{customprop}{1}{(Extended from \citep{avron17})}
	Suppose $\tK + \lambda I$ is $(\Delta_1, \Delta_2)$-spectral approximation of $K+\lambda I$, for $\Delta_1 \in [0,1)$ and $\Delta_2 \geq 0$. Let $m$ denote the rank of $\tK$, and let $f_{K}$ and $f_{\tK}$ be the kernel ridge regression estimators learned using these matrices, with regularizing constant $\lambda \geq 0$ and label noise variance $\sigma^2 < \infty$. Then
	\begin{equation}
	\cR(f_{\tK}) \leq \frac{1}{1-\Delta_1} \hcR(f_K) +  \frac{\Delta_2}{1+\Delta_2}\frac{m}{n}\sigma^2,
	\label{eq:risk_bound_app}
	\end{equation}
	where $\cR$ (expected risk) and $\hcR$ (upper bound on $\cR$) are defined in Section~\ref{subsec:delta12}.

	\label{prop:alphabeta_app}
\end{customprop}

\begin{proof}
	This proof closely follows the proof of Lemma 2 in \citet{avron17}, with the primary difference being that we replace $(1-\Delta)$ and $(1+\Delta)$ with $(1-\Delta_1)$ and $(1+\Delta_2)$, respectively.

	We begin by replacing $K$ with $\tK$ in the definition for $\hcR(f_{K})$:
	$$\cR(f_{\tK}) \leq  \frac{1}{n}\lambda \by^T(\tK+\lambda I)^{-1}\by + \frac{1}{n}\sigma^2 \tr\Big(\tK(\tK+\lambda I)^{-1}\Big) =  \hcR(f_{\tK}).$$
	We now continue this chain of inequalities, using the fact that $A\preceq B$ implies $B^{-1}\preceq A^{-1}$.  Thus, $(1-\Delta_1)(K+\lambda I) \preceq \tK + \lambda I \Rightarrow (\tK + \lambda I)^{-1} \preceq (1-\Delta_1)^{-1}(K+\lambda I)^{-1} \Rightarrow \by^T(\tK+\lambda I)^{-1}\by \leq  (1-\Delta_1)^{-1} \by^T(K+\lambda I)^{-1} \by$.  This upper bounds the first term in the above sum.
	
	We now consider the second term.  Let $m = \rank(\tK)$, and let $s_{\lambda}(\tK) = \tr\Big(\tK(\tK+\lambda I)^{-1}\Big)$.
	
	\begin{eqnarray*}
		s_{\lambda}(\tK) &=& \tr\Big(\tK(\tK+\lambda I)^{-1}\Big) \\
		&=& \sum_{i=1}^m \frac{\lambda_i(\tK)}{\lambda_i(\tK) + \lambda} \\
		&=& m -\sum_{i=1}^m \frac{\lambda}{\lambda_i(\tK) + \lambda} \\
		&\leq& m -\sum_{i=1}^m \frac{\lambda}{(1+\Delta_2)(\lambda_i(K) + \lambda)} \\
		&=& m -(1 + (1+\Delta_2)^{-1} - 1)\sum_{i=1}^m \frac{\lambda}{\lambda_i(K) + \lambda} \\
		&=& m - \sum_{i=1}^m \frac{\lambda}{\lambda_i(K) + \lambda} + \frac{\Delta_2}{1+\Delta_2}\sum_{i=1}^m\frac{\lambda}{\lambda_i(K) + \lambda} \\
		&\leq& n - \sum_{i=1}^n \frac{\lambda}{\lambda_i(K) + \lambda} + \frac{\Delta_2}{1+\Delta_2}m \\
		&=& s_{\lambda}(K) +\frac{\Delta_2}{1+\Delta_2}m
	\end{eqnarray*}
	\begin{eqnarray*}
		&\leq& \frac{1}{1-\Delta_1}  s_{\lambda}(K) +\frac{\Delta_2}{1+\Delta_2}m \\
	\end{eqnarray*}
	
	Combining the above results, we get that:
	
	\begin{eqnarray*}
		\cR(f_{\tK}) &\leq&  \frac{1}{n}\lambda \by^T(\tK+\lambda I)^{-1}\by + \frac{1}{n}\sigma^2 s_{\lambda}(\tK) \\
		&\leq& \frac{1}{n}\lambda \bigg(\frac{1}{1-\Delta_1} \by^T(K+\lambda I)^{-1} \by\bigg) +  \frac{1}{n}\sigma^2 \bigg(\frac{1}{1-\Delta_1}  s_{\lambda}(K) +\frac{\Delta_2}{1+\Delta_2}m\bigg)\\
		&=& \frac{1}{1-\Delta_1} \bigg(\frac{1}{n}\lambda \by^T(K+\lambda I)^{-1}\by + \frac{1}{n}\sigma^2 s_{\lambda}(K)\bigg) +  \frac{\Delta_2}{1+\Delta_2}\frac{m}{n}\sigma^2 \\
		&=& \frac{1}{1-\Delta_1} \hcR(f_K) +  \frac{\Delta_2}{1+\Delta_2}\frac{m}{n}\sigma^2
	\end{eqnarray*}
\end{proof}

\paragraph{Remark} Above, $\Delta_1 \in [0,1]$ and $\Delta_2 \geq 0$. Note that as $\Delta_1$ approaches 1, the above upper bound diverges to infinity.  Whereas as $\Delta_2$ approaches $\infty$, the second term in the upper bound approaches $\frac{m}{n}\sigma^2$. This suggests that choosing $\tK$ and $\lambda$ such that $\Delta_1$ does not get too close to 1 is very important. A necessary condition for $(1-\Delta_1)(K+\lambda I)\preceq \tK + \lambda I$ is for $\tK$ to be high rank, as discussed in Section~\ref{sec:refine}.

\subsection{Heuristic Effort to Better Understand Influence of $(\Delta_1,\Delta_2)$ on Generalization Error}
\label{subsec:d1_d2_generalization}
Here we present a heuristic argument to explain how $\Delta_1$ and $\Delta_2$ in the
spectral approximation affects the generalization error.
We are particularly interested in demonstrating that $\Delta_2$ can have an important influence on the bias squared term ($\frac{\lambda^2}{n} \by^T(\tilde{K}+\lambda I)^{-2}\by$) of the generalization error,
even though the upper bound $\frac{\lambda^2}{n} \by^T(\tilde{K}+\lambda I)^{-2}\by \leq \frac{1}{1-\Delta_1} \hcR(f_K)$ on the bias squared term (from Proposition~\ref{prop:alphabeta}) is only in terms of $\Delta_1$.

Suppose that the approximate kernel matrix $\tilde{K}$ is a
$(\Delta_1, \Delta_2)$-spectral approximation of the true kernel matrix $K$, that is:
\begin{equation*}
(1 - \Delta_1) (K + \lambda I_n) \preceq \tilde{K} + \lambda I_n \preceq (1 + \Delta_2) (K + \lambda I_n).
\end{equation*}
We focus on the bias squared term of $\tilde{K}$ ($\frac{\lambda^2}{n} \bar{y}^T (\tK + \lambda I_n)^{-2} \bar{y}$), and compare it to the bias squared
term of $K$.
Following Theorem 15 of~\citet{musco17}, we first bound the bias (not squared):
\begin{align*}
\norm{(\tilde{K} + \lambda I_n)^{-1} \bar{y}}
&\leq \norm{(K + \lambda I_n)^{-1} \bar{y}} + \norm{((\tilde{K} + \lambda I_n)^{-1} - (K + \lambda
	I_n)^{-1}) \bar{y}}  \\ %
&= \norm{(K + \lambda I_n)^{-1} \bar{y}} + \norm{(\tilde{K} + \lambda I_n)^{-1} ((K + \lambda
	I_n) - (\tilde{K} + \lambda I_n)) (K + \lambda I_n)^{-1} \bar{y}} \\
&= \norm{(K + \lambda I_n)^{-1} \bar{y}} + \norm{(\tilde{K} + \lambda I_n)^{-1} (K -
	\tilde{K}) (K + \lambda I_n)^{-1} \bar{y}} \\
&\leq \norm{(K + \lambda I_n)^{-1} \bar{y}} + \norm{(\tilde{K} + \lambda I_n)^{-1} (K -
	\tilde{K})} \norm{(K + \lambda I_n)^{-1} \bar{y}} \\ %
&= \norm{(K + \lambda I_n)^{-1} \bar{y}} \left( 1 + \norm{(\tilde{K} + \lambda I_n)^{-1}
	(K - \tilde{K})} \right). \numberthis \label{eqn:bias_bound}
\end{align*}
Now it reduces to bounding $\norm{(\tilde{K} + \lambda I_n)^{-1} (K - \tilde{K})}$.
As $\tilde{K} + \lambda I_n \preceq (1 + \Delta_2) (K + \lambda I_n)$, we have
\begin{equation*}
\tilde{K} - K \preceq \Delta_2 (K + \lambda I_n) \preceq \frac{\Delta_2}{1 - \Delta_1}  (\tilde{K} + \lambda I_n) \preceq \frac{1 + \Delta_2}{1 - \Delta_1} (\tilde{K} + \lambda I_n).
\end{equation*}
Similarly, since $K + \lambda I_n \preceq \frac{1}{1 - \Delta_1} (\tilde{K} + \lambda I_n)$,
\begin{equation*}
K - \tilde{K} \preceq \frac{\Delta_1}{1 - \Delta_1} (\tilde{K} + \lambda I_n) \preceq \frac{1 + \Delta_2}{1 - \Delta_1} (\tilde{K} + \lambda I_n).
\end{equation*}
Hence
\begin{equation*}
- \frac{1 + \Delta_2}{1 - \Delta_1} (\tilde{K} + \lambda I_n) \preceq K - \tilde{K} \preceq \frac{1 + \Delta_2}{1 - \Delta_1} (\tilde{K} + \lambda I_n).
\end{equation*}
Because $t \mapsto t^2$ is not operator monotone, it is not easy to obtain a bound on
$(K - \tilde{K})^2$ from the bound on $K - \tilde{K}$.
However, under a restricted setting where $K$ and $\tilde{K}$ have the same eigenvectors,
we can square the corresponding eigenvalues to obtain
\begin{equation*}
(K - \tilde{K})^2 \preceq \frac{(1 + \Delta_2)^2}{(1 - \Delta_1)^2} (\tilde{K} + \lambda I_n)^2.
\end{equation*}
Thus
\begin{equation*}
(\tilde{K} + \lambda I_n)^{-1} (K - \tilde{K})^2 (\tilde{K} + \lambda I_n)^{-1} \preceq \frac{(1 + \Delta_2)^2}{(1 - \Delta_1)^2}.
\end{equation*}
And hence $\norm{(\tilde{K} + \lambda I_n)^{-1} (K - \tilde{K})} \leq \frac{1 + \Delta_2}{1 - \Delta_1}$.
Plugging this into the bound (Eq.~\ref{eqn:bias_bound}) yields
\begin{equation*}
\norm{(\tilde{K} + \lambda I_n)^{-1} \bar{y}}
\leq \left( 1 + \frac{1 + \Delta_2}{1 - \Delta_1} \right) \norm{(K + \lambda I_n)^{-1} \bar{y}}.
\end{equation*}
Thus
\begin{eqnarray*}
\frac{\lambda^2}{n} \bar{y}^T (\tilde{K} + \lambda I_n)^{-2} \bar{y}
&=& \frac{\lambda^2}{n} \norm{(\tilde{K} + \lambda I_n)^{-1} \bar{y}}^2 \\
&\leq& \left( 1 + \frac{1 + \Delta_2}{1 - \Delta_1} \right)^2 \frac{\lambda^2}{n} \norm{(K + \lambda I_n)^{-1} \bar{y}}^2\\
&=& \left( 1 + \frac{1 + \Delta_2}{1 - \Delta_1} \right)^2 \frac{\lambda^2}{n} \bar{y}^T (K + \lambda I_n)^{-2} \bar{y}.
\end{eqnarray*}
In other words, in this restricted setting the bias squared of $\tilde{K}$ is at most a factor
$(1 + \frac{1 + \Delta_2}{1 - \Delta_1})^2$ larger than the bias squared of $K$.
Though this heuristic analysis only holds when $K$
and $\tilde{K}$ have the same eigenvectors, it reveals the dependency of the
generalization performance on $\Delta_1$ and $\Delta_2$; 
in particular, it reveals that $\Delta_2$ could have an important influence on the bias squared term of the generalization error.

\subsection{The Empirical Influence of $\Delta_2$ on the Bias Squared Term}
\label{subsec:delta2_bias_empirical}
We now empirically validate that $\Delta_2$ can have a large impact on the bias squared term, as suggested by the theoretical discussion in the previous section.
The influence of $\Delta_2$ on the bias squared term helps explain our empirical observations on the influence of $\Delta_2$ on generalization performance from 
Section~\ref{subsec:gen_perf_and_delta}.
Though the generalization bound in Proposition~\ref{prop:alphabeta} suggests that performance should scale roughly linearly in $\Delta_2/(1+\Delta_2)$, we empirically found that generalization performance does not asymptote as $\Delta_2$ grows (as $\Delta_2/(1+\Delta_2)$ would suggest it would).
In this section, we empirically validate our hypothesis that this is due to looseness in the generalization bound.
Specifically, the expected mean squared error for fixed design kernel ridge regression is 
\[\cR(f_{\tilde{K}})=\frac{\lambda^2}{n} \by^T(\tilde{K}+\lambda I)^{-2}\by + \frac{\sigma^2}{n} \tr\Big(\tilde{K}^2(\tilde{K}+\lambda \id)^{-2}\Big),
\] where $\tilde{K}$ is an approximate kernel matrix.
We show in experiments that the bias squared term ($\frac{\lambda^2}{n} \by^T(\tilde{K}+\lambda I)^{-2}\by$) can be strongly influenced by $\Delta_2$,
even though the upper bound on it ($\frac{\lambda^2}{n} \by^T(\tilde{K}+\lambda I)^{-2}\by \leq \frac{1}{1-\Delta_1} \hcR(f_K)$) in Proposition~\ref{prop:alphabeta} is only in terms of $\Delta_1$.
\begin{figure}
	\centering
	\includegraphics[width=0.4\linewidth]{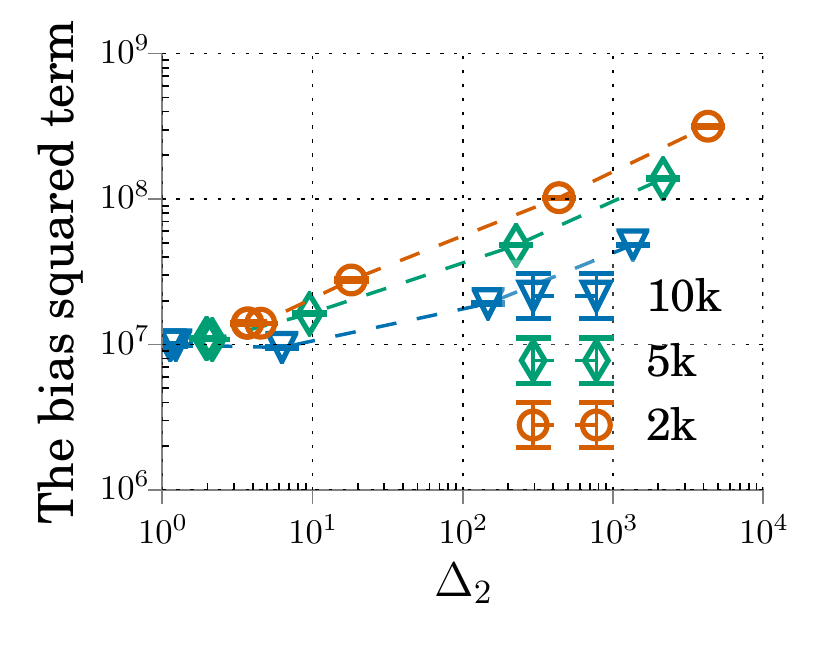}
	\caption{In the fixed design setting, the bias squared term of the generalization performance grows with $\Delta_2$.}
	\label{fig:bias_delta_right}
\end{figure}

In our experiments, we compute the value of the bias squared term and $\Delta_2$ on the Census dataset.
To gain statistically meaningful insights, we collect and average the value of $\frac{\lambda^2}{n} \by^T(\tilde{K}+\lambda I)^{-2}\by$ and $\Delta_2$ using 3 independent runs with different random seeds.
In Figure~\ref{fig:bias_delta_right}, we plot the value of $\frac{\lambda^2}{n} \by^T(\tilde{K}+\lambda I)^{-2}\by$ as a function of $\Delta_2$ for 3 different numbers of features;
by controlling the number of features, we can demonstrate the influence of $\Delta_2$ while $\Delta_1$ is held roughly fixed.
In each curve in Figure~\ref{fig:bias_delta_right}, the data points are collected from FP-RFFs, circulant FP-RFFs, as well as LP-RFFs using $\{1,2,4,8\}$ bit precision.
We can see that for each number of features, the value of $\frac{\lambda^2}{n} \by^T(\tilde{K}+\lambda I)^{-2}\by$ grows with $\Delta_2$.
These results demonstrate that the upper bound on the bias term in Proposition~\ref{prop:alphabeta} is quite loose, and is not capturing the influence of $\Delta_2$ properly.

\section{\uppercase{Theoretical guarantees for LP-RFFs}}
\label{sec:lprff_theory_appendix}
In this section, we first lower bound the probability that using LP-RFFs results in a kernel approximation matrix that is a $(\Delta_1,\Delta_2)$-spectral approximation of the exact kernel matrix (Section~\ref{sec:theory_d1_d2_app}).
We then present bounds on the Frobenius kernel approximation error for LP-RFFs (Section~\ref{sec:app_kernel_error}).

\subsection{$(\Delta_1,\Delta_2)$-Spectral Approximation Bounds for LP-RFFs}
\label{sec:theory_d1_d2_app}
As usual, let $K \in \mathbb{R}^{n \times n}$ denote the kernel matrix, and $Z \in \mathbb{R}^{n \times m}$ be the random Fourier feature matrix, where $n$ is
the number of data points and $m$ is the number of features.
We can write $Z = \frac{1}{\sqrt{m}} \begin{bmatrix} z_1, \dots, z_m \end{bmatrix}$ 
where $z_i$ are the (scaled) columns of $Z$.
Each entry of $Z$ has the form $\sqrt{2/m} \cos(w^T x + a)$ for some $w, x \in
\mathbb{R}^{d}$ and $a \in \mathbb{R}^{m}$, where $d$ is the dimension of the
original dataset.  Then $\E[z_i z_i^T] = K$, so $\E[Z Z^T] = K$.

Now suppose we quantize $Z$ to $b$ bits using the quantization method described in Section~\ref{subsec:method_details}, for some fixed $b \geq 1$.
Then the quantized feature matrix is $Z + C$ for some random 
$C \in \mathbb{R}^{n \times m}$ whose entries are independent conditioned on $Z$ 
(but not identically distributed) with $\E[C \mid Z] = 0$.
We can write $C = \frac{1}{\sqrt{m}} \begin{bmatrix} c_1, \dots, c_m \end{bmatrix}$
where $c_i$ are the (scaled) columns of $C$.
Moreover, the $c_i$ are independent conditioned on $Z$.
Defining $\delta^2_b \defeq \frac{2}{(2^b - 1)^2}$, 
the entries $C_{ij}$ have variance $\E[C_{ij}^2 \mid Z_{ij}] \leq \delta_b^2/m$ by Proposition~\ref{prop:var_bound} in Appendix~\ref{sec:app_kernel_error}.
In terms of the vectors $c_i$, we can also see that $\E[c_{i,j}^2 \mid Z_{ij}] \leq \delta_b^2$,
where $c_{i,j}$ denotes the $j^{th}$ element of $c_i$.

We first analyze the expectation of $(Z + C) (Z + C)^T$ (over both the
randomness of $Z$ and of $C$).
\begin{lemma}
$\E[(Z + C)(Z + C)^T] = K + D$, where $D \defeq \E[c_1 c_1^T] = s_b I_n$ is a multiple of the identity, for $0 \leq s_b \leq \delta_b^2$. $D$ does not depend on the number of random features $m$.
\label{lem:expectation_CCstar}
\end{lemma}

\begin{proof}
\begin{equation*}
\E[(Z + C) (Z + C)^T] = \E\bigg[\frac{1}{m} \sum_{i=1}^{m} (z_i + c_i) (z_i + c_i)^T\bigg] = 
\E[(z_1 + c_1) (z_1 + c_1)^T]
\end{equation*}
Since $\E_{c_1}[c_1 \mid z_1] = 0$, it follows that
\begin{eqnarray*}
\E[(z_1 + c_1) (z_1 + c_1)^T] &=& \E_{z_1}\Big[\E_{c_1}[(z_1 + c_1) (z_1 + c_1)^T \mid z_1]\Big] \\
&=& \E_{z_1}[z_1z_1^T] + \E_{z_1}\Big[\E_{c_1}[c_1 c_1^T \mid z_1]\Big] \\
&=& K + \E[c_1 c_1^T]
\end{eqnarray*}
  It is clear that $D\defeq \E[c_1 c_1^T]$ is a diagonal matrix, because each element of $c_1$ is a zero-mean independent random variable.
  It is also easy to see that the $j^{th}$ entry on the diagonal of $D$ is equal to $\E[c_{1,j}^2 \mid z_{1,j}] \leq \delta_b^2$.
  Lastly, we argue that each element $z_{1,j}=\sqrt{2}\cos(w_1^T x_j + a_1)$ has the same distribution, because it is distributed the same as $\sqrt{2}\cos(a_1)$ for $a_1$ uniform in $[0,2\pi]$.
  Thus, $\E[c_{1,j}^2 \mid z_{1,j}]$ is independent of $j$. Letting $s_b \defeq \E[c_{1,1}^2 \mid z_{1,1}] \leq \delta_b^2$ completes the proof.
 
\end{proof}

With $D \defeq \E[c_1 c_1^T]$, we can use matrix concentration to show that the
quantized kernel matrix $\tK = (Z + C)(Z + C)^T$ is close to its
expectation $K + D$.
We first strengthen the matrix Bernstein inequality with intrinsic dimension
(Theorem 7.7.1 in \citet{tropp2015introduction}) by removing the requirement on
the deviation.\footnote{Theorem 7.7.1 in \citet{tropp2015introduction} requires that $t \geq \sqrt{v} + L/3$, where $t$, $v$, and $L$ are as defined in Theorem~\ref{thm:bernstein}.}

\begin{theorem}[Matrix Bernstein: Hermitian Case with Intrinsic Dimension] \label{thm:intdim-bernstein-herm}
Consider a finite sequence $\{ X_k \}$ of random Hermitian matrices of the same size,
and assume that
\begin{equation*}
  \E [X_k] = 0
  \quad\text{and}\quad
  \lambda_{\max}(X_k) \leq L
  \quad\text{for each index $k$.}
\end{equation*}
Introduce the random matrix
\begin{equation*}
  Y = \sum\nolimits_k X_k.
\end{equation*}
Let $V$ be a semidefinite upper bound for the matrix-valued variance $\var{}{Y}$:
\begin{equation*}
  V \succeq \var{}{Y} = \E [Y^2] = \sum\nolimits_k \E [X_k^2].
\end{equation*}
Define the intrinsic dimension bound and variance bound
\begin{equation*}
  \operatorname{intdim}(V) = \frac{\tr(V)}{\norm{V}}
  \quad\text{and}\quad
  v = \norm{V}.
\end{equation*}
Then, for $t \geq 0$,
\begin{equation} \label{eqn:intdim-bernstein-tail}
\Prob \left( \lambda_{\max}(Y) \geq t \right)
	\leq 4 \operatorname{intdim}(V) \cdot \exp\left( \frac{-t^2/2}{v + Lt/3} \right).
\end{equation}
\label{thm:bernstein}
\end{theorem}

\begin{proof}
  The case of $t \geq \sqrt{v} + L/3$ is exactly Theorem 7.7.1 in \citet{tropp2015introduction}.
  We just need to show that the bound is vacuous when $0 \leq t < \sqrt{v} + L/3$.

  Suppose that $0 \leq t < \sqrt{v} + L/3$.
  We show that then $t^2 \leq 2(v + Lt/3)$.
  Indeed, $t^2 - 2Lt/3 - 2v$ has roots $\frac{L}{3} \pm \sqrt{\frac{L^2}{9} + 2v}$.
  The condition $t^2 \leq 2(v + Lt/3)$ is then equivalent to
  \begin{equation*}
    \frac{L}{3} - \sqrt{\frac{L^2}{9} + 2v} \leq t \leq \frac{L}{3} +
    \sqrt{\frac{L^2}{9} + 2v}.
  \end{equation*}
  The lower bound is negative since $v \geq 0$, and $t < \sqrt{v} + L/3$ implies $t
  < L/3 + \sqrt{L^2/9 + 2v}$, satisfying the upper bound.
  Thus $0 \leq t < \sqrt{v} + L/3$ implies that $t^2 \leq 2(v + Lt/3)$.
  The bound in equation~\eqref{eqn:intdim-bernstein-tail} becomes
  \begin{equation*}
    4\operatorname{intdim}(V) \exp \left( -\frac{t^2/2}{v + Lt/3} \right) \geq 4\operatorname{intdim}(V) \exp (-1) \geq 4/e > 1,
  \end{equation*}
  since $\operatorname{intdim}(V) \geq 1$.
  Thus~\eqref{eqn:intdim-bernstein-tail} holds vacuously for $0 \leq t < \sqrt{v} +
  L/3$.
\end{proof}

We now present Lemma~\ref{lem:quantized_concentration_two_sided}, in which we lower bound the probability that $\tK$ is ``close'' to its expectation $K+D$, in the specific sense we describe below.
\begin{lemma}
  Let $K$ be an exact kernel matrix, and $\tK = (Z+C)(Z+C)^T$ be an $m$-features $b$-bit LP-RFF approximation of $K$ with expectation $K+D$.
  For any deterministic matrix $B$, let $L \defeq 2n \norm{B}^2$ and $M \defeq B(K + \delta^2_b I_n) B^T$,
  then for any $t_1, t_2 \geq 0$,
  \begin{align*}
    &\Prob\bigg[- t_1 I_n \preceq B\Big((Z + C)(Z + C)^T - (K+D)\Big)B \preceq t_2 I_n\bigg] \\
    \geq\ &1 - \frac{4\tr(M)}{\norm{M}} \left[ \exp \left( \frac{-mt_1^2}{2L(\norm{M} +
        2t_1/3)} \right) + \exp \left(\frac{-mt_2^2}{2L(\norm{M} + 2t_2/3)} \right)  \right].
  \end{align*}
  \label{lem:quantized_concentration_two_sided}
\end{lemma}

\begin{proof}
  Let $S_i = \frac{1}{m} \left( B (z_i + c_i) (z_i + c_i)^T B^T - B (K + D)B^T
  \right)$ and $S = \sum_{i=1}^{m} S_i = B \left( (Z + C) (Z + C)^T - (K + D)
  \right)B$.
  We see that $\E[S_i] = 0$.
  We will bound $\lambda_{\max}(S)$ and $\lambda_{\min}(S)$ by applying the matrix Bernstein
  inequality for symmetric matrices with intrinsic dimension (Theorem~\ref{thm:bernstein}).
  Thus we need to bound $\norm{S_i}$ and $\norm{\sum_{i=1}^{m} \E[S_i^2]}$.

  Let $u_i = B(z_i + c_i) \in \mathbb{R}^{n}$, then $S_i = \frac{1}{m} (u_i
  u_i^T  - \E[u_i u_i^T])$.
  We first bound $\norm{u_i u_i^T}$.
  Since this is a rank 1 matrix,
  \begin{equation*}
    \norm{u_i u_i^T} = \norm{u_i}^2 = \norm{B(z_i + c_i)}^2 \leq \norm{B}^2
    \norm{z_i + c_i}^2 \leq 2n \norm{B}^2,
  \end{equation*}
  where we have used the fact that $z_i + c_i$ is a vector of length $n$ whose
  entries are in $[-\sqrt{2}, \sqrt{2}]$.
  This gives a bound on $\norm{S_i}$:
  \begin{equation*}
    \norm{S_i} = \frac{1}{m} \norm{u_i u_i^T - \E[u_i u_i^T]}
    \leq \frac{1}{m} \norm{u_i u_i^T} + \frac{1}{m} \E \norm{u_i u_i^T}
    \leq \frac{4n \norm{B}^2}{m} = 2L/m.
  \end{equation*}
  Thus $\lambda_{\max}(S_i) \leq 2L/m$ and $\lambda_{\max}(-S_i) =
  -\lambda_{\min}(S_i) \leq 2L/m$.

  Now it's time to bound $\E[S_i^2]$.  We will use
  \begin{eqnarray*}
    \E[S_i^2] &=& \frac{1}{m^2}\Big(\E\big[(u_i u_i^T)^2\big] - \E\big[u_iu_i^T\big]^2\Big) \preceq \frac{1}{m^2} \E[(u_i u_i^T)^2]
    = \frac{1}{m^2} \E[u_i u_i^T u_i u_i^T] \\
    &=& \frac{1}{m^2} \E[\norm{u_i}^2 u_i
    u_i^T] \preceq \frac{2n \norm{B}^2}{m^2} \E[u_i u_i^T].
  \end{eqnarray*}
  Thus
  \begin{equation*}
    \sum_{i=1}^{m} \E[S_i^2] \preceq \frac{2n \norm{B}^2}{m} \E[v_1 v_1^T] = \frac{2n
      \norm{B}^2}{m} B(K + D) B^T \preceq \frac{2n \norm{B}^2}{m} B(K + \delta^2_b I_n)B^T
    = LM/m.
  \end{equation*}

  Applying Theorem~\ref{thm:intdim-bernstein-herm} with $S$, for any $t_2 \geq 0$,
  we have
  \begin{equation*}
    \Prob\bigg[\lambda_{\max}(B\Big((Z + C)(Z + C)^T - (K+D)\Big)B) \succeq t_2 I_n \bigg] \leq \frac{4\tr(M)}{\norm{M}}
    \exp \left( \frac{-mt_2^2}{2L(\norm{M} + 2t_2/3)} \right).
  \end{equation*}
  Similarly, applying Theorem~\ref{thm:intdim-bernstein-herm} with $-S$ and
  using the fact that $\lambda_{\max}(-S) = -\lambda_{\min}(S)$, for any $t_1 \geq 0$,
  we have
  \begin{equation*}
    \Prob\bigg[\lambda_{\min}(B\Big((Z + C)(Z + C)^T - (K+D)\Big)B) \preceq -t_1 I_n \bigg] \leq \frac{4\tr(M)}{\norm{M}}
    \exp \left( \frac{-mt_1^2}{2L(\norm{M} + 2t_1/3)} \right).
  \end{equation*}
  Combining the two bounds with the union bound yields the desired inequality.
\end{proof}

We are now ready to show that low-precision features yield close spectral
approximation to the exact kernel matrix.

\begin{customthm}{2}
	Let $\tK$ be an $m$-feature $b$-bit LP-RFF approximation of a kernel matrix $K$, and assume
	$\norm{K} \geq \lambda \geq \delta^2_b \defeq 2/(2^b-1)^2$.
	Then for any $\Delta_1 \geq 0$, $\Delta_2 \geq \delta^2_b/\lambda$,
	\begin{align*}
	\Prob\Big[(1-\Delta_1)&(K+\lambda I) \preceq \tK+\lambda I \preceq (1+\Delta_2)(K+\lambda I)\Big] \geq \\
	&1-8 \tr \left( (K + \lambda I_n)^{-1}(K + \delta^2_bI_n)\right)\left(
	\exp\left(\frac{-m\Delta_1^2}{\frac{4n}{\lambda}(1 + \frac{2}{3}\Delta_1)} \right) +
	\exp\left(\frac{-m(\Delta_2 - \frac{\delta^2_b}{\lambda})^2}{\frac{4n}{\lambda}(1 + \frac{2}{3}(\Delta_2 - \frac{\delta^2_b}{\lambda}))}\right)\right).
	\end{align*}
	\label{thm2_app}
\end{customthm}

\begin{proof}
  We conjugate the desired inequality with $B \defeq (K + \lambda I_n)^{-1/2}$ (i.e.,
  multiply by $B$ on the left and right), noting that semidefinite ordering is
  preserved by conjugation:
  \begin{align*}
    &(1 - \Delta_1) (K + \lambda I_n) \preceq \tilde{K} + \lambda I_n \preceq (1 + \Delta_2) (K + \lambda I_n) \\
    \iff\ &(1 - \Delta_1) I_n \preceq B (\tilde{K} + \lambda I_n) B \preceq (1 + \Delta_2) I_n \\
    \iff\ &-\Delta_1 I_n \preceq B (\tilde{K} + \lambda I_n) B - I_n \preceq \Delta_2 I_n \\
    \iff\ &-\Delta_1 I_n \preceq B (\tilde{K} + \lambda I_n - K - \lambda I_n) B \preceq \Delta_2 I_n \\
    \iff\ &-\Delta_1 I_n \preceq B (\tilde{K} - K) B \preceq \Delta_2 I_n.
  \end{align*}

  We show that $-\Delta_1 I_n \preceq B (\tilde{K} - K - D) B \preceq
  (\Delta_2 - \delta^2_b/\lambda) I_n$ implies $-\Delta_1 I_n \preceq B (\tilde{K} - K) B \preceq \Delta_2 I_n$.
  Indeed, $\norm{BDB} \leq \delta^2_b/\lambda$ because
  $\norm{B}^2 = \norm{B^2} = \norm{(K+\lambda I_n)^{-1}} \leq 1/\lambda$ and
  $\norm{D} \leq \delta_b^2$, so $BDB \preceq (\delta^2_b/\lambda) I_n$.
  Moreover, since $D \succeq 0$ by Lemma~\ref{lem:expectation_CCstar} and $B$ is symmetric,
  $B D B \succeq 0$.
  Thus the condition $-\Delta_1 I_n \preceq B (\tilde{K} - K - D) B \preceq
  (\Delta_2 - \delta^2_b/\lambda) I_n$ implies:
  \begin{align*}
    B(\tilde{K} - K)B &= B(\tilde{K} - K - D)B + BDB \preceq (\Delta_2 - \delta^2_b/\lambda)I_n + \delta^2_b/\lambda I_n = \Delta_2 I_n, \\
    B(\tilde{K} - K)B &= B(\tilde{K} - K - D)B + BDB \succeq -\Delta_1I_n + 0 = -\Delta_1 I_n.
  \end{align*}
  Hence
  $\Prob \left[ -\Delta_1 I_n \preceq B (\tilde{K} - K) B \preceq \Delta_2 I_n \right] \geq \Prob \left[ -\Delta_1 I_n \preceq B (\tilde{K} - K - D) B \preceq (\Delta_2 - \delta^2_b/\lambda) I_n \right]$.
  It remains to show that $-\Delta_1 I_n \preceq B (\tilde{K} - K - D) B \preceq (\Delta_2
  - \delta^2_b/\lambda) I_n$ with the desired probability, by applying
  Lemma~\ref{lem:quantized_concentration_two_sided}.
  We apply Lemma~\ref{lem:quantized_concentration_two_sided} for $B \defeq (K + \lambda I_n)^{-1/2}$,
   $L \defeq 2n \norm{B}^2 \leq 2n/\lambda$, and $M \defeq B(K + \delta^2_b I_n) B$.

  To simplify the bound one gets from applying Lemma~\ref{lem:quantized_concentration_two_sided} with the above $B$, $L$, and $M$, we will use the following expression for $\tr(M)$, and the following upper and lower bounds on $\|M\|$.
  $\tr(M) = \tr\Big(B^2(K + \delta^2_b I_n)\Big) = \tr\Big((K + \lambda I_n)^{-1}(K + \delta^2_b I_n)\Big)$.
  Letting $K=USU^T$ be the SVD of $K$, we get that $M = U(S+\lambda I_n)^{-1/2} U^T U
  (S+\delta^2_b I_n) U^T U (S+\lambda I_n)^{-1/2} U^T = U(S+\lambda I_n)^{-1} (S+\delta^2_b I_n) U^T$.
  Thus, letting $\lambda_1$ be the largest eigenvalue of $K$ (recall $\lambda_1 \geq \lambda$ by
  assumption), $\norm{M} = (\lambda_1 + \delta^2_b)/(\lambda_1 + \lambda) \geq (\lambda_1 + \delta^2_b)/(2\lambda_1) \geq  1/2$.
  We also assume that $\delta^2_b \leq \lambda$, so $\norm{M} \leq 1$.
  Lemma~\ref{lem:quantized_concentration_two_sided} allows us to conclude the following:
  \begin{align*}
&\Prob\bigg[(1 - \Delta_1) (K + \lambda I_n) \preceq \tilde{K} + \lambda I_n \preceq (1 + \Delta_2) (K + \lambda I_n)\bigg] \\
=& \;\;\Prob\bigg[-\Delta_1 I_n \preceq B (\tilde{K} - K) B \preceq \Delta_2 I_n\bigg]\\
\geq& \;\; \Prob \left[ -\Delta_1 I_n \preceq B (\tilde{K} - (K + D)) B \preceq (\Delta_2 - \delta^2_b/\lambda) I_n \right] \\
\geq& \;\;1 - \frac{4\tr(M)}{\norm{M}} \left[ \exp \left( \frac{-m\Delta_1^2}{2L(\norm{M} + 2\Delta_1/3)} \right) + \exp
\left(\frac{-m(\Delta_2 - \delta^2_b/\lambda)^2}{2L(\norm{M} + 2(\Delta_2 - \delta^2_b/\lambda)/3)}
\right) \right]  \\
\geq& \;\; 1 - 8 \tr \left( (K + \lambda I_n)^{-1}(K + \delta^2_bI_n) \right) \left[ \exp \left( \frac{-m\Delta_1^2}{4n/\lambda(1 + 2\Delta_1/3)} \right) + \exp
\left(\frac{-m(\Delta_2 - \delta^2_b/\lambda)^2}{4n/\lambda(1 + 2(\Delta_2 - \delta^2_b/\lambda)/3)}
\right) \right].
  \end{align*}
\end{proof}

There is a bias-variance trade-off: as we decrease the number of bits $b$, under
a fixed memory budget we can use more features, and $(Z + C)(Z + C)^T$
concentrates more strongly (lower variance) around the expectation $K + D$ with
$0 \preceq D \preceq \delta^2_b I_n$.  However, this expectation is further away from the true kernel
matrix $K$ (larger bias).
Thus there should be an optimal number of bits $b^*$ that balances the bias and
the variance.

\begin{proof}[Proof of Corollary~\ref{cor:n_features_required}]
  Letting $\Delta_2 \to \infty$ in Theorem~\ref{thm2} gives
  \begin{equation*}
    \Prob \left[ (1 - \Delta_1) (K + \lambda I_n) \preceq \tilde{K} + \lambda I_n \right]
    \geq 1 - 8 \tr \left( (K + \lambda I_n)^{-1}(K + \delta^2_bI_n) \right) \exp \left( \frac{-m\Delta_1^2}{4n/\lambda(1 + 2\Delta_1/3)} \right).
  \end{equation*}
  Using the assumption that $\Delta_1 \leq 3/2$, we can simplify the bound:
  \begin{equation*}
    \Prob \left[ (1 - \Delta_1) (K + \lambda I_n) \preceq \tilde{K} + \lambda I_n \right]
    \geq 1 - 8 \tr \left( (K + \lambda I_n)^{-1}(K + \delta^2_bI_n) \right) \exp \left( \frac{-m\Delta_1^2}{8n/\lambda} \right).
  \end{equation*}
  Letting the RHS be $1-\rho$ and solving for $m$ yields
  \begin{equation*}
    m \geq \frac{8n/\lambda}{\Delta_1^2}\log\Big(\frac{a}{\rho}\Big).
  \end{equation*}

  Similarly, letting $\Delta_1 \to \infty$ in Theorem~\ref{thm2} gives
  \begin{equation*}
    \Prob \left[ \tilde{K} + \lambda I_n \preceq (1 - \Delta_2) (K + \lambda I_n) \right]
    \geq 1 - 8 \tr \left( (K + \lambda I_n)^{-1}(K + \delta^2_bI_n) \right) \exp \left( \frac{-m(\Delta_2 - \delta_b^2/\lambda)^2}{4n/\lambda(1 + 2(\Delta_2 - \delta_b^2/\lambda)/3)} \right).
  \end{equation*}
  Using the assumption that $\Delta_1 \leq 3/2$, we can simplify the bound:
  \begin{equation*}
    \Prob \left[ (1 - \Delta_1) (K + \lambda I_n) \preceq \tilde{K} + \lambda I_n \right]
    \geq 1 - 8 \tr \left( (K + \lambda I_n)^{-1}(K + \delta^2_bI_n) \right) \exp \left( \frac{-m(\Delta_2 - \delta_b^2/\lambda)^2}{8n/\lambda} \right).
  \end{equation*}
  Letting the RHS be $1-\rho$ and solving for $m$ yields
  \begin{equation*}
    m \geq \frac{8n/\lambda}{(\Delta_2 - \delta_b^2/\lambda)^2}\log\Big(\frac{a}{\rho}\Big).
  \end{equation*}
\end{proof}

If we let the number of bits $b$ go to $\infty$ and set $\Delta_1 = \Delta_2 = \Delta$, we get
the following corollary, similar to the result from \citet{avron17}:
\begin{corollary}
Suppose that $\tK = ZZ^T$, $\norm{K} \geq \lambda$. Then for any $\Delta \leq 1/2$,
\begin{equation*}
\Prob\left[(1 - \Delta)(K + \lambda I_n) \preceq \tilde{K} + \lambda I_n \preceq (1 + \Delta)(K + \lambda I_n) \right] \geq 1 - 16 \tr((K +
\lambda I_n)^{-1} K) \exp \left( -\frac{3m \Delta^2}{16n/\lambda} \right).
\end{equation*}
Thus if we use $m \geq \frac{16}{3\Delta^2} n/\lambda \log (16 \tr((K + \lambda I_n)^{-1} K) / \rho)$
features, then $(1 - \Delta)(K + \lambda I_n) \preceq \tilde{K} + \lambda I_n \preceq (1 + \Delta)(K + \lambda I_n)$ with probability at least $1 - \rho$.
\end{corollary}
The constants are slightly different from that of \citet{avron17} as we use the
real features $\sqrt{2} \cos(w^T x + a)$ instead of the complex features $\exp(i
w^T x)$.

From these results, we now know that the number of features required depends linearly on $n/\lambda$; more precisely, we know that if we use $m \geq c_0 \cdot n/\lambda$ features (for some constant $c_0 > 0$), $\tilde{K} + \lambda I_n$ will be a $(\Delta, \Delta)$-spectral approximation of $K + \lambda I_n$ with high probability.  \citet{avron17} further provide a lower bound, showing that if $m \leq c_1 \cdot n/\lambda$ (for some other constant $c_1 > 0$), $\tilde{K} + \lambda I_n$ will not be a $(\Delta, \Delta)$-spectral approximation of $K + \lambda I_n$ with high probability. This shows that the number of random Fourier features \textit{must} depend linearly on $n/\lambda$.

\subsection{Frobenius Kernel Approximation Error Bounds for LP-RFFs}
\label{sec:app_kernel_error}
We begin this section by bounding the variance of the quantization noise $C$ added to the random feature matrix $Z$.  We prove this as a simple consequence of the following Lemma.\footnote{This lemma is also a direct consequence of Popoviciu's inequality on variances \citep{popoviciu1935equations}. Nonetheless, we include a stand-alone proof of the lemma here for completeness.}

\begin{lemma}
	\label{lem:var_bound}
	For $z \in [a,c]$, let $X_z^{a,c}$ be the random variable which with probability $\frac{z-a}{c-a}$ equals $c-z$, and with probability $\frac{c-z}{c-a}$ equals $a-z$.  Then $\expect{}{X_z^{a,c}} = 0$, and $\var{}{X_z^{a,c}} = (c-z)(z-a) \leq \frac{(c-a)^2}{4}$.
\end{lemma}
\begin{proof}
	\begin{eqnarray*}
		\expect{}{X_z^{a,c}} &=&  (c-z)\cdot \frac{z-a}{c-a} + (a-z)\cdot \frac{c-z}{c-a}\\
		&=& 0 .\\
		\var{}{X_z^{a,c}} &=& (c-z)^2\cdot \frac{z-a}{c-a} + (a-z)^2\cdot \frac{c-z}{c-a}\\
		&=& \frac{(c-z)(z-a)((c-z + z-a))}{c-a} \\
		&=& (c-z)(z-a)\\
		\frac{d}{dz}[\var{}{X_z^{a,c}}] &=& \frac{d}{dz}[-z^2 + (c+a)z -ac] \\
		&=& -2z + c+a.
	\end{eqnarray*}
	Now, setting the derivative to 0 gives $z^* = \frac{c+a}{2}$.  Thus, $\argmax_{z\in[a,c]} (c-z)(z-a) = \frac{c+a}{2}$, and $\max_{z\in[a,c]} (c-z)(z-a) = (c-\frac{c+a}{2})(\frac{c+a}{2}-a) = \frac{(c-a)^2}{4}$.
\end{proof}
\begin{proposition}
	\label{prop:var_bound}
	$\E[C_{ij}^2 \mid Z_{ij}] \leq \delta_b^2/m$, for $\delta_b^2 \defeq \frac{2}{(2^b - 1)^2}$.
\end{proposition}
\begin{proof}
	Given a feature $Z_{ij} \in [-\sqrt{2/m},\sqrt{2/m}]$, we quantize it to $b$ bits by dividing this interval into $2^b - 1$ equally-sized sub-intervals (each of size $\frac{2\sqrt{2/m}}{2^b-1}$), and randomly rounding to the top or bottom of the sub-interval containing $Z_{ij}$ (in an unbiased manner). Let $a,c$ denote the boundaries of the sub-interval containing $Z_{ij}$, where $c = a + \frac{2\sqrt{2/m}}{2^b-1}$.  We can now see that $C_{ij} = X_{Z_{ij}}^{a,c}$ is the unique random variable such that $Z_{ij}+C_{ij}\in\{a,c\}$ and $\expect{}{C_{ij} \mid Z_{ij}} = 0$.  Because $c-a = \frac{2\sqrt{2/m}}{2^b-1}$, it follows from Lemma~\ref{lem:var_bound} that $\E[C_{ij}^2 \mid Z_{ij}] = \var{}{X_{Z_{ij}}^{a,c}\mid Z_{ij}} \leq \frac{(c-a)^2}{4} = \frac{8/m}{4(2^b-1)^2} = \delta_b^2/m$.
\end{proof}

We now prove an upper bound on the expected kernel approximation error for LP-RFFs, which applies for any quantization function with bounded variance.  This error corresponds exactly to the variance of the random variable which is the product of two quantized random features.

\begin{theorem}
	\label{thm:var1}
	For $x,y\in\cX$, assume we have random variables $Z_x,Z_y$ satisfying $\expect{}{Z_xZ_y} = k(x,y)$, $\var{}{Z_xZ_y} \leq \sigma^2$, and that $k(x,x) = k(y,y) = 1$.\footnote{For example, one specific instance of the random variables $Z_x,Z_y$ is given by random Fourier features, where $z_x = \sq\cos(w^Tx+b),z_y = \sq\cos(w^Ty+b)$, $z_x,z_y\in[-\sq,\sq]$, for random $w,b$.  We need not assume that $k(x,x)=k(y,y) = 1$, but this simplifies the final expression, as can be seen in the last step of the proof.}  For any unbiased random quantization function $Q$ with bounded variance $\var{}{Q(z)} \leq \tsigma^2$ for any (fixed) $z$, it follows that $\expect{}{Q(Z_x)Q(Z_y)} = k(x,y)$, and that $\var{}{Q(Z_x)Q(Z_y)} \leq 2\tsigma^2 + \tsigma^4 + \sigma^2$.
\end{theorem}
\begin{proof}
	Let $Q(Z_x) = Z_x + \eps_x$, and $Q(Z_y) = Z_y + \eps_y$, where $\expect{}{\eps_x} =\expect{}{\eps_y} = 0$, $\expect{}{\eps_x^2} \leq \tsigma^2$, $\expect{}{\eps_y^2} \leq \tsigma^2$, and $\eps_x$,$\eps_y$ are independent random variables.
	
	\begin{eqnarray*}
		\expect{}{Q(Z_x)Q(Z_y)} &=& \expect{}{(Z_x + \eps_x)(Z_y + \eps_y)} \\
		&=& \expect{}{Z_xZ_y + Z_y \eps_x + Z_x \eps_y + \eps_x \eps_y} \\
		&=& \expect{}{Z_xZ_y} \\
		&=& k(x,y).\\
		\var{}{Q(Z_x)Q(Z_y)} &=&  \expect{}{\Big(Q(Z_x)Q(Z_y) - k(x,y)\Big)^2} \\
		&=& \expect{}{\Big((Z_x + \eps_x)(Z_y + \eps_y) - k(x,y)\Big)^2} \\
		&=& \expect{}{\Big(Z_y \eps_x + Z_x\eps_y + \eps_x \eps_y +  Z_xZ_y - k(x,y)\Big)^2} \\
		&=& \expect{}{\Big(Z_y \eps_x + Z_x\eps_y + \eps_x \eps_y\Big)^2} +  \expect{}{\Big(Z_xZ_y - k(x,y)\Big)^2} \\
		&\leq& \expect{}{Z_y^2 \eps_x^2 + Z_x^2\eps_y^2 + \eps_x^2 \eps_y^2} +  \sigma^2 \\
		&\leq& k(y,y) \tsigma^2 + k(x,x) \tsigma^2 + \tsigma^4 +  \sigma^2 \\
		&=& 2\tsigma^2 + \tsigma^4 +  \sigma^2. 
	\end{eqnarray*}
Note that if $x=y$, for this proof to hold, we would need to randomly quantize $Z_x$ twice, giving two independent quantization noise samples $\eps_{x}^{(1)}$ and $\eps_{x}^{(2)}$.
\end{proof}

This theorem suggests that if $2\tsigma^2 + \tsigma^4 \ll \sigma^2$, quantizing the random features will have negligable effect on the variance.  In practice, for the random quantization method we use with random Fourier features (described in Section \ref{subsec:method_details}), we have $\tsigma^2 = \frac{2}{(2^b - 1)^2}$.  Thus, for a large enough precision $b$, the quantization noise will be tiny relative to the variance inherent to the random features.

We note that the above theorem applies to one-dimensional random features, but can be trivially extended to $m$ dimensional random features.  We show this in the following corollary.
\begin{corollary}
	\label{cor:varn}
	For $x,y\in\cX$, assume we have random variables $Z_x,Z_y$ satisfying $\expect{}{Z_xZ_y} = k(x,y)$, and $\var{}{Z_x},\var{}{Z_y} \leq \sigma^2$, and that $k(x,x) = k(y,y) = 1$.
	Let $Q$ be any unbiased quantization function with bounded variance ($\expect{}{Q(z)} = z$, $\var{}{Q(z)} \leq \tsigma^2$ for any $z$).  Let $S=Z_x Z_y$, $T = Q(Z_x)Q(Z_y)$, and $(S_1,\ldots,S_n)$, $(T_1,\ldots,T_n)$ be a random sequence of i.i.d. draws from $S$ and $T$ respectively.  Define $\bar{S}_n = \frac{1}{n}\sum_{i=1}^n S_i$, and $\bar{T}_n =  \frac{1}{n}\sum_{i=1}^n T_i$, to be the empirical mean of these draws.  It follows that $\expect{}{\bS_n} = \expect{}{\bT_n} = k(x,y)$, and that
	$\var{}{\bS_n} \leq \frac{\sigma^2}{n}$, and $\var{}{\bT_n} \leq \frac{2\tsigma^2 + \tsigma^4 +  \sigma^2}{n}$.
\end{corollary}
\begin{proof}
	\begin{eqnarray*}
		\var{}{\bS_n} &=& \var{}{\frac{1}{n}\sum_{i=1}^n S_i} \\
		&=& \sum_{i=1}^n \var{}{\frac{1}{n}S_i} \\
		&\leq& n \cdot \frac{\sigma^2}{n^2} \\
		&=& \frac{\sigma^2}{n}
	\end{eqnarray*}
	The result for $\var{}{\bT_n}$ follows in the same way, using the result from Theorem \ref{thm:var1}.
\end{proof}

\section{\uppercase{Experiment details and extended results}}
\label{sec:exp_details}
\subsection{Datasets and Details Applying to All Experiments}
\label{app:datasets_and_global_exp_details}
In this work, we present results using FP-\NystromNS, FP-RFFs, circulant FP-RFFs, and LP-RFFs on the TIMIT, YearPred, CovType, and Census datasets.
These datasets span regression, binary classification, and multi-class classification tasks.
We present details about these datasets in Table~\ref{tab:dataset_details}.
In these experiments we use the Gaussian kernel with the bandwidth $\sigma$ specified in Table~\ref{tab:kernel_hyper}; 
we use the same bandwidths as \citet{may2017}.
To evaluate the performance of these kernel models, we measure the classification error for classification tasks, and the mean squared error (MSE) for regression tasks ($\frac{1}{n}\sum_{i=1}^n (f_{\tK}(x_i) - y_i)^2$), on the heldout set. 
We compute the total memory utilization as the sum of all the components in Table~\ref{table:mem-usage}.

\begin{table}
	\caption{Dataset details.  For classification tasks, we write the number
		of classes in parentheses in the ``task'' column.}
	\begin{center}
		\begin{tabular}{llllll} 
			\toprule
			\textbf{Dataset}  & \textbf{Task} & \textbf{Train} & \textbf{Heldout} & \textbf{Test} & \textbf{\# Features} \\ 
			\midrule
			Census   & Reg.   & 16k   & 2k      & 2k   & 119 \\ 
			YearPred & Reg.   & 417k  & 46k     & 52k  & 90  \\ 
			CovType  & Class. (2) & 418k  & 46k     & 116k & 54  \\ 
			TIMIT    & Class. (147) & 2.3M  & 245k    & 116k & 440 \\
			\bottomrule
		\end{tabular}
	\end{center}
	\label{tab:dataset_details}
\end{table}

\begin{table}
	\caption{The Gaussian kernel bandwidths used, and the search grid for initial learning rate on the Census, YearPred, Covtype and TIMIT datasets. Optimal learning rate in bold.}
	\begin{center}
		\begin{tabular}{lll}
			\toprule
			Dataset & $1/2\sigma^2$ & Initial learning rate grid \\
			\midrule
			Census & 0.0006 & {0.01, 0.05, 0.1, \textbf{0.5}, 1.0} \\
			YearPred & 0.01 & {0.05, 0.1, \textbf{0.5}, 1.0, 5.0} \\
			Covtype & 0.6 & {1.0, 5.0, 10.0, \textbf{50.0}, 100.0} \\
			TIMIT & 0.0015 & {5.0, 10.0, 50.0, \textbf{100.0}, 500.0} \\
			\bottomrule
		\end{tabular}
	\end{center}
	\label{tab:kernel_hyper}
\end{table}

For all datasets besides TIMIT, we pre-processed the features and labels as
follows: We normalized all continuous features to have zero mean and unit
variance. We did not normalize the binary features in any way.  For regression datasets,
we normalized the labels to have zero mean across the training set.

TIMIT \citep{timit} is a benchmark dataset in the speech recognition community which 
contains recordings of 630 speakers, of various English dialects, each reciting 
ten sentences, for a total of 5.4 hours
of speech. The training set (from which the heldout set is then taken) consists
of data from 462 speakers each reciting 8 sentences (SI and SX sentences).
We use 40 dimensional feature space maximum likelihood linear regression (fMLLR) features
\citep{gales1998}, and concatenate the 5 neighboring frames in either direction,
for a total of 11 frames and 440 features.  This dataset has 147 labels, 
corresponding to the beginning, middle, and end of
49 phonemes.  For reference, we use the exact same features,
labels, and divisions of the dataset, as \citep{huang14kernel,chen16,may2017}.

We acquired the CovType (binary) and YearPred datasets from the LIBSVM 
webpage,\footnote{https://www.csie.ntu.edu.tw/~cjlin/libsvmtools/datasets/}
and the Census dataset from Ali Rahimi's 
webpage.\footnote{https://keysduplicated.com/~ali/random-features/data/}
For these datasets, we randomly set aside 10\% of the training data as a
heldout set for tuning the learning rate and kernel bandwidth.

The specific files we used were as follows:
\begin{itemize}
	\item CovType: We randomly chose 20\% of ``covtype.libsvm.binary'' as test, and used
	the rest for training/heldout.
	\item YearPred: We used ``YearPredictionMSD'' as training/heldout set, and ``YearPredictionMSD.t''
	as test.
	\item Census: We used the included matlab dataset file ``census.mat'' from Ali Rahimi's webpage.
	This Matlab file had already split the data into train and test.  We used a random 10\% of the
	training data as heldout.
\end{itemize}

\subsection{\Nystrom vs.\ RFFs (Section~\ref{subsec:nys_vs_rff_exp})}
\label{app:nys_vs_rff_details}
We compare the generalization performance of full-precision RFFs and the \Nystrom method across four datasets, for various memory budgets.
We sweep the following hyperparameters: For \NystromNS, we use $m \in \{1250, 2500, 5000, 10000, 20000\}$. 
For RFFs, we use $m\in \{1250, 2500, 5000, 10000,20000, 50000, 100000,$
$200000, 400000\}$.
We choose these limits differently because 20k \Nystrom features have roughly the same memory footprint as 400k FP-RFFs.
For all experiments, we use a mini-batch size of 250.
We use a single initial learning rate per dataset across all experiments, which we tune via grid search using 20k \Nystrom features. We choose to use \Nystrom features to tune the initial learning rate in order to avoid biasing the results in favor of RFFs.
We use an automatic early-stopping protocol, as in \citep{morgan1990generalization,sainath2013b,sainath2013low}, to regularize our models \citep{zhang2005boosting,wei2017early} and avoid expensive hyperparameter tuning.
It works as follows: at the end of each epoch, we decay the learning rate in half if the heldout loss is less than $1\%$ better relative to the previous best model, using MSE for regression and cross entropy for classification.
Furthermore, if the model performs \textit{worse} than the previous best, we revert the model.
The training terminates after the learning rate has been decayed 10 times.

We plot our results comparing the \Nystrom method to RFFs in Figure~\ref{fig:generalization_col_app}.
We compare the performance of these methods in terms of their number of features, their training memory  footprint, and the squared Frobenius norm and spectral norm of their kernel approximation error.

\begin{figure}
\centering
\begin{tabular}{@{\hskip -0.0in}c@{\hskip -0.0in}c@{\hskip -0.0in}c@{\hskip -0.0in}c@{\hskip -0.0in}}
	\includegraphics[width=0.245\linewidth]{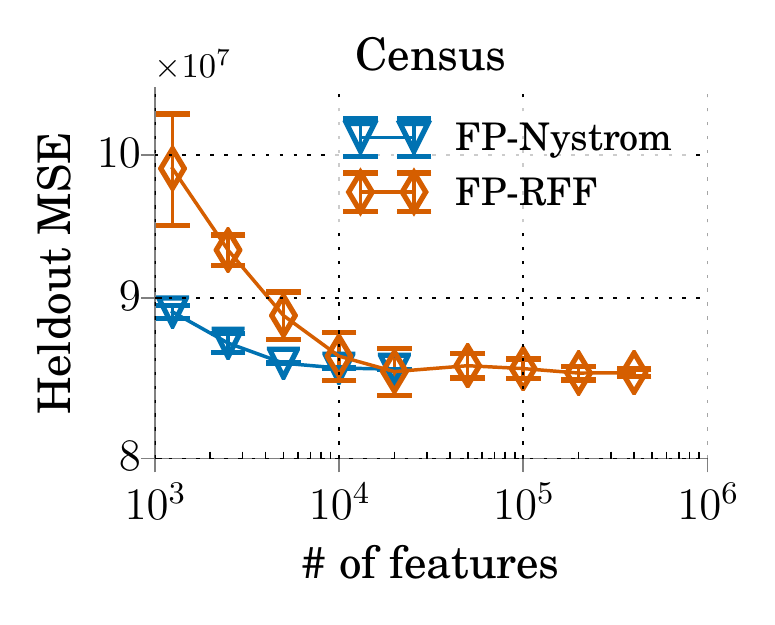} &
	\includegraphics[width=0.245\linewidth]{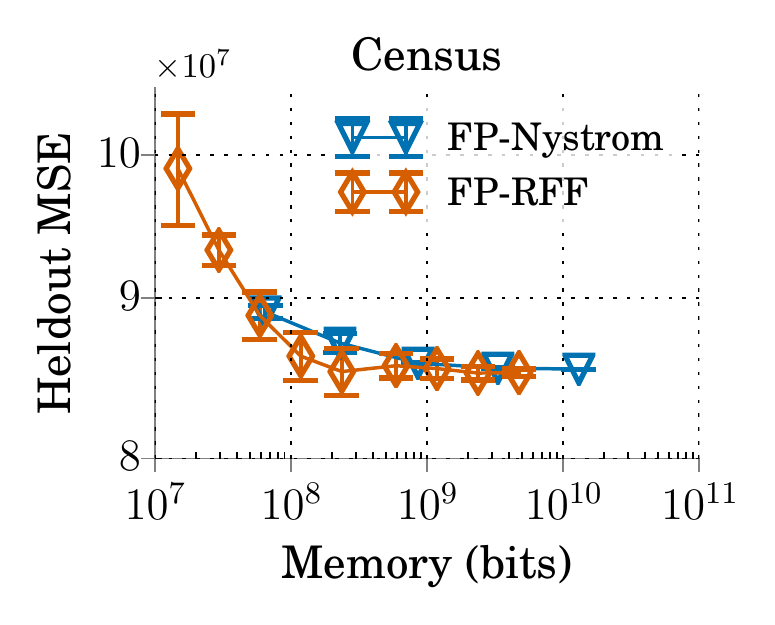} &
	\includegraphics[width=0.245\linewidth]{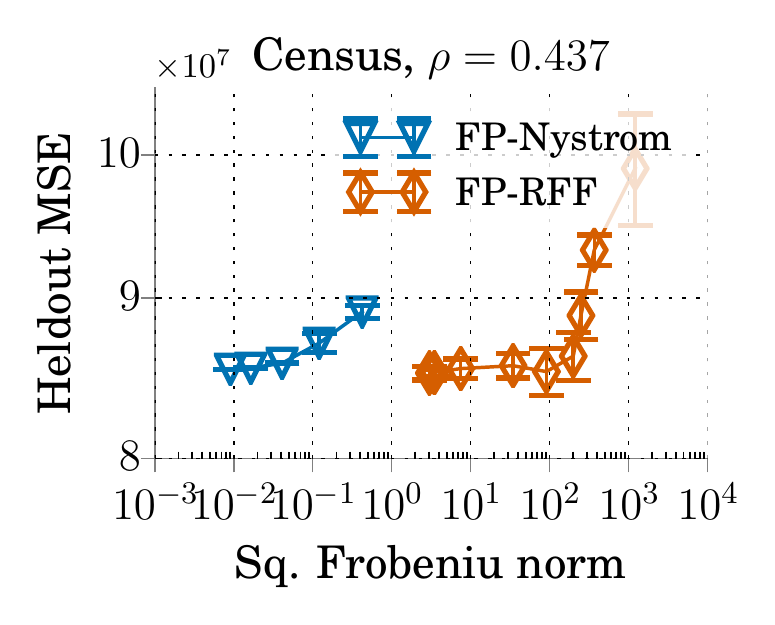} &
	\includegraphics[width=0.245\linewidth]{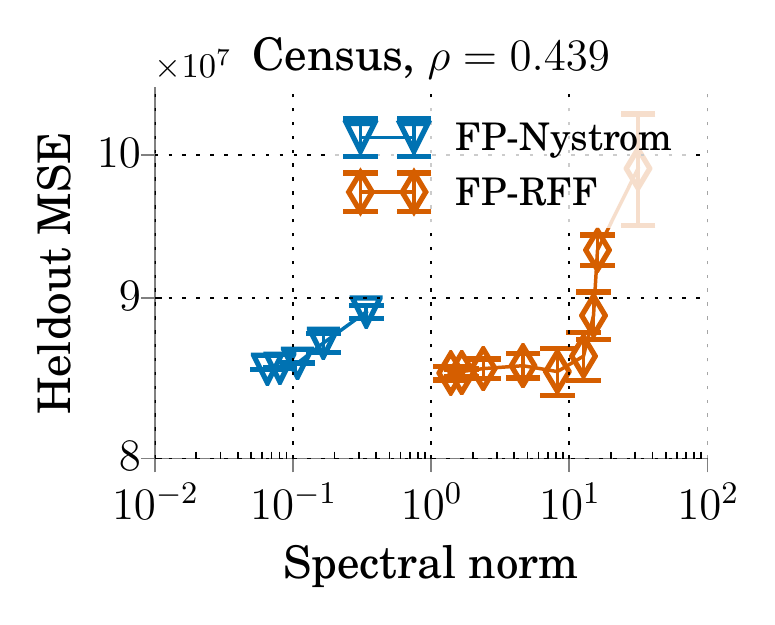} \\
	\includegraphics[width=0.245\linewidth]{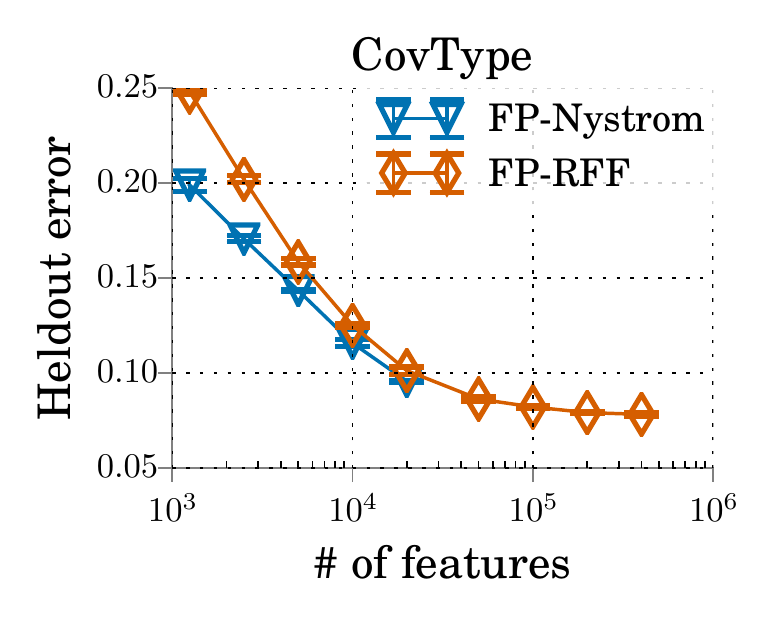} &
	\includegraphics[width=0.245\linewidth]{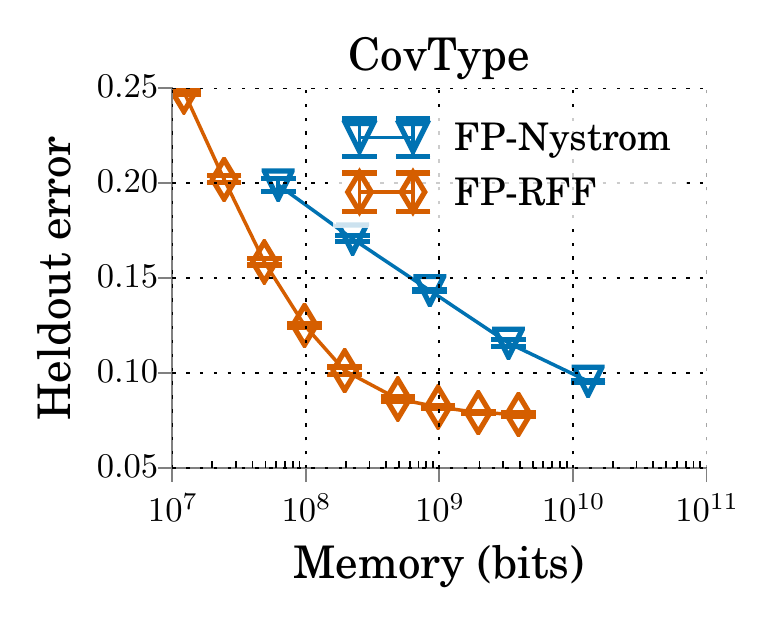} &
	\includegraphics[width=0.245\linewidth]{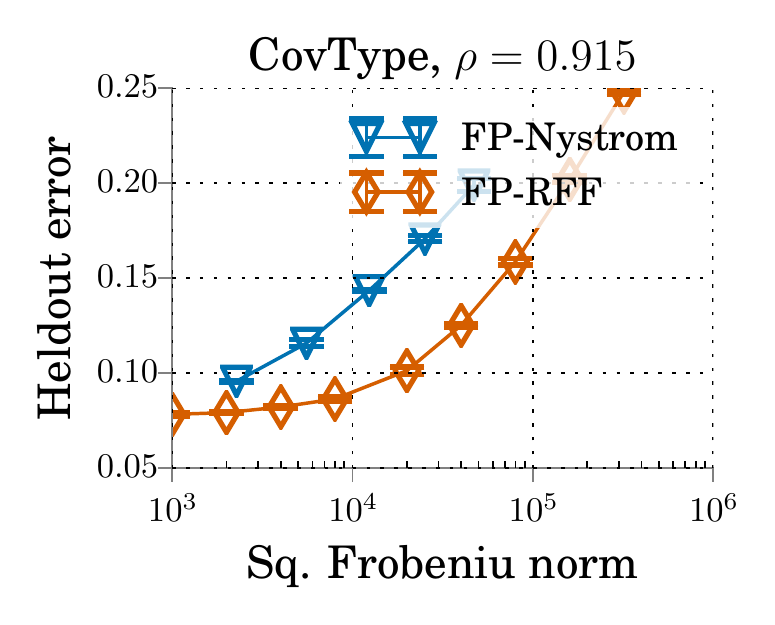} &
	\includegraphics[width=0.245\linewidth]{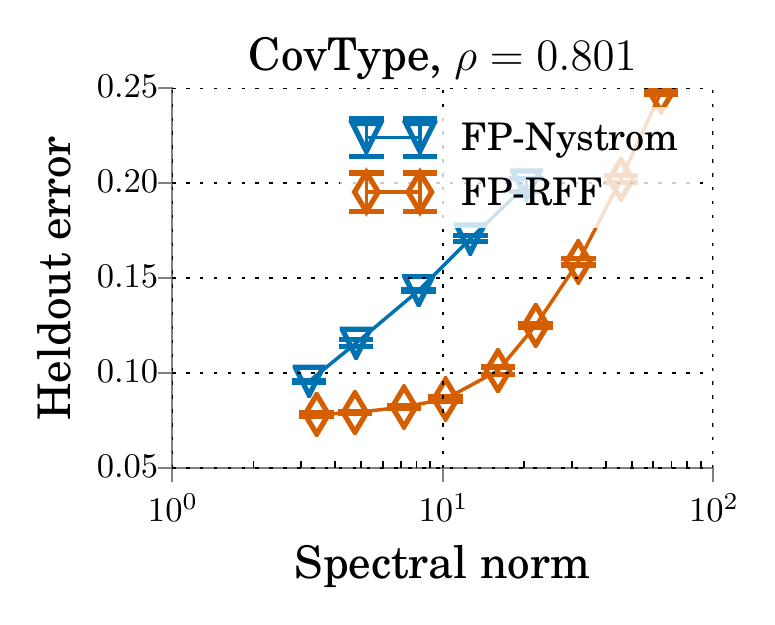} \\
	\includegraphics[width=0.245\linewidth]{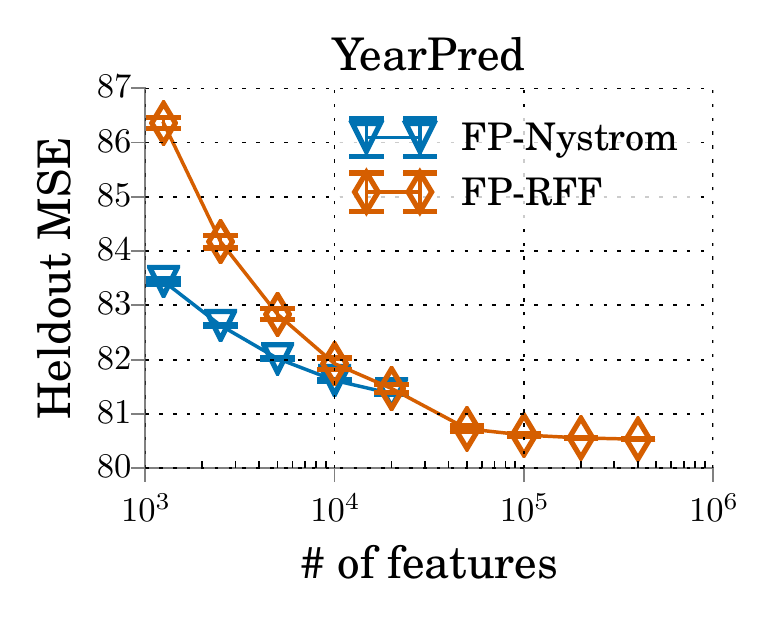} &
	\includegraphics[width=0.245\linewidth]{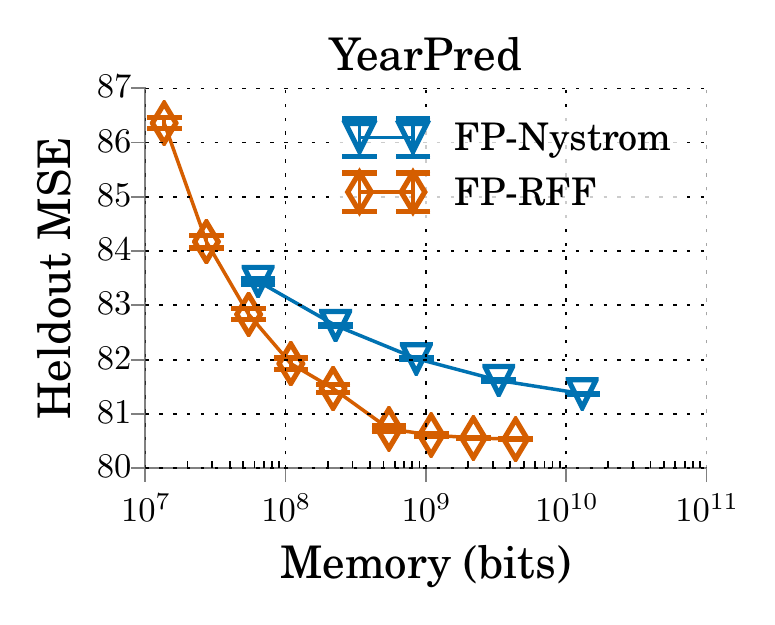} &
	\includegraphics[width=0.245\linewidth]{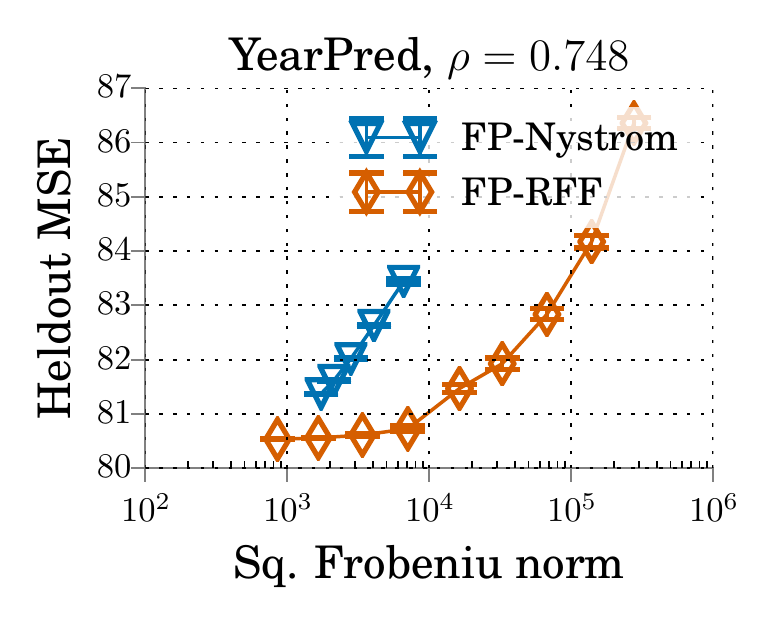} &
	\includegraphics[width=0.245\linewidth]{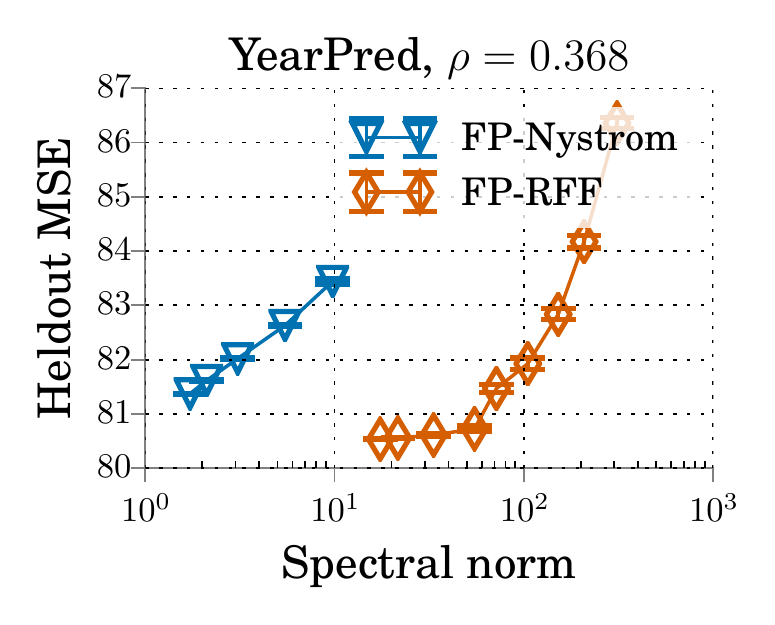} \\
	\includegraphics[width=0.245\linewidth]{figures/timit_error_vs_n_feat_nystrom_vs_rff.pdf} &
	\includegraphics[width=0.245\linewidth]{figures/timit_error_vs_n_memory_nystrom_vs_rff.pdf} &
	\includegraphics[width=0.245\linewidth]{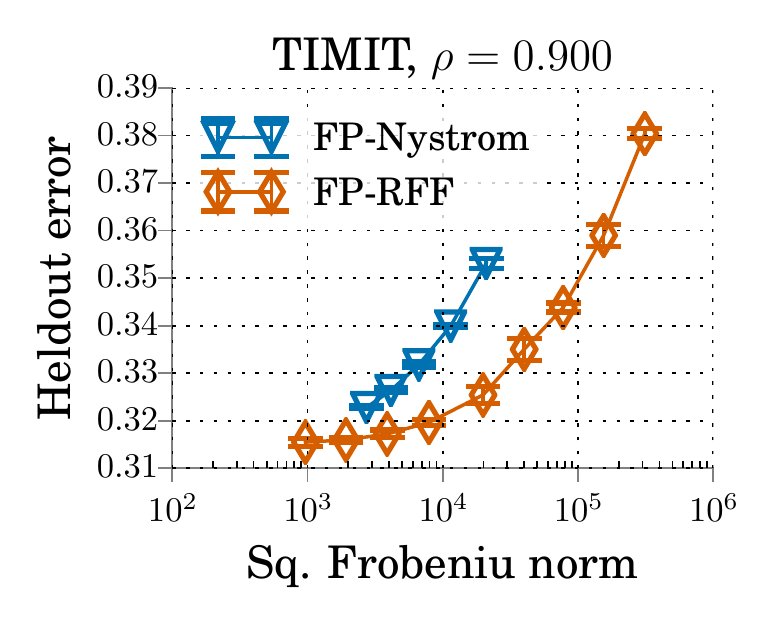} &
	\includegraphics[width=0.245\linewidth]{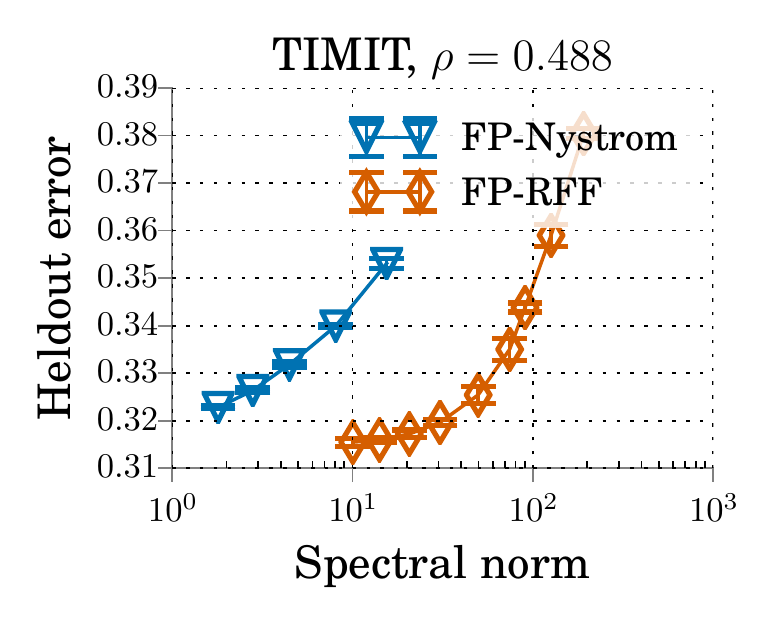} \\
	\;\;\;\;\;(a) & \;\;\;\;\;\;(b) & \;\;\;\;\;\;(c) & \;\;\;\;\;\;(d)
\end{tabular}
\caption{
		Generalization performance of FP-RFFs and \Nystrom with respect to \# features (a) and training memory footprint (b) on Census, CovType, Yearpred and TIMIT.
		\Nystrom performs better for a fixed number of features, while FP-RFFs perform better under a memory budget.
		We also see that the relative performance between these methods cannot be fully explained by the Frobenius norms (c) or spectral norms (d) of their respective kernel approximation error matrices;
		in particular, we see many example of \Nystrom models that have lower Frobenius or spectral error, but worse heldout performance, than various RFF models.
		To quantify the degree of alignment between these error metrics and generalization performance (right plots), we show in the figure titles the Spearman rank correlation coefficients $\rho$ between the corresponding $x$ and $y$ metrics.
		We see in Figure~\ref{fig:perc_delta} that $1/(1-\Delta_1)$ attains notably higher values of $\rho$ than the Frobenius and spectral norms of the kernel approximation error.
		Note that although we plot the average performance across three random seeds for each experimental setting (error bars indicate standard deviation), when we compute the $\rho$ values we treat each experimental result independently (without averaging).
}
	\label{fig:generalization_col_app}
\end{figure}

\subsection{\Nystrom vs.\ RFFs Revisited (Section~\ref{subsec:nys_vs_rff_revisited})}
\label{app:nys_vs_rff_revisited_details}

\subsubsection{Small-Scale Experiments}
\label{subsec:app_exp_smallscale}
In order to better understand what properties of kernel approximation features lead to strong generalization performance, 
we perform a more fine-grained analysis on two smaller datasets from the UCI machine learning repository---we consider the Census regression task, 
and a subset of the CovType task with 20k randomly sampled training points and 20k randomly sampled heldout points.
The reason we use these smaller datasets is because computing the spectral norm, as well as ($\Delta_1$,$\Delta_2$) are expensive operations, which requires instantiating the kernel matrices fully, and performing singular value decompositions.
For the Census experiments, we use the closed-form solution for the kernel ridge regression estimator, and choose the $\lambda$ which gives the best performance on the heldout set.
For CovType, because there is no closed-form solution for logistic regression, we used the following training protocol to (approximately) find the model which minimizes the regularized training loss (just like the closed-form ridge regression solution does).
For each value of $\lambda$, we train the model to (near) convergence using 300 epochs of SGD (mini-batch size 250) at a constant learning rate, and pick the learning rate which gives the lowest regularized training loss for that $\lambda$.
We then evaluate this converged model on the heldout set to see which $\lambda$ gives the best performance.
We pick the best learning rate, as well as regularization parameter, by using 20k \Nystrom features as a proxy for the exact kernel (note that because there are only 20k training points, this \Nystrom approximation is exact). 
We choose the learning rate $5.0$ from the set $\{0.01, 0.05, 0.1, 0.5, 1.0, 5.0, 10.0, 50.0\}$, and the regularization parameter 
$\lambda=\num{5e-6}$ from $\{\num{5e-8}, \num{1e-7}, \num{5e-7}, \num{1e-6}, \num{5e-6}, \num{1e-5}, \num{5e-5}, \num{1e-4}\}$. For the Census dataset, we pick $\lambda=\num{5e-4}$ from $\{\num{1e-5}, \num{5e-5}, \num{1e-4}, \num{5e-4}, \num{1e-3}, \num{5e-3}, \num{1e-2}, \num{5e-2}, \num{1e-1}\}$.
We sweep the number of features shown in Table~\ref{table:n_feat_grid_small_exp}.
For both datasets, we report the average squared Frobenius norm, spectral norm, $\Delta$, $(\Delta_1, \Delta_2)$ and the average generalization performance, along with standard deviations, using five different random seeds.
The results are shown in Figure~\ref{fig:metrics_vs_perf_app}.
As can be seen in Figure~\ref{fig:metrics_vs_perf_app}, $1/(1-\Delta_1)$ aligns much better with generalization performance than the other metrics.

\begin{figure}
	\centering
	\begin{small}
		\begin{tabular}{@{\hskip -0.0in}c@{\hskip -0.0in}c@{\hskip -0.0in}c@{\hskip -0.0in}c@{\hskip -0.0in}}
			\includegraphics[width=0.245\linewidth]{figures/regression_l2_vs_f_norm_FP_only_nystrom_vs_rff.pdf} &
			\includegraphics[width=0.245\linewidth]{figures/regression_l2_vs_s_norm_FP_only_nystrom_vs_rff.pdf} &
			\includegraphics[width=0.245\linewidth]{figures/regression_l2_vs_delta_Avron.pdf} &	
			\includegraphics[width=0.245\linewidth]{figures/regression_l2_vs_delta_FP_delta_left_transform_nystrom_vs_rff.pdf} 
			\\
			\includegraphics[width=0.245\linewidth]{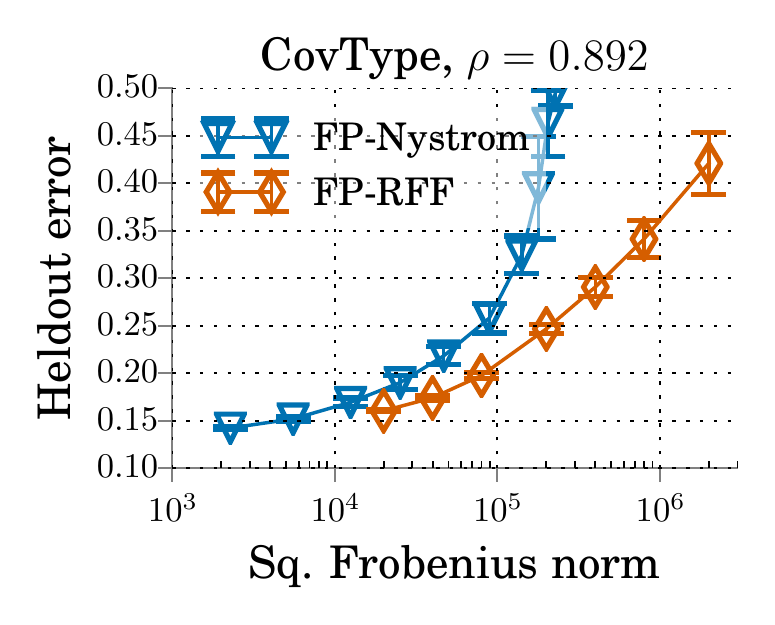} &
			\includegraphics[width=0.245\linewidth]{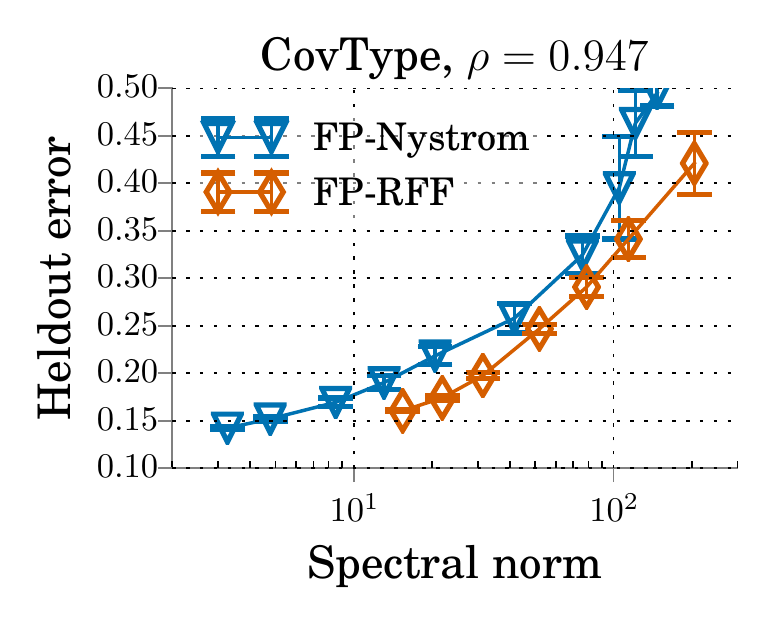} &		
			\includegraphics[width=0.245\linewidth]{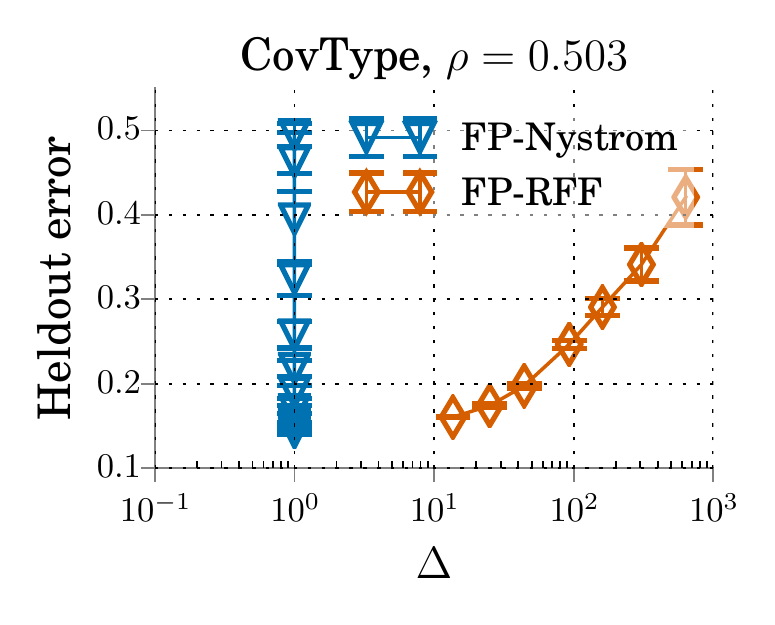} &
			\includegraphics[width=0.245\linewidth]{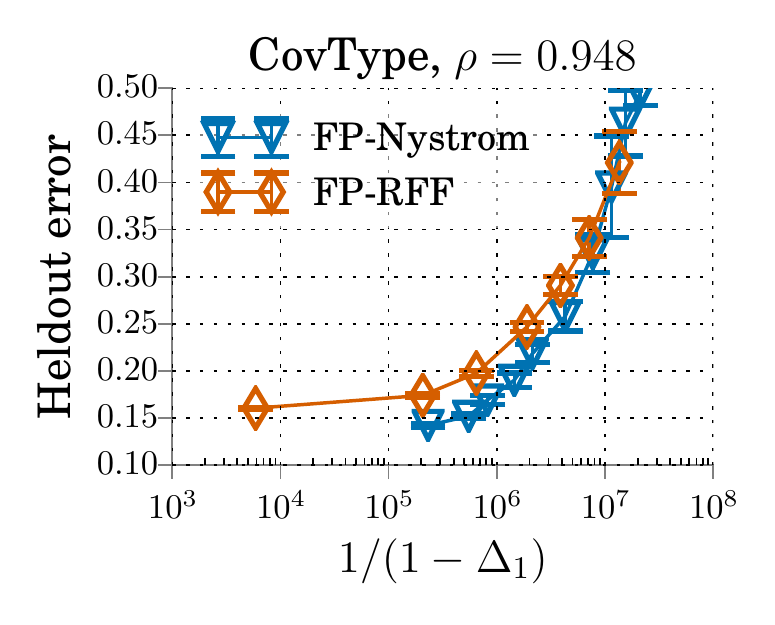}
		\end{tabular}
	\end{small}
	\caption{The correlation between generalization performance and squared Frobenius norm, spectral norm, $\Delta$, and $1/(1-\Delta_1)$ for FP-RFFs and FP-\Nystrom on the Census and subsampled CovType datasets.
	To quantify the alignment between these metrics and downstream performance, we include the Spearman rank correlation coefficients $\rho$ computed between these metrics and the downstream performance of our trained models in the figure titles.
	Note that although we plot the average performance across five random seeds for each experimental setting (error bars indicate standard deviation), when we compute the $\rho$ values we treat each experimental result independently (without averaging).
	}
	\label{fig:metrics_vs_perf_app}
\end{figure}

\begin{figure}
	\centering
	\begin{small}
		\begin{tabular}{@{\hskip -0.0in}c@{\hskip -0.0in}c@{\hskip -0.0in}c@{\hskip -0.0in}c@{\hskip -0.0in}}
			\includegraphics[width=0.245\linewidth]{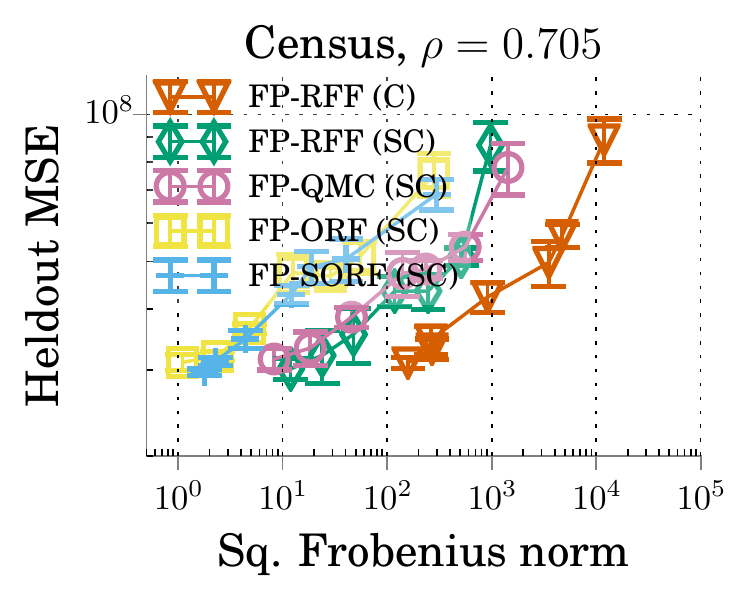} &
			\includegraphics[width=0.245\linewidth]{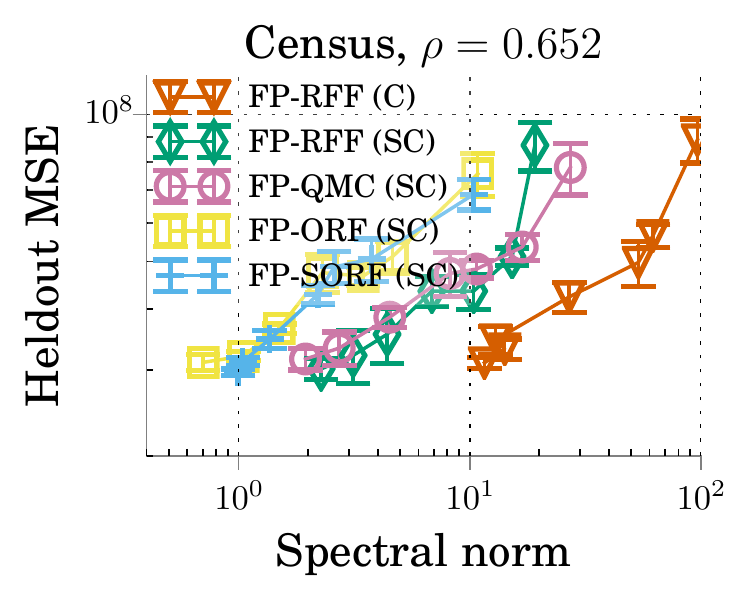} &
			\includegraphics[width=0.245\linewidth]{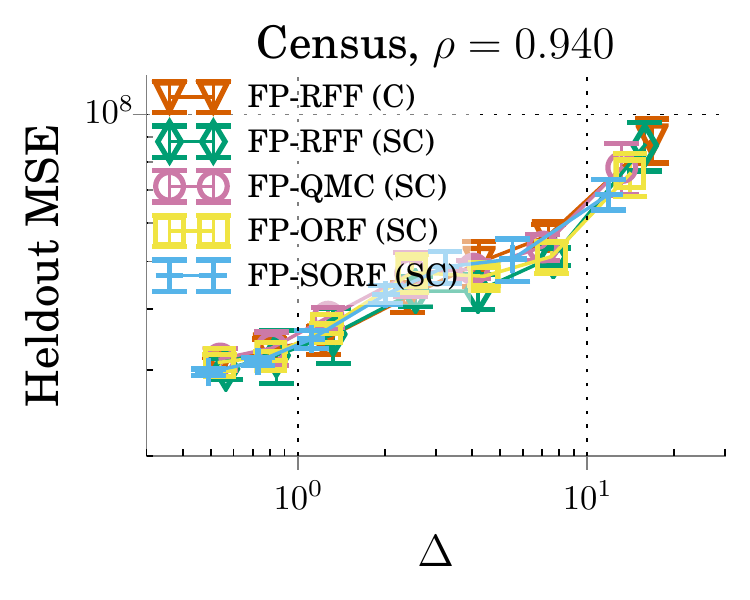} &	
			\includegraphics[width=0.245\linewidth]{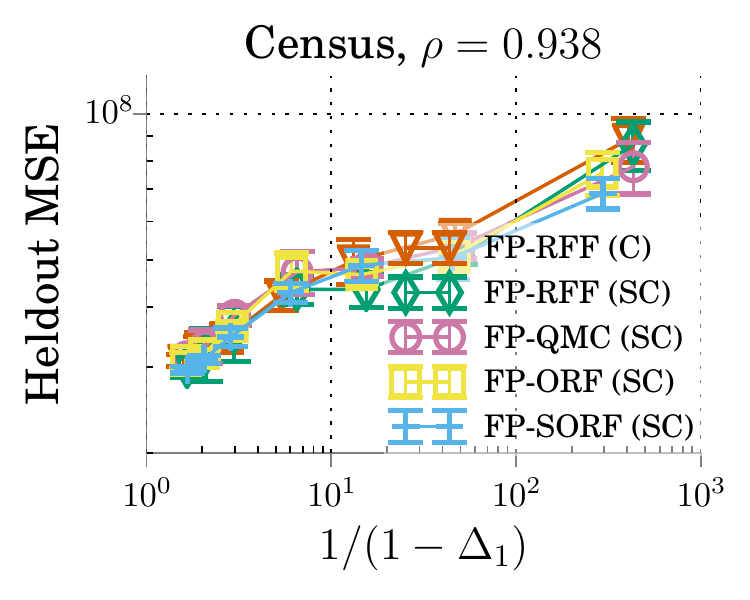} 
		\end{tabular}
	\end{small}
	\caption{The correlation between generalization performance and squared Frobenius norm, spectral norm, $\Delta$, and $1/(1-\Delta_1)$ for various types of full-precision RFFs.
	The Spearman rank correlation coefficients $\rho$ between the $x$ and $y$ metrics of each figure are included in the figure titles.
	For these full-precision RFF experiments, we see that the original $\Delta$ \citep{avron17} as well as $1/(1-\Delta_1)$ both align very well with downstream performance ($\rho = 0.940$ and $\rho=0.938$, respectively), while the Frobenius and spectral norms do not ($\rho = 0.705$ and $\rho=0.652$, respectively).
	Note that although we plot the average performance across five random seeds for each experimental setting (error bars indicate standard deviation), when we compute the $\rho$ values we treat each experimental result independently (without averaging).
	}
	\label{fig:metrics_vs_perf_app_many_rff}
\end{figure}

\begin{figure}
\centering	
\begin{tabular}{c c}
	\includegraphics[height=0.275\linewidth]{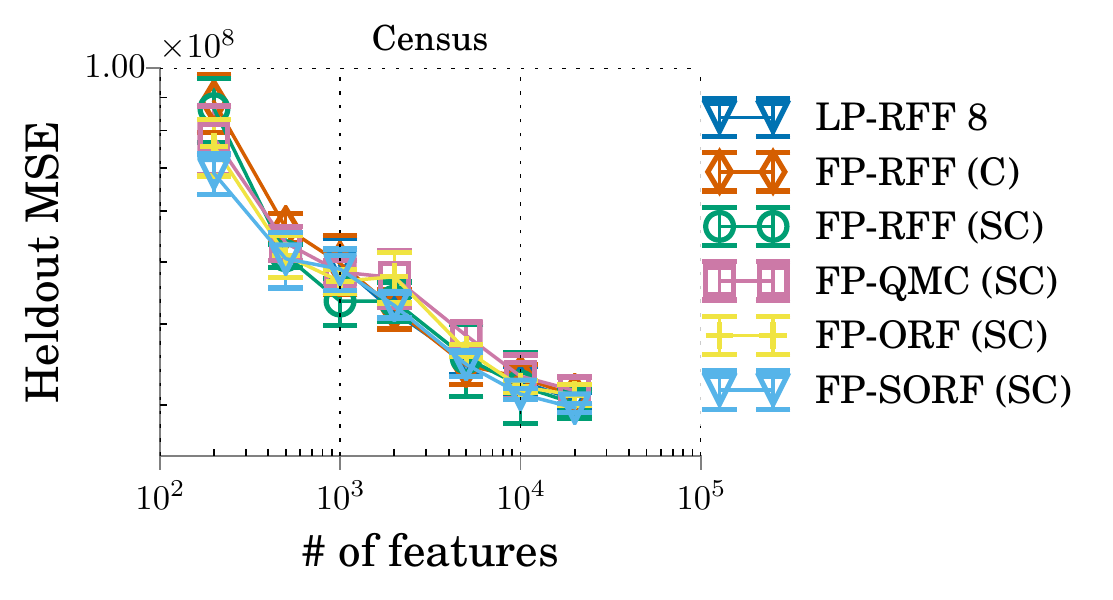} &
	\includegraphics[height=0.275\linewidth]{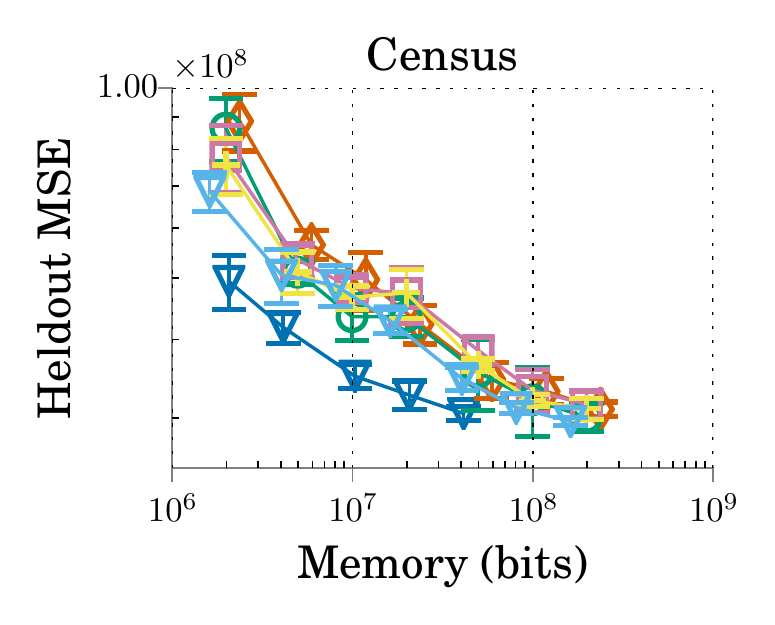}
\end{tabular}
\caption{The generalization performance as a function of the number of features (left) and the training memory footprint (right) for various types of RFFs on the Census dataset.
We observe that LP-RFFs can achieve lower heldout mean-squared error (MSE) than other types of RFFs, included the memory-efficient structured orthogonal random features (``FP-SORF (SC)''), under the same memory budget.
We plot performance averaged over five random seeds, with error bars indicating standard deviations.
}
\label{fig:many_rff_mem_feat}
\end{figure}

\begin{table}
		\caption{Number of features used for the different kernel approximation methods in the experiments on generalization performance vs. $(\Delta_1,\Delta_2)$ in Sections \ref{subsec:nys_vs_rff_revisited} and \ref{subsec:gen_perf_and_delta}.}
	\centering
	\begin{tabular}{c | c}
		\toprule
		Methods & Number of features \\
		\midrule
		FP-\Nystrom & $25, 50, 100, 200, 500, 1250, 2500, 5000, 10000, 20000$ \\
		FP-RFF & $200, 500, 1000, 2000, 5000, 10000, 20000$ \\
		Cir. FP-RFF & $200, 500, 1000, 2000, 5000, 10000, 20000$ \\
		LP-RFF $16$ & $500, 1000, 2000, 5000, 10000, 20000, 50000$ \\
		LP-RFF $8, 4, 2, 1$ & $1000, 2000, 5000, 10000, 20000, 50000$ \\
		\bottomrule
	\end{tabular}
	\label{table:n_feat_grid_small_exp}
\end{table}

It is important to note that although $1/(1-\Delta_1)$ aligns quite well with generalization performance (Spearman rank correlation coefficients $\rho$ equal to 0.958 and 0.948 on Census and CovType, respectively), it does not align perfectly.
In particular, we see that for a fixed value of $\Delta_1$, \Nystrom generally attains better heldout performance than RFFs.
We believe this is largely explained by the fact that \Nystrom always has $\Delta_2=0$, while RFFs can have relatively large values of $\Delta_2$ when the number of features is small (see Figure~\ref{fig:theory_supporting}).
As we show in the generalization bound in Proposition~\ref{prop:alphabeta}, in Figure~\ref{fig:gen_delta_correlation}, as well as in the empirical and theoretical analysis in Appendix~\ref{sec:app_generalization_bound}, we expect generalization performance to deteriorate as $\Delta_2$ increases.

\subsubsection{Experiments with Different Types of RFFs}
In Figure~\ref{fig:metrics_vs_perf_app_many_rff} we repeat the above experiment on the Census dataset using a number of variations of RFFs.
Specifically, we run experiments using the $[\sin(w^Tx),\cos(w^Tx)]$ parameterization of RFFs, which \citet{sutherland15} show has lower variance than the $\cos(w^Tx + a)$ parameterization (we denote the $[\sin,\cos]$ method by ``FP-RFF (SC)'', and the $\cos$ method by ``FP-RFF (C)'').
We additionally use Quasi-Monte Carlo features \citep{qmc}, as well as orthogonal random features and its structural variant \citep{yu16};
we denote these by ``FP-QMC (SC)'', ``FP-ORF (SC)'', ``FP-SORF (SC)'' respectively, because we implement them using the $[\sin,\cos]$ parameterization.
We can see in Figure~\ref{fig:metrics_vs_perf_app_many_rff} that although these methods attain meaningful improvements in the Frobenius and spectral norms of the approximation error, these improvements do not translate into corresponding gains in heldout performance.
Once again we see that $1/(1-\Delta_1)$ is able to much better explain the relative performance between different kernel approximation methods than Frobenius and spectral error.
In Figure~\ref{fig:many_rff_mem_feat} we see that LP-RFFs outperform these other methods in terms of performance under a memory budget.

\subsubsection{Large-Scale $(\Delta_1,\Delta_2)$ Experiments}
In Figure~\ref{fig:perc_delta} we plot the performance vs.\ $1/(1-\Delta_1)$ on the large-scale experiments from Section~\ref{subsec:nys_vs_rff_exp}.
We measure $\Delta_1$ using the exact and approximate kernel matrices on a random subset of 20k heldout points (except for Census, where we use the entire heldout set which has approximately 2k points).
We pick several $\lambda$ values between the smallest and largest eigenvalues of the exact (subsampled) training kernel matrix.
The strong alignment between generalization performance and $1/(1-\Delta_1)$ is quite robust to the choice of $\lambda$.

\subsubsection{Measuring $\Delta$ and $(\Delta_1,\Delta_2)$}
In order to measure the $\Delta_1$ and $\Delta_2$ between a kernel matrix $K$ and an approximation $\tK$, we first observe that the following statements are equivalent.
Note that to get the second statement, we multiply the expressions in the first statement by $(K + \lambda I_n)^{-1/2}$ on the left and right.
\begin{eqnarray*}
	(1-\Delta_1)(K + \lambda I_n) &\preceq& \tK + \lambda I_n \preceq (1 + \Delta_2)(K + \lambda I_n) \\
	(1-\Delta_1) I_n &\preceq& (K + \lambda I_n)^{-1/2} (\tK + \lambda I_n) (K + \lambda I_n)^{-1/2} \preceq (1 + \Delta_2) I_n \\
	-\Delta_1 I_n &\preceq& (K + \lambda I_n)^{-1/2} \Big(\tK + \lambda I_n - (K + \lambda I_n)\Big) (K + \lambda I_n)^{-1/2} \preceq \Delta_2 I_n \\
	-\Delta_1 I_n &\preceq& (K + \lambda I_n)^{-1/2} (\tK - K) (K + \lambda I_n)^{-1/2} \preceq \Delta_2 I_n.
\end{eqnarray*}
For $A=(K + \lambda I_n)^{-1/2}(\tK - K) (K + \lambda I_n)^{-1/2}$, this is equivalent to 
$-\Delta_1 \leq  \lambda_{min}\big(A\big)$ and $\lambda_{max}\big(A\big) \leq \Delta_2$.
Thus, we choose $\Delta_1 \defeq -\lambda_{min}\big(A\big)$ and $\Delta_2\defeq \lambda_{max}\big(A\big)$.
Lastly, we note that $\Delta = \max(\Delta_1,\Delta_2)$.

\begin{figure}
	\centering
	\begin{tabular} {c c c c}
	\includegraphics[width=0.23\linewidth]{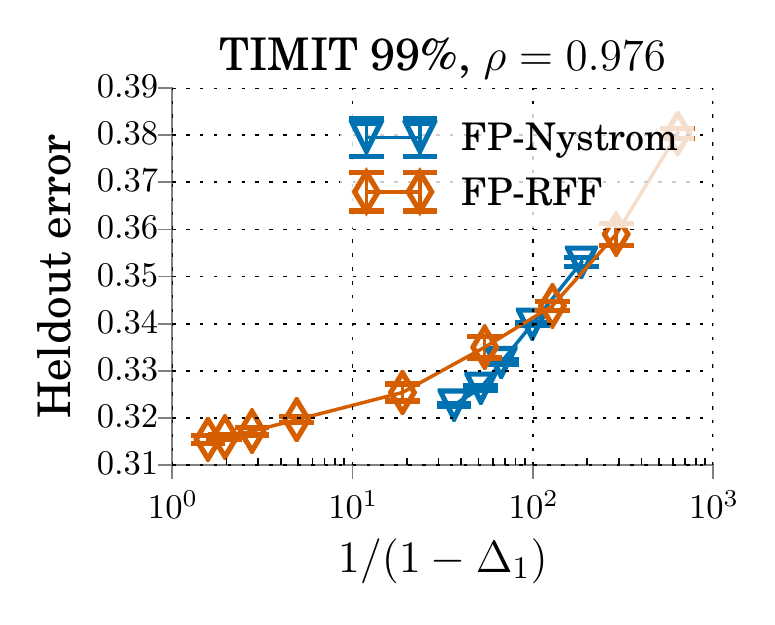} &
	\includegraphics[width=0.23\linewidth]{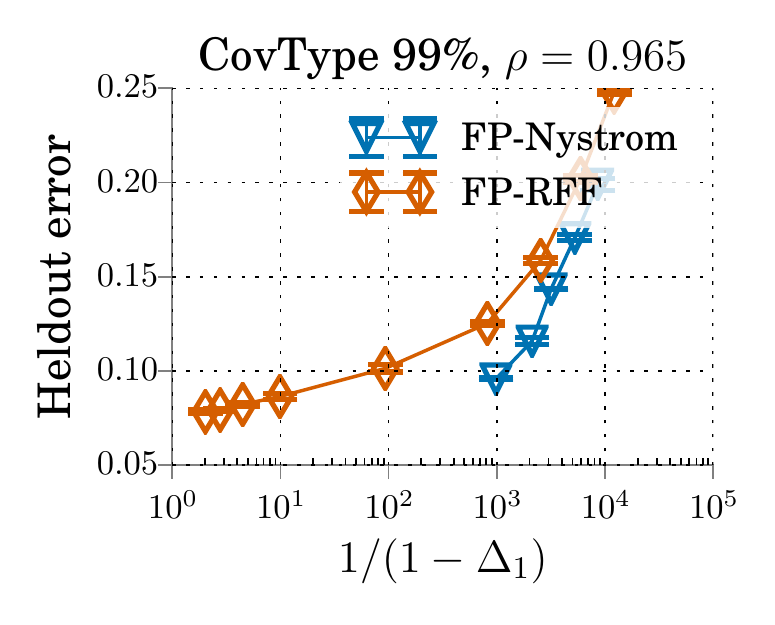} &
	\includegraphics[width=0.23\linewidth]{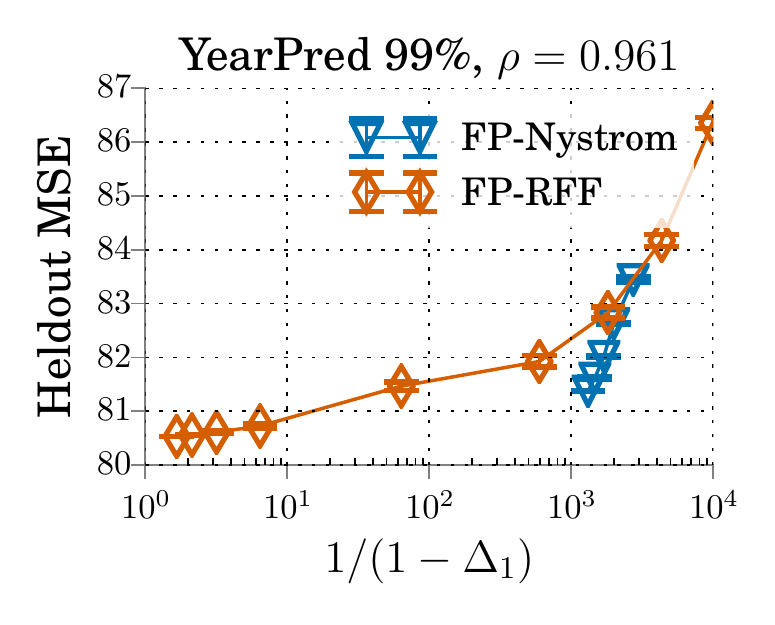} &
	\includegraphics[width=0.23\linewidth]{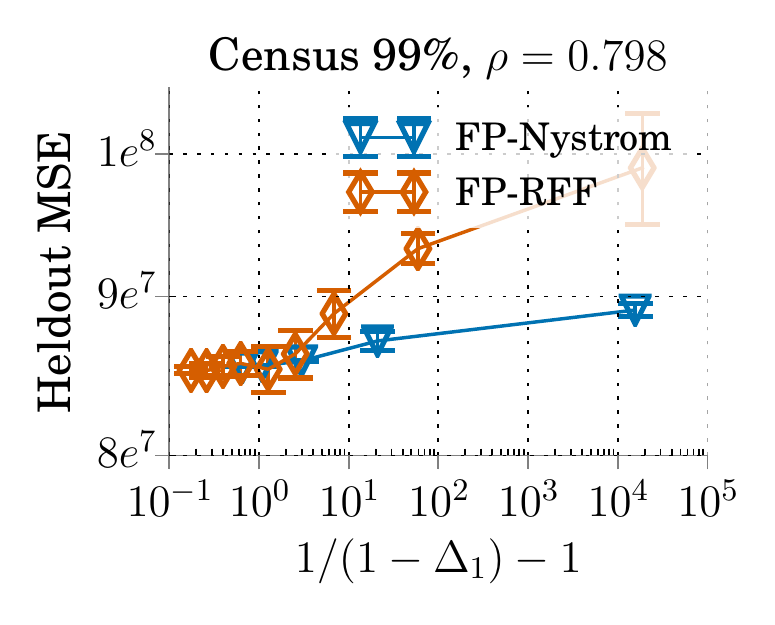} \\[-0.5em]
	\includegraphics[width=0.23\linewidth]{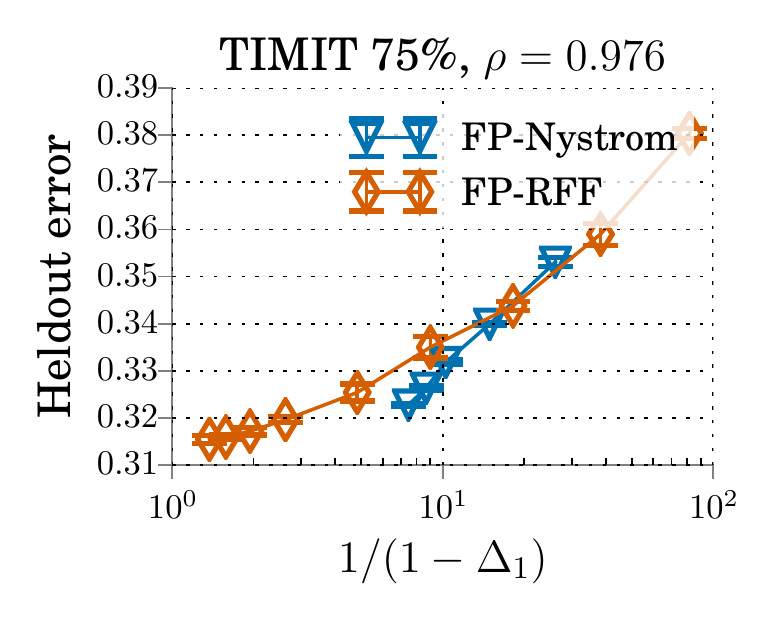} &
	\includegraphics[width=0.23\linewidth]{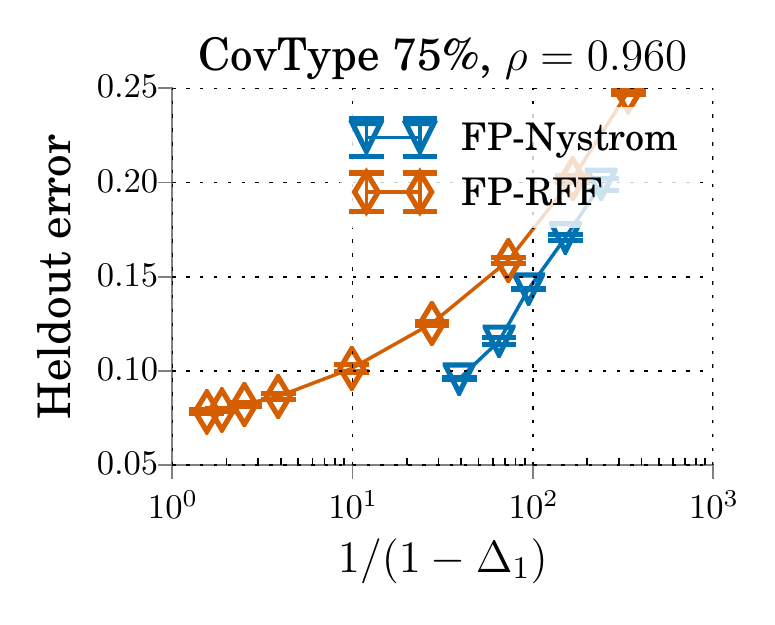} &
	\includegraphics[width=0.23\linewidth]{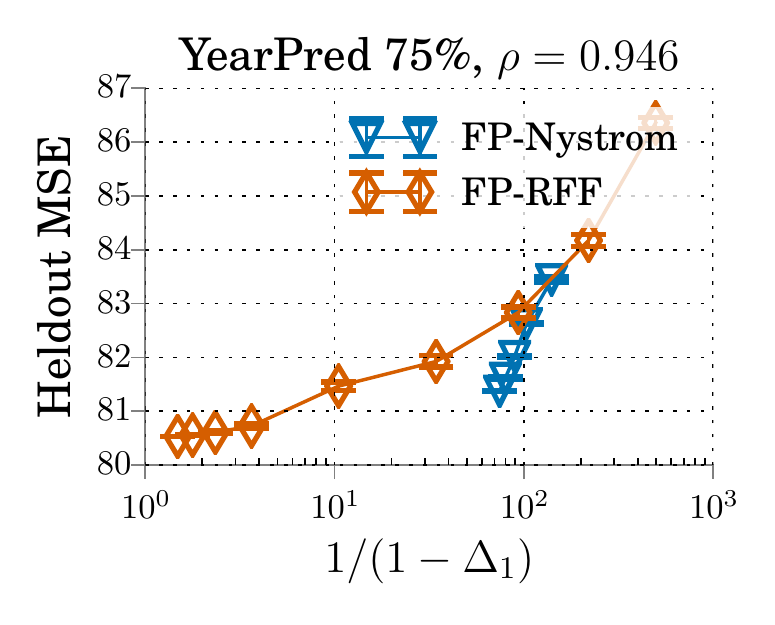} &
	\includegraphics[width=0.23\linewidth]{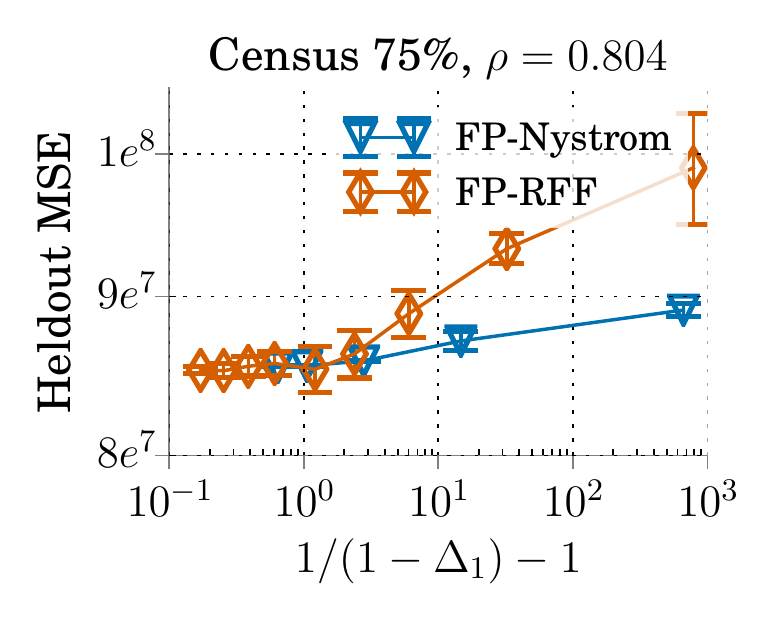} \\[-0.5em]
	\includegraphics[width=0.23\linewidth]{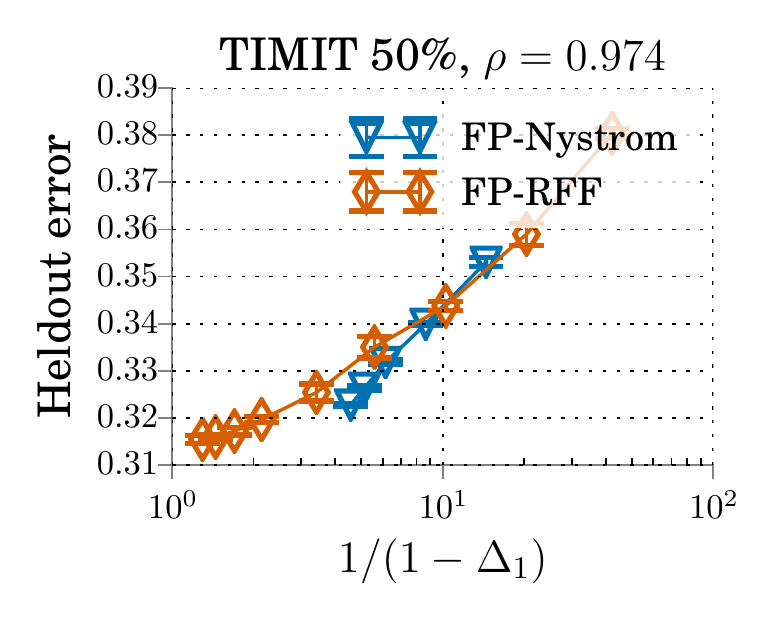} &
	\includegraphics[width=0.23\linewidth]{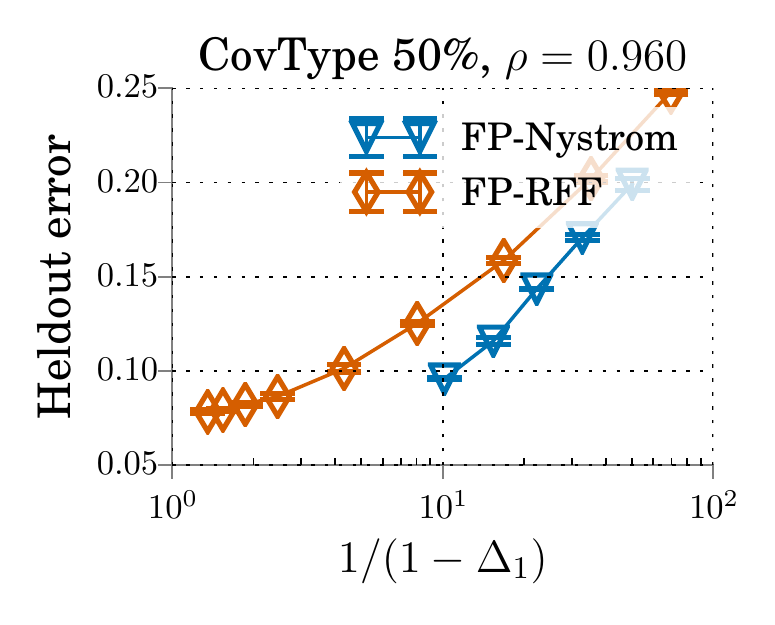} &
	\includegraphics[width=0.23\linewidth]{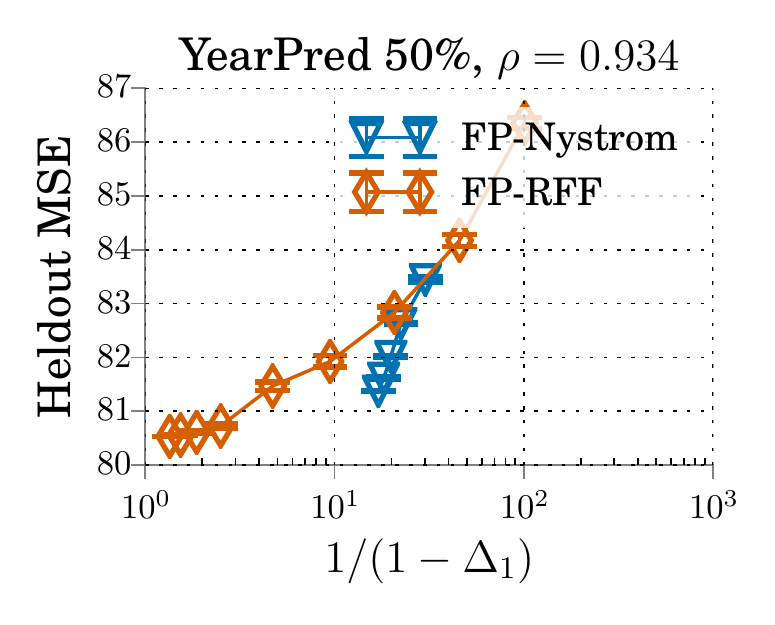} &
	\includegraphics[width=0.23\linewidth]{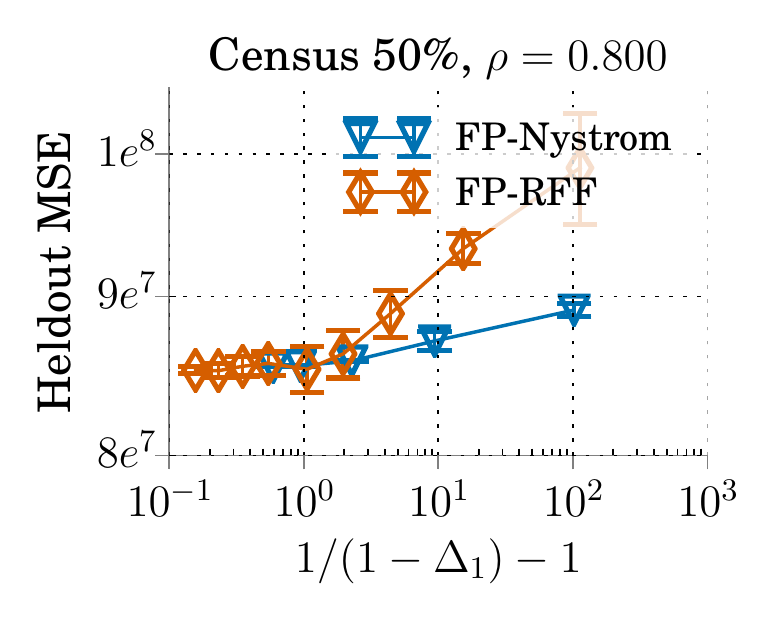} \\[-0.5em]
	\includegraphics[width=0.23\linewidth]{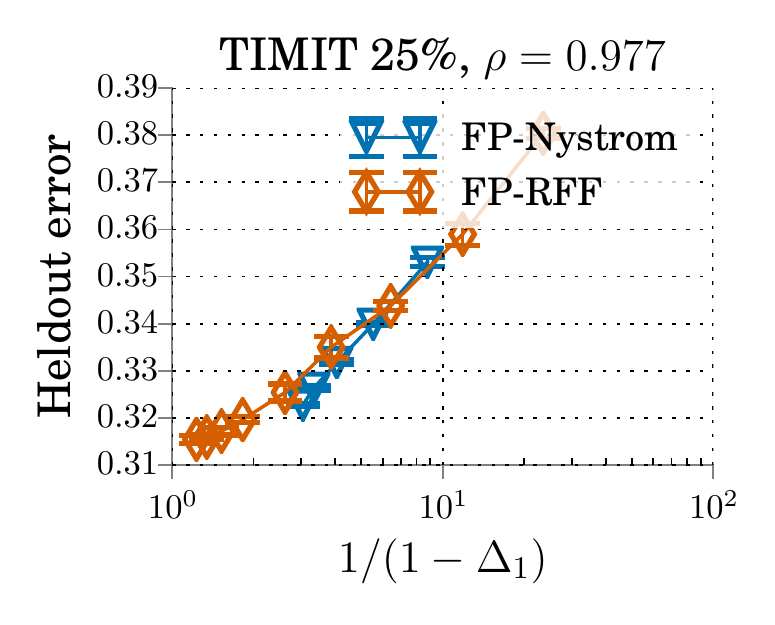} &
	\includegraphics[width=0.23\linewidth]{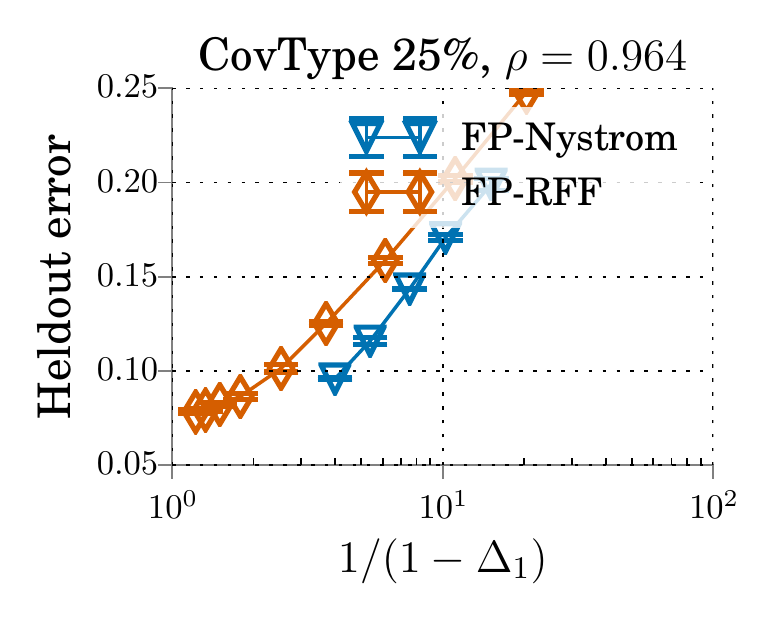} &
	\includegraphics[width=0.23\linewidth]{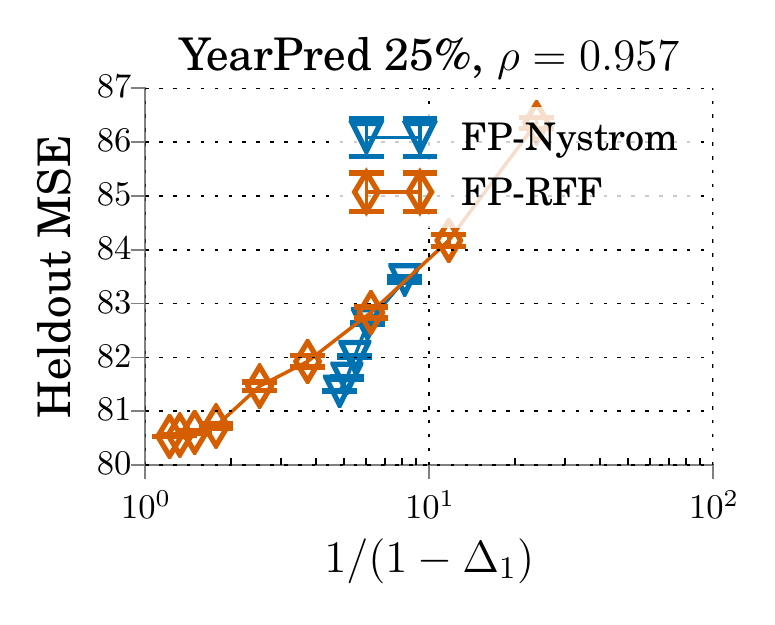} &
	\includegraphics[width=0.23\linewidth]{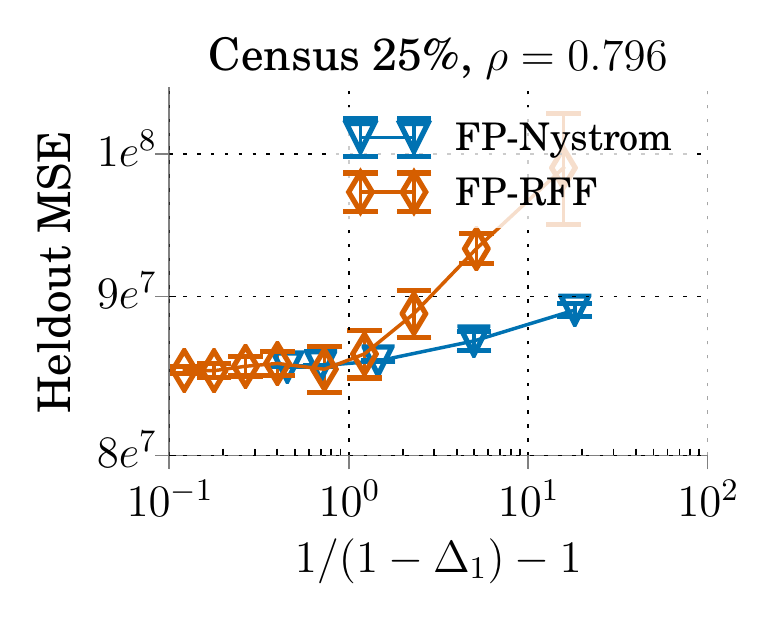} \\[-0.5em]
	\includegraphics[width=0.23\linewidth]{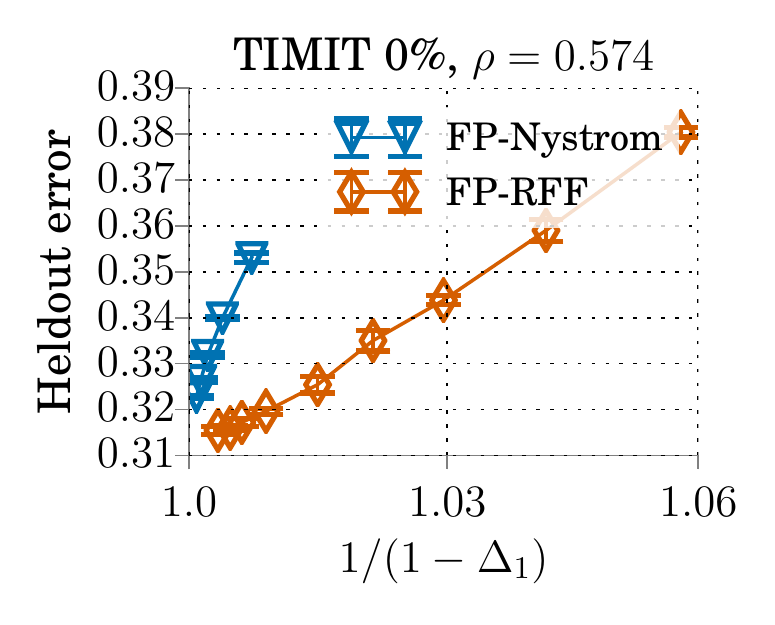} &
	\includegraphics[width=0.23\linewidth]{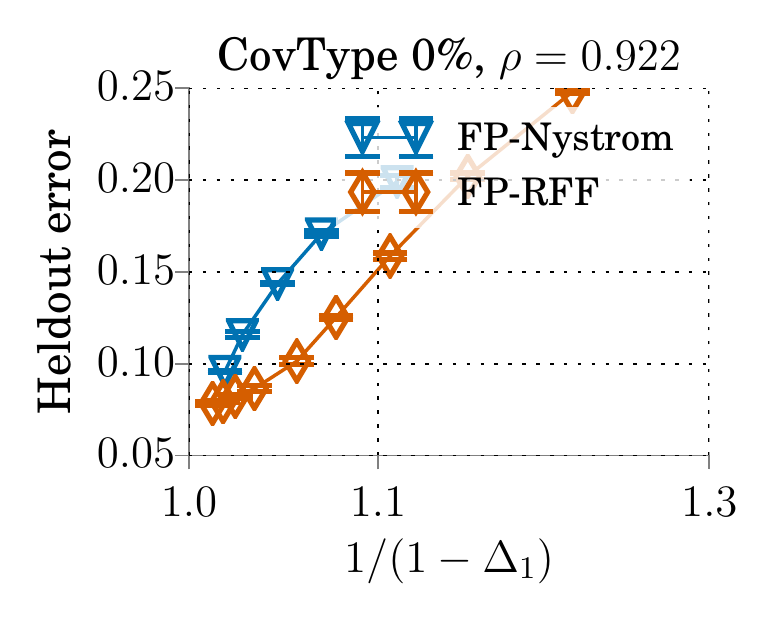} &
	\includegraphics[width=0.23\linewidth]{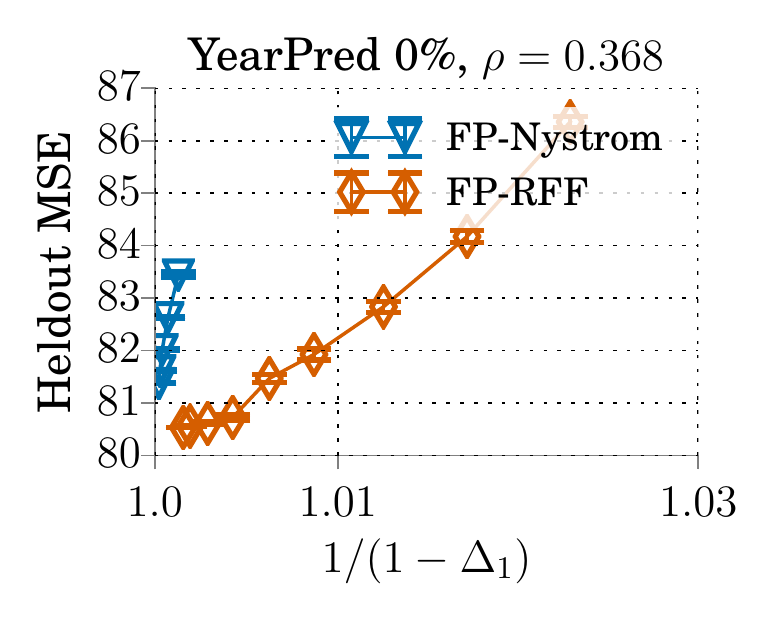} &
	\includegraphics[width=0.23\linewidth]{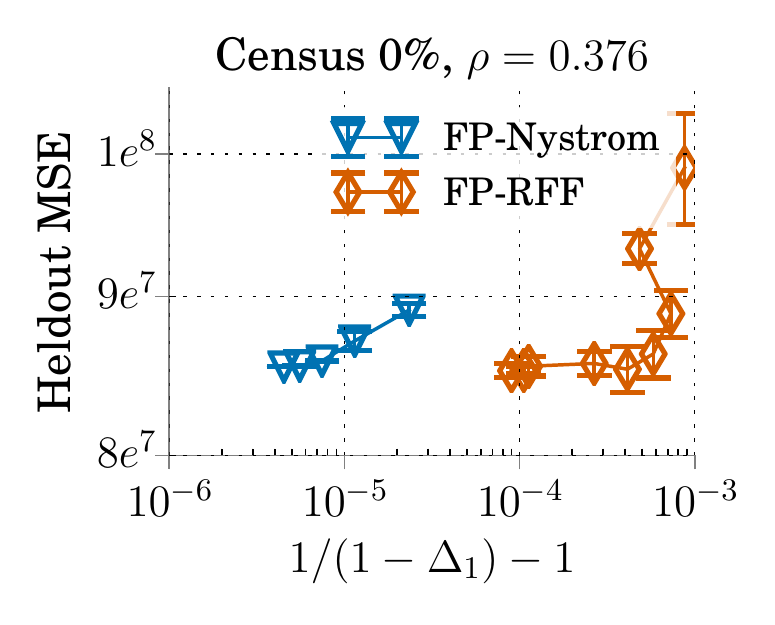} \\[-0.5em]
	\end{tabular}
	\caption{Generalization performance vs.\ $\frac{1}{1-\Delta_1}$, where we measure $\Delta_1$ using regularizer strength $\lambda$ equal to the 0, 25, 50, 75, or 99 percentile eigenvalues of the exact kernel matrix (0 percentile indicates largest eigenvalue). Note for the Census dataset, we plot the heldout MSE as a function of $1/(1-\Delta_1) - 1$ to avoid cluttering the data points on the left end of the figure. For comparison to spectral and Frobenius norm plots, see Figure~\ref{fig:generalization_col_app}.
	To quantify the degree of alignment between $1/(1-\Delta_1)$ and generalization performance for these different values of $\lambda$, we compute the Spearman rank correlation coefficients $\rho$.
	We see $1/(1-\Delta_1)$ generally attains much higher values of $\rho$ than the Frobenius and spectral approximation errors (Figure~\ref{fig:generalization_col_app}).
	Although we plot performance averaged across three random seeds (error bars indicate standard deviations), when we compute $\rho$ we treat each experimental result independently.
	}
	\label{fig:perc_delta}
\end{figure}

\subsection{Theory Validation (Section~\ref{subsec:theory})}
\label{app:theory_validation_details}
To validate our theory in Section~\ref{subsec:theory}, we perform two sets of experiments to (1) demonstrate the asymptotic behavior of $\Delta_1$ and $\Delta_2$ as the number of features increases, and (2) demonstrate that quantization has negligible effect on $\Delta_2$ when $\delta_b^2/\lambda \ll \Delta_2$.

To demonstrate the behavior of $\Delta_1$ and $\Delta_2$ as a function of the number of features, we collect $\Delta_1$, and $\Delta_2$ using \Nystrom features, circulant FP-RFFs, and LP-RFFs using $b \in \{1, 4, 8\}$, on both the Census and the sub-sampled CovType datasets (20k random heldout points).
For each approximation, we sweep the number of features as listed in Table~\ref{table:n_feat_grid_theory_val}. We use the same value of $\lambda$ as we used in Section~\ref{subsec:nys_vs_rff_revisited} for each dataset.
In Figure~\ref{fig:theory_supporting_app}, we plot the values of $\Delta_1$ and $\Delta_2$ attained by these methods as a function of the number of features.
As discussed in Section~\ref{subsec:theory}, $\Delta_1$ is primarily determined by the rank of the approximation, and approaches 0 for all the methods as the number of features grows.
$\Delta_2$, on the other hand, only approaches 0 for the high-precision methods---for $b\in\{1,4\}$, $\Delta_2$ converges to higher values (at most $\delta_b^2/\lambda$, marked by dashed lines).

To demonstrate that quantization has negligible effect on $\Delta_2$ when $\delta_b^2/\lambda \ll \Delta_2$, we use 8000 random training points from the Census dataset.
For $\lambda \in \{10^{-4},10^0,10^4\}$, and $b\in\{1,2,4,8,16,32\}$, we measure the $\Delta_2$ attained by the LP-RFFs relative to the exact kernel matrix.
We see that for larger $\lambda$ (corresponding to smaller $\delta_b^2/\lambda$) in Figure~\ref{fig:theory_supporting}, lower precisions can be used while not influencing $\Delta_2$ significantly; this aligns with the theory.
\begin{figure*}
	\centering
	\begin{small}
\vfigsp
		\begin{tabular}{c c}
			\includegraphics[height=0.23\linewidth]{figures/regression_delta_left_vs_n_feat_append.pdf} &      
			\includegraphics[height=0.23\linewidth]{figures/regression_delta_right_vs_n_feat.pdf} \\
			\includegraphics[height=0.23\linewidth]{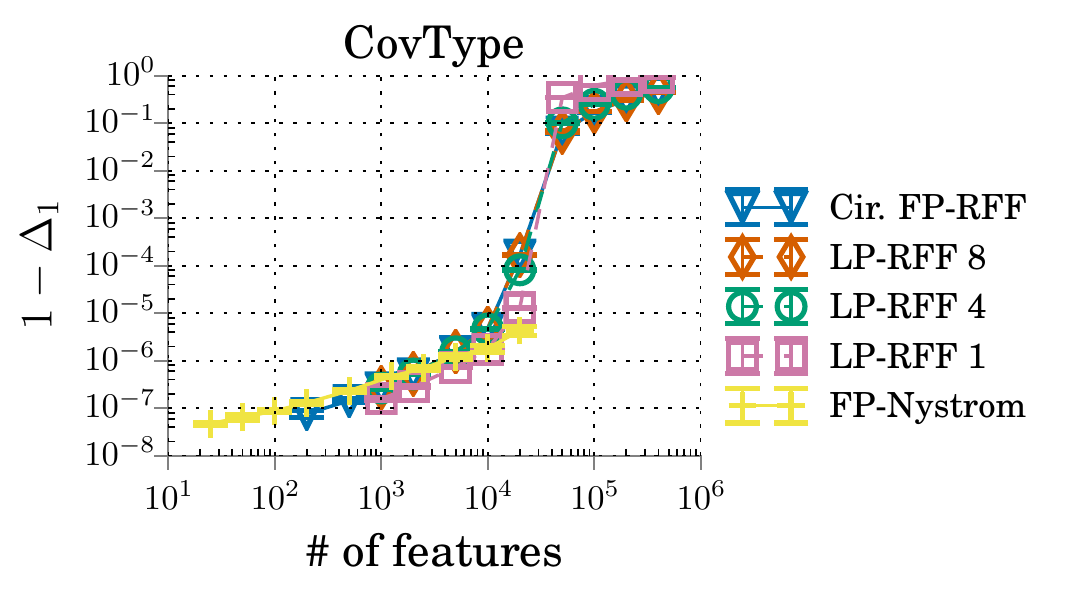} &      
			\includegraphics[height=0.23\linewidth]{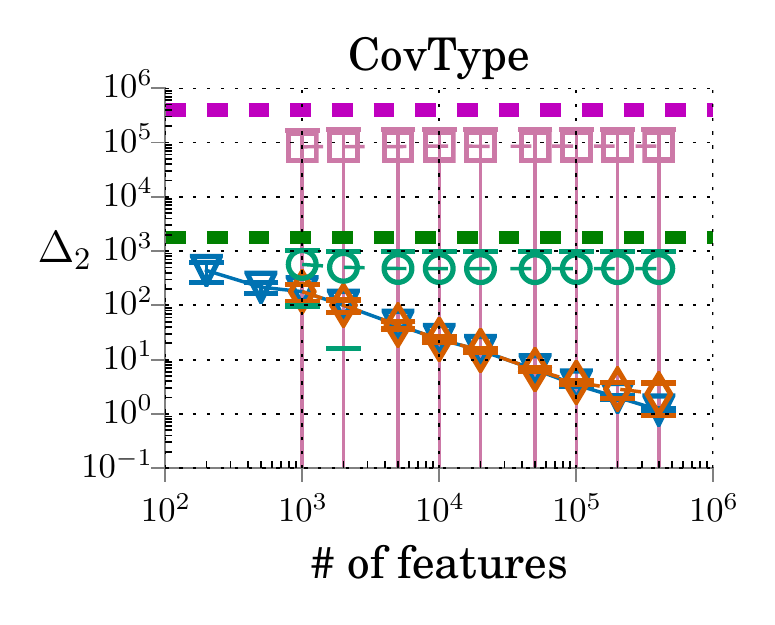} \\
		\end{tabular}
	\end{small}
\vfigsp
	\caption{Empirical validation of Theorem~\ref{thm2}. As the number of features grows, LP-RFFs approach $\Delta_1$ values of 0 (left), but plateau at larger $\Delta_2$ for very low precisions (right).
	These are an extended version of the results from Figure~\ref{fig:theory_supporting} (left, middle) in Section~\ref{subsec:theory}.
	We plot average results across five random seeds, with error bars indicating standard deviations.
}
	\label{fig:theory_supporting_app}
\end{figure*}

\begin{table}
	\caption{Number of features used for the different kernel approximation methods in the theory validation experiments in Section~\ref{subsec:theory} (Figure~\ref{fig:theory_supporting} (left,middle)).}
	\centering
	\begin{tabular}{c | c}
		\toprule
		Methods & Number of features \\
		\midrule
		FP-\Nystrom & $25, 50, 100, 200, 500, 1250, 2500, 5000, 10000, 20000$ \\
		Cir. FP-RFF & $200, 500, 1000, 2000, 5000, 10000, 20000, 50000, 100000, 200000, 400000$ \\
		LP-RFF $8, 4, 2, 1$ & $1000, 2000, 5000, 10000, 20000, 50000, 100000, 200000, 400000$ \\
		\bottomrule
	\end{tabular}
	\label{table:n_feat_grid_theory_val}
\end{table}

\begin{figure}
	\centering
	\vspace{-4em}
	\begin{tabular}{c c}
		\includegraphics[height=0.275\linewidth]{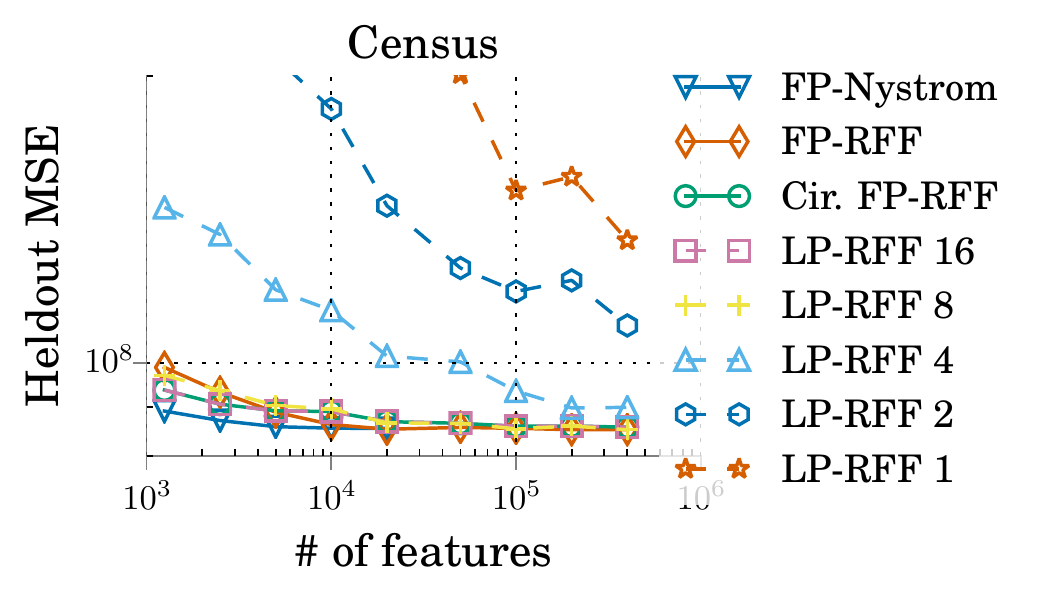} &
		\includegraphics[height=0.275\linewidth]{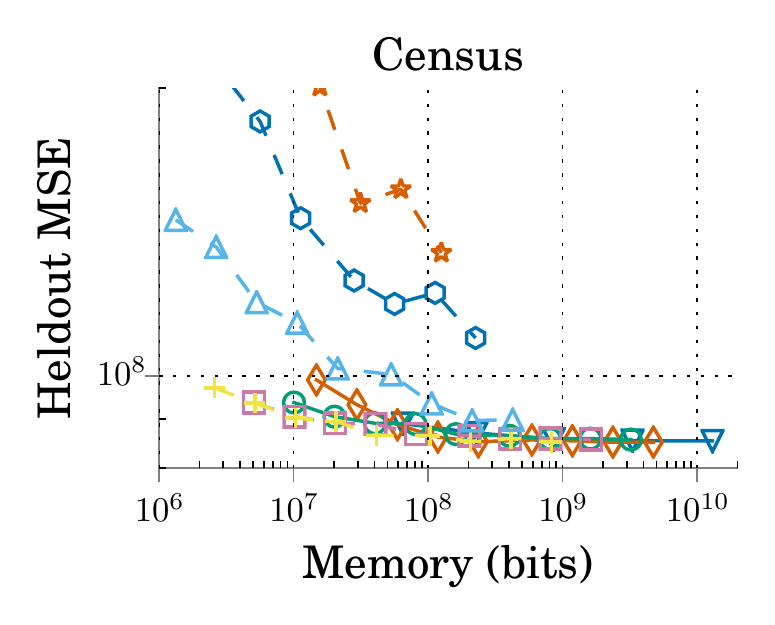} \\ [-0.5em]
		\includegraphics[height=0.275\linewidth]{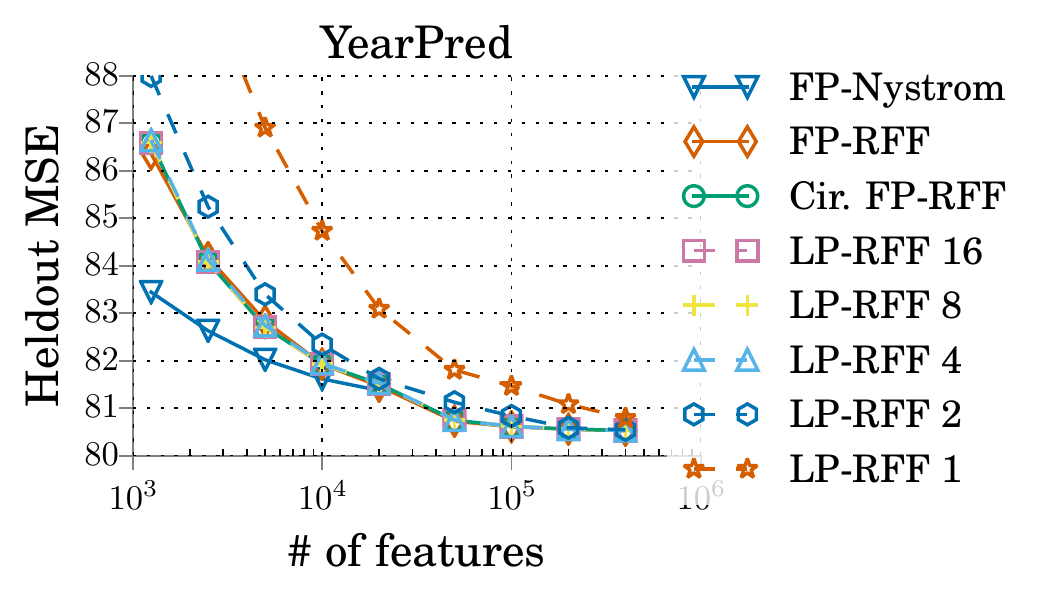} &
		\includegraphics[height=0.275\linewidth]{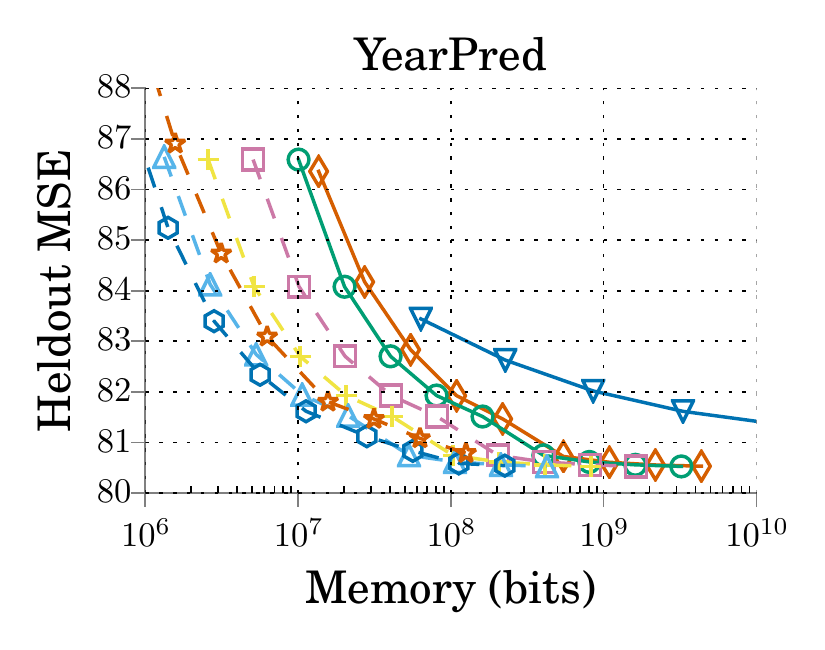} \\ [-0.5em]
		\includegraphics[height=0.275\linewidth]{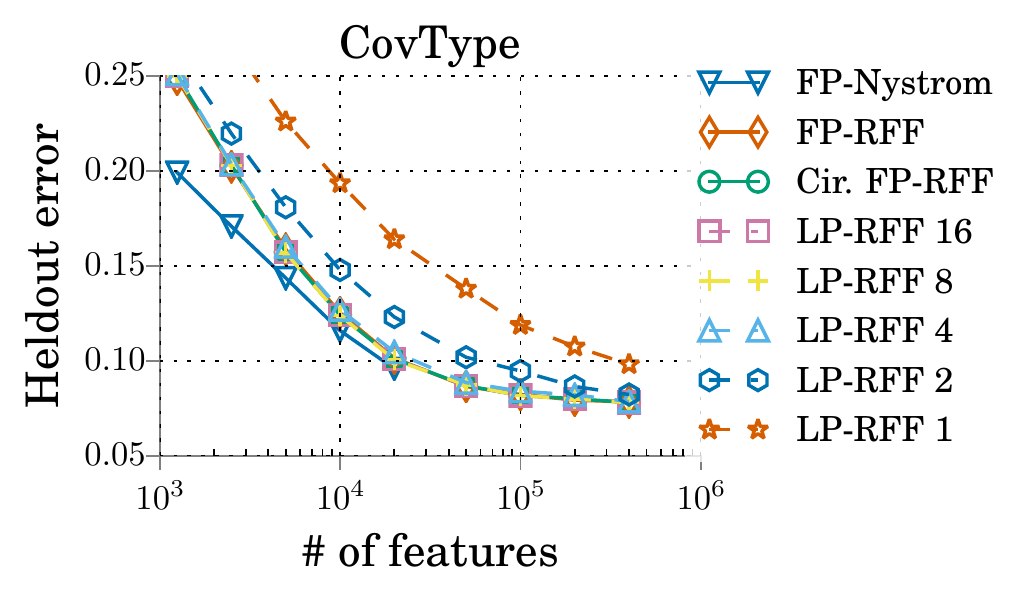} &
		\includegraphics[height=0.275\linewidth]{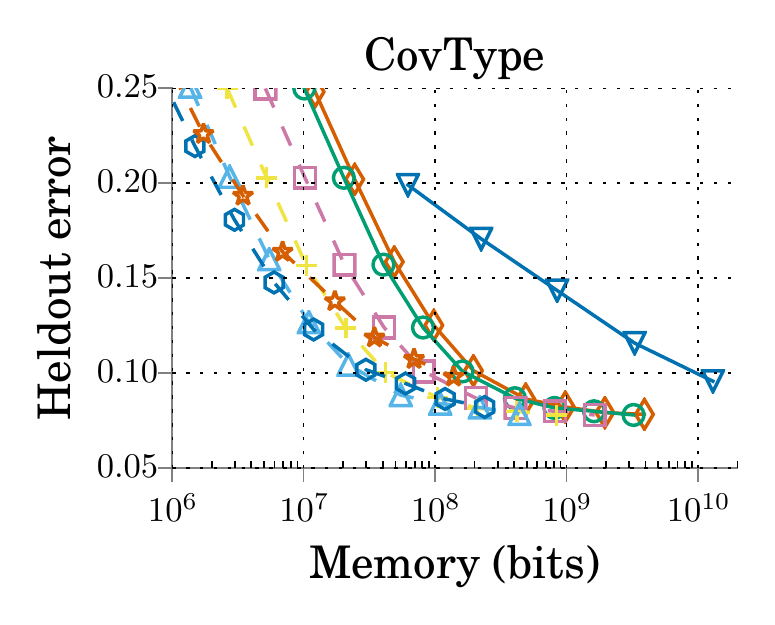} \\ [-0.5em]
		\includegraphics[height=0.275\linewidth]{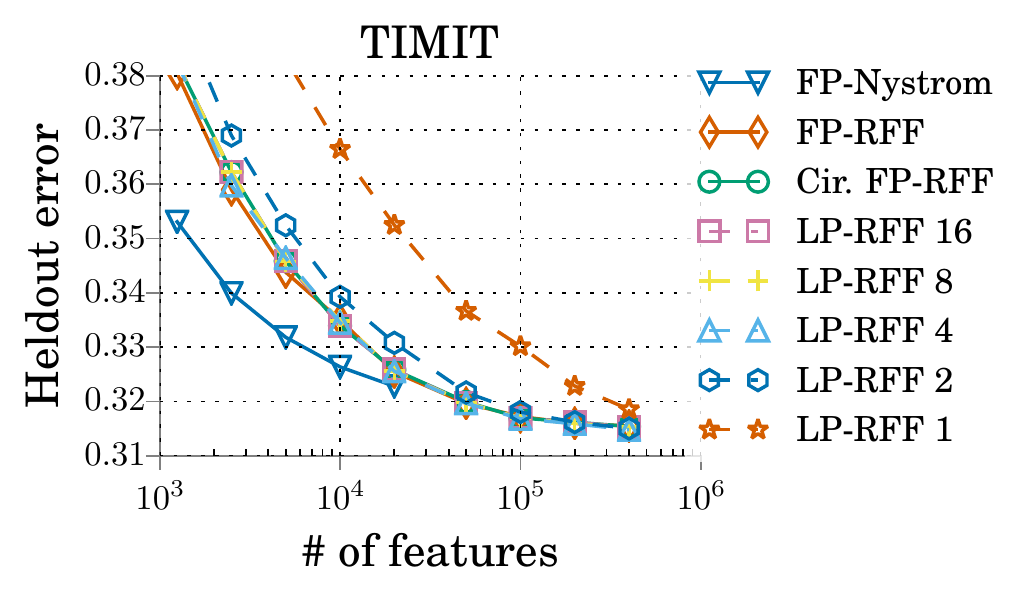} &
		\includegraphics[height=0.275\linewidth]{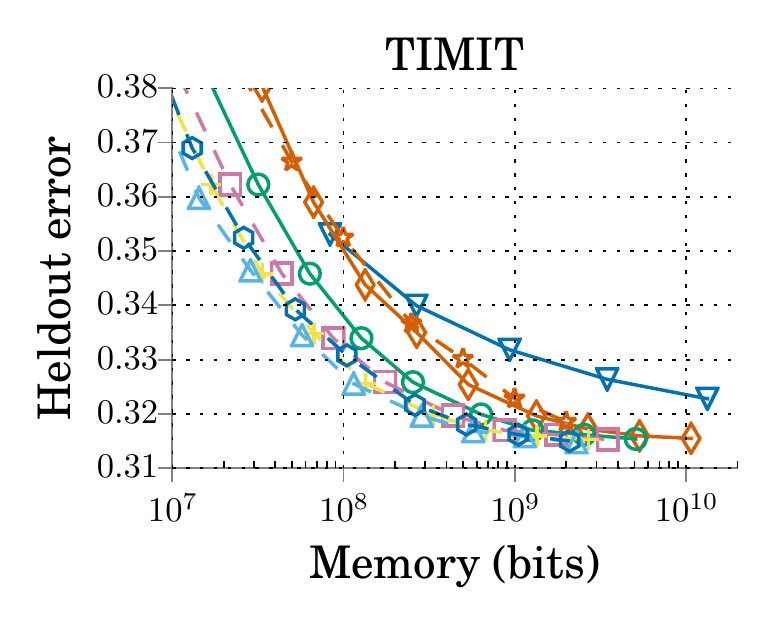} \\ [-1em]
	\end{tabular}
	\caption{Generalization performance of the kernel approximation methods (\NystromNS, FP-RFFs, circ. FP-RFFs, LP-RFFs) with respect to the number of features used, as well as with respect to memory used.
		We observe that LP-RFFs demonstrate better generalization performance than the full-precision baselines under memory constraints, with 2-8 bits typically giving the best performance.
		Importantly, the relative ranking of the methods changes depending on whether we compare the methods based on their number of features, or their memory utilization.
		For example, the \Nystrom method shows better generalization performance than the RFF-based approaches with the same number of features.
		However, the \Nystrom method often performs significantly worse than the RFF methods under fixed memory budgets.
		We plot results averaged over three random seeds.
	}
	\label{fig:all_line_generalization_col_app}
\end{figure}

\subsection{Empirical Evaluation of LP-RFFs (Section~\ref{subsec:full_run})}
\label{app:lprff_exp_details}
To empirically demonstrate the generalization performance of LP-RFFs, we compare LP-RFFs to FP-RFFs, circulant FP-RFFs, and \Nystrom features for various memory budgets.
We use the same datasets as in Section~\ref{sec:nys_vs_rff}, including Census and YearPred for regression, as well as CovType and TIMIT for classification.
We use the same experimental protocol as in Section~\ref{sec:nys_vs_rff} (details in Appendix~\ref{app:nys_vs_rff_details}), with the only significant additions being that we also evaluate the performance of circulant FP-RFFs, and LP-RFFs for precisions $b \in \{1,2,4,8,16\}$.
As noted in the main text, all our LP-RFF experiments are done in \textit{simulation}; 
in particular, we represent each low-precision feature as a 64-bit floating point number, whose value is one of the $2^b$ values representable in $b$ bits.
For our full-precision experiments we also use 64-bit floats, but we report the memory utilization of these experiments as if we used 32 bits, to avoid inflating the relative gains of LP-RFFs over the full-precision approaches.
We randomly sample the quantization noise for each mini-batch independently each epoch.

In Section~\ref{subsec:full_run} (Figure~\ref{fig:generalization_col_full}), we demonstrated the generalization performance of LP-RFFs on the TIMIT, YearPred, and CovType datasets, using 4 bits per feature.
In Figure~\ref{fig:all_line_generalization_col_app}, we additionally include results on the Census dataset, and include results for a larger set of precisions ($b\in\{1,2,4,8,16\}$).
We also include plots of heldout performance as a function of the number of features used.
We observe that LP-RFFs, using 2-8 bits, systematically outperform the full-precision baselines under different memory budgets.

We run on four additional classification and regression datasets (Forest, Cod-RNA, Adult, CPU) to compare the empirical performance of LP-RFFs to full-precision RFFs, circulant RFFs and \NystromNS.
We present the results in Figure~\ref{fig:add_datasets}, and observe that LP-RFFs can achieve competitive generalization performance to the full-precision baselines with lower training memory budgets.
With these 4 additional datasets, our empirical evaluation of LP-RFFs now covers all the datasets investigated in~\cite{nysvsrff12}.
We include details about these datasets and the hyperparameters we used in Tables~\ref{tab:add_dataset_details} and \ref{tab:add_kernel_hyper}.

\begin{table}
	\caption{Dataset details for the additional datasets from Section~\ref{app:lprff_exp_details}.  For classification tasks, we write the number
		of classes in parentheses in the ``Task'' column.}
	\begin{center}
		\begin{tabular}{llllll} 
			\toprule
			\textbf{Dataset}  & \textbf{Task} & \textbf{Train} & \textbf{Heldout} & \textbf{Test} & \textbf{\# Features} \\ 
			\midrule
			Forest  & Class. (2) & 470k  & 52k  & 58k   & 54  \\ 
			Cod-RNA & Class. (2) & 54k   & 6k   & 272k  & 8   \\ 
			Adult   & Class. (2) & 29k   & 3k   & 16k   & 123 \\
			CPU     & Reg.       & 6k    & 0.7k & 0.8k  & 21  \\ 
			\bottomrule
		\end{tabular}
	\end{center}
	\label{tab:add_dataset_details}
\end{table}

\begin{table}
	\caption{The Gaussian kernel bandwidths used, and the search grid for initial learning rate on the Forest, Cod-RNA, Adult, and CPU datasets. Optimal learning rate in bold.}
	\begin{center}
		\begin{tabular}{lll}
			\toprule
			Dataset & $1/2\sigma^2$ & Initial learning rate grid \\
			\midrule
			Forest & 0.5 & {5.0, 10.0, \textbf{50.0}, 100.0, 500.0} \\
			Cod-RNA & 0.4 & {10.0, 50.0, \textbf{100.0}, 500.0, 1000.0} \\
			Adult & 0.1 & {5.0, \textbf{10.0}, 50.0, 100.0, 500.0, 1000.0}  \\
			CPU & 0.03 & {0.05, 0.1, \textbf{0.5}, 1.0, 5.0} \\
			\bottomrule
		\end{tabular}
	\end{center}
	\label{tab:add_kernel_hyper}
\end{table}

\begin{figure}
	\centering
	\begin{tabular}{c c}
		\includegraphics[width=0.35\linewidth]{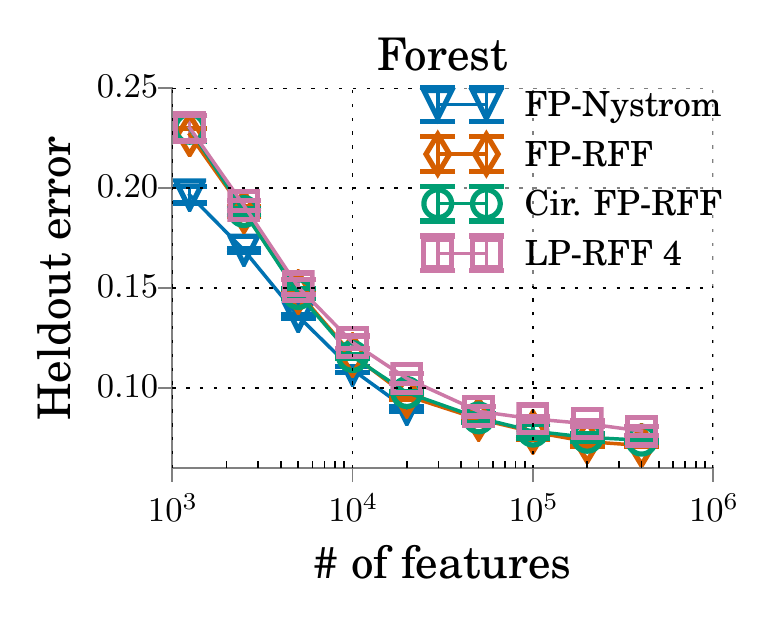} &
		\includegraphics[width=0.35\linewidth]{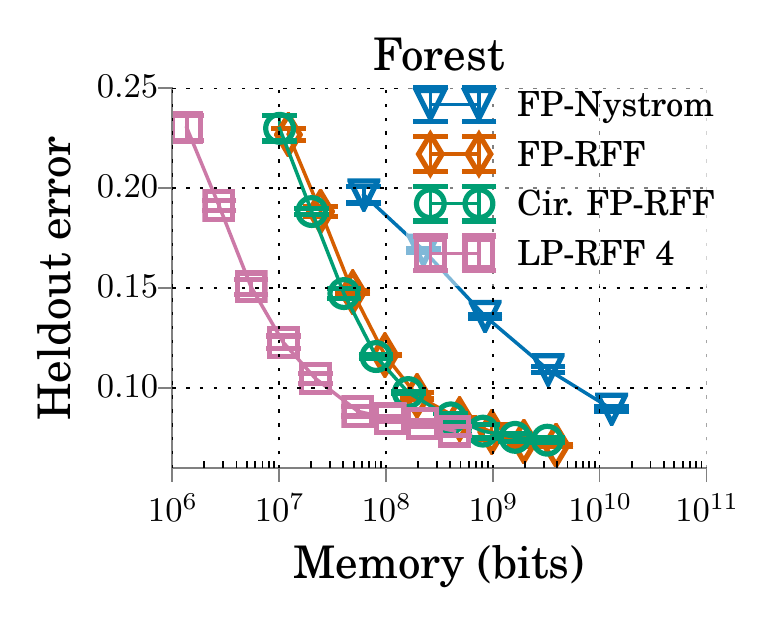} \\
		\includegraphics[width=0.35\linewidth]{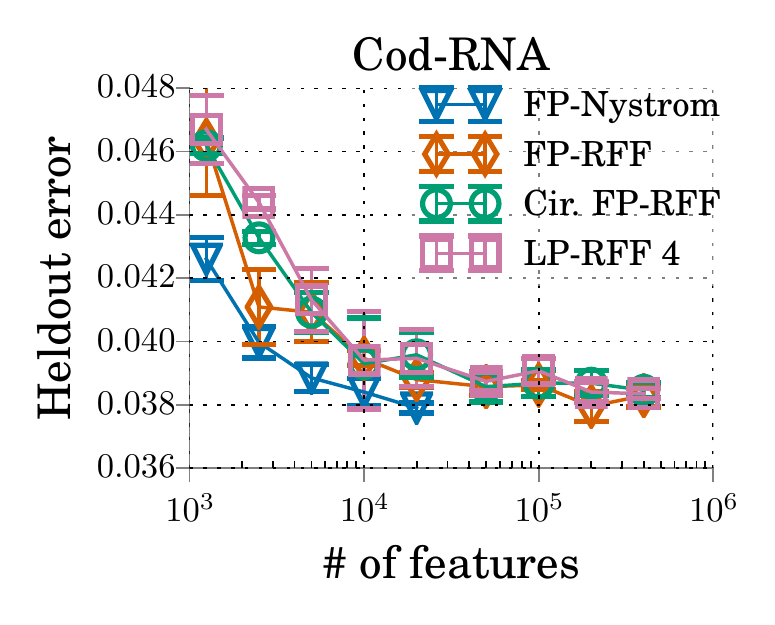} &
		\includegraphics[width=0.35\linewidth]{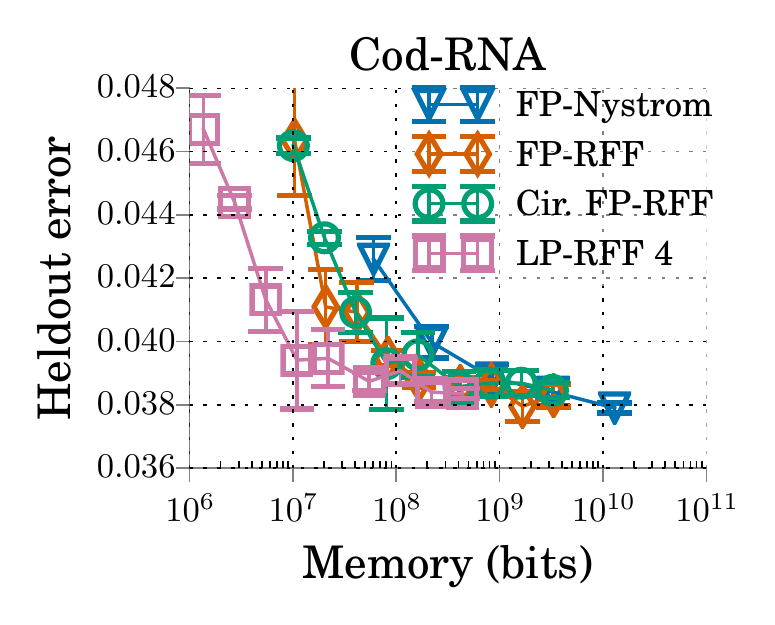} \\
		\includegraphics[width=0.35\linewidth]{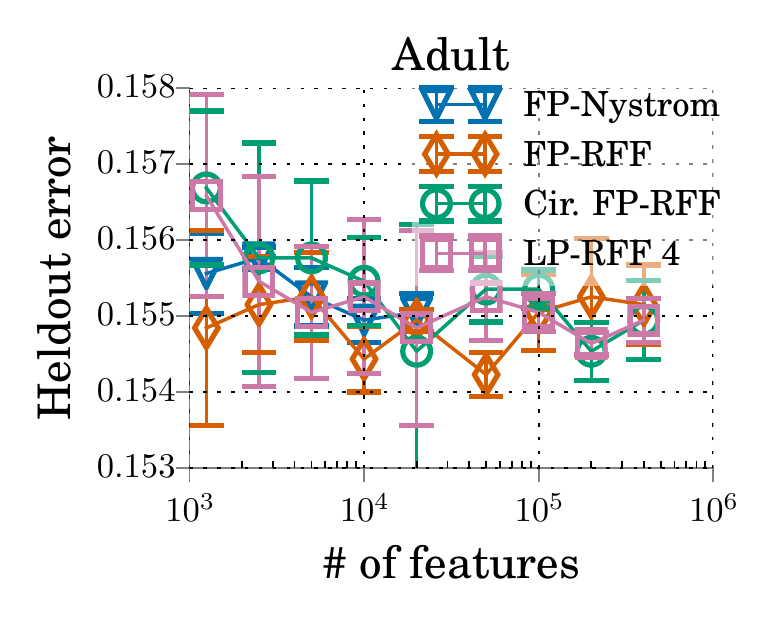} &
		\includegraphics[width=0.35\linewidth]{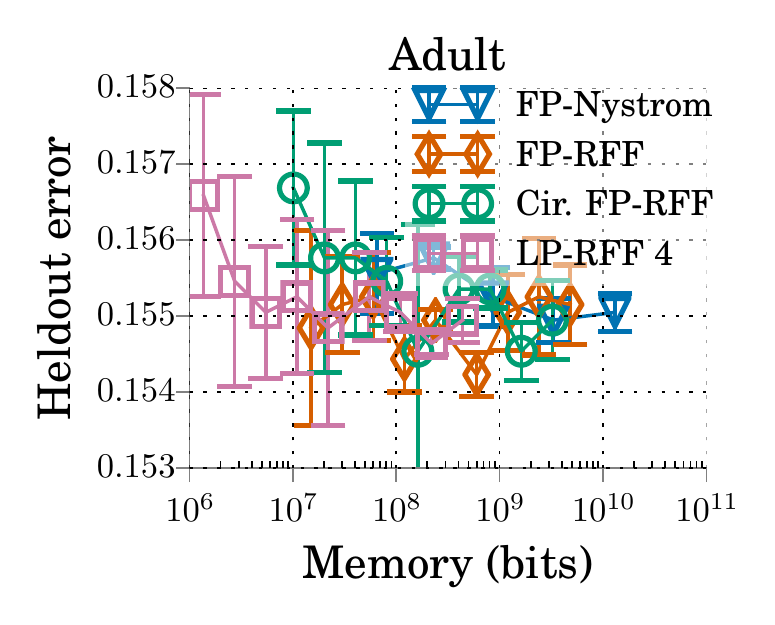} \\
		\includegraphics[width=0.35\linewidth]{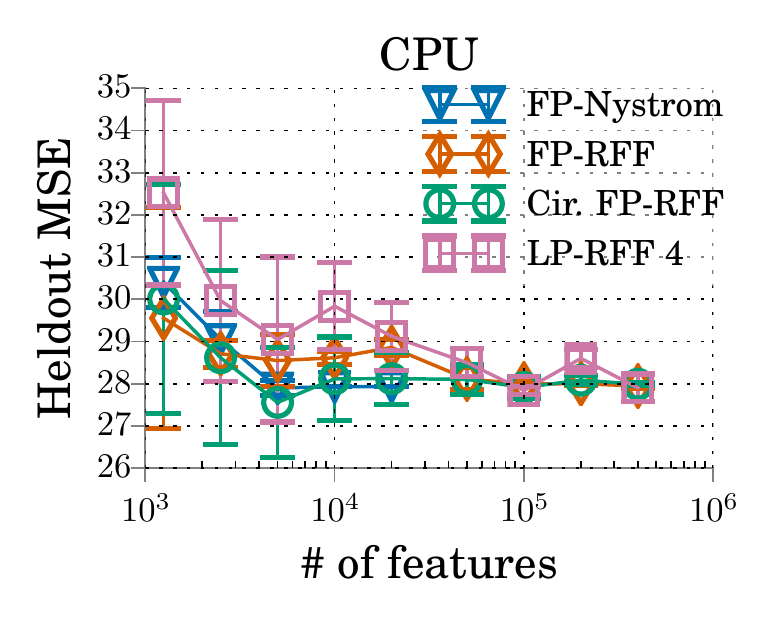} &
		\includegraphics[width=0.35\linewidth]{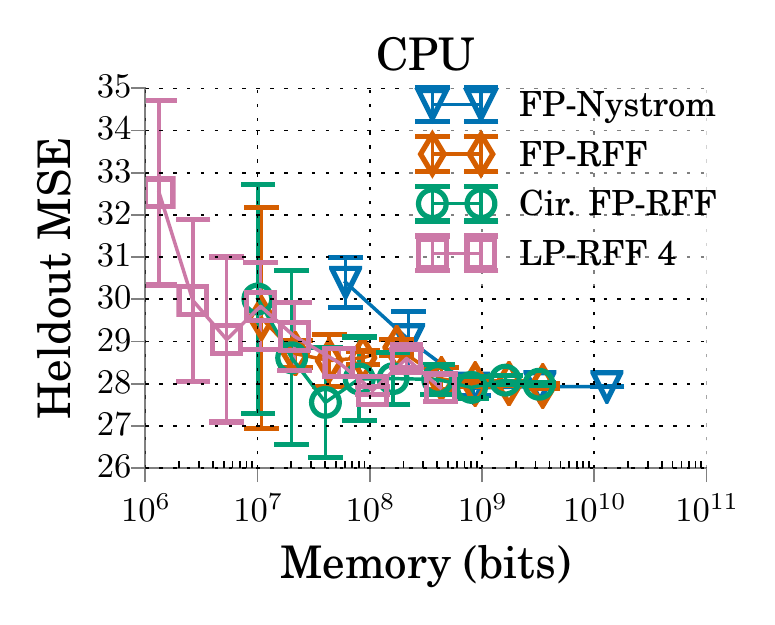} \\
	\end{tabular}
	\caption{Comparison of the performance of LP-RFFs and the full-precision baselines on additional datasets, with respect to the number of features used, as well as with respect to memory used. We observe that LP-RFFs can achieve performance competitive with the full-precision baseline methods, with significant memory savings.
	We plot the average performance across three random seeds, with error bars indicating standard deviations.
	}
	\label{fig:add_datasets}
\end{figure}

\subsection{Generalization Performance vs. $(\Delta_1, \Delta_2)$ (Section~\ref{subsec:gen_perf_and_delta})}
\label{subsec:gen_perf_and_delta_app}
For our experiments in Section~\ref{subsec:gen_perf_and_delta}, we use the same protocol and hyperparameters (learning rate, $\lambda$) as for the $(\Delta_1,\Delta_2)$ experiments in Section~\ref{subsec:nys_vs_rff_revisited}.
However, we additionally run experiments with LP-RFFs for precisions $b \in \{1,2,4,8,16\}$.
We use five random seeds for each experimental setting, and plot the average results, with error bars indicating standard deviations.

In Figure~\ref{fig:delta_max_perf_app}, we plot an extended version of the right plots in Figure~\ref{fig:gen_delta_correlation}.
We include results for all precisions, and for both Census and CovType.
We observe that on Census, $1/(1-\Delta_1)$ does not align well with performance (Spearman rank correlation coefficient $\rho=0.403$), because the low-precision features ($b=1$ or $b=2$) perform significantly worse than the full-precision features of the same dimensions.
In this case, when we consider the impact of $\Delta_2$ by taking $\max\big(1/(1-\Delta_1), \Delta_2\big)$, we see that performance aligns much better ($\rho=0.959$).

On CovType, on the other hand, the impact on generalization performance from using low-precision is much less pronounced, so $1/(1-\Delta_1)$ and $\max\big(1/(1-\Delta_1), \Delta_2\big)$ both align well with performance ($\rho = 0.942$).
Furthermore, in the case of CovType, $\Delta_2$ is generally smaller than $1/(1-\Delta_1)$, so taking the $\max$ does not change the plot significantly.

To compute the Spearman rank correlation coefficients $\rho$ for these plots, we take the union of all the experiments which are a part of the plot.
In particular, we include FP-RFFs, circulant FP-RFFs, FP-\Nystrom, and LP-RFFs with precisions $b\in\{1,2,4,8,16\}$.
Although we plot the average performance across five random seeds in the plot, to compute $\rho$ we treat each experiment independently.

\begin{figure}
	\centering
	\begin{tabular}{c c} 
		\includegraphics[height=0.25\linewidth]{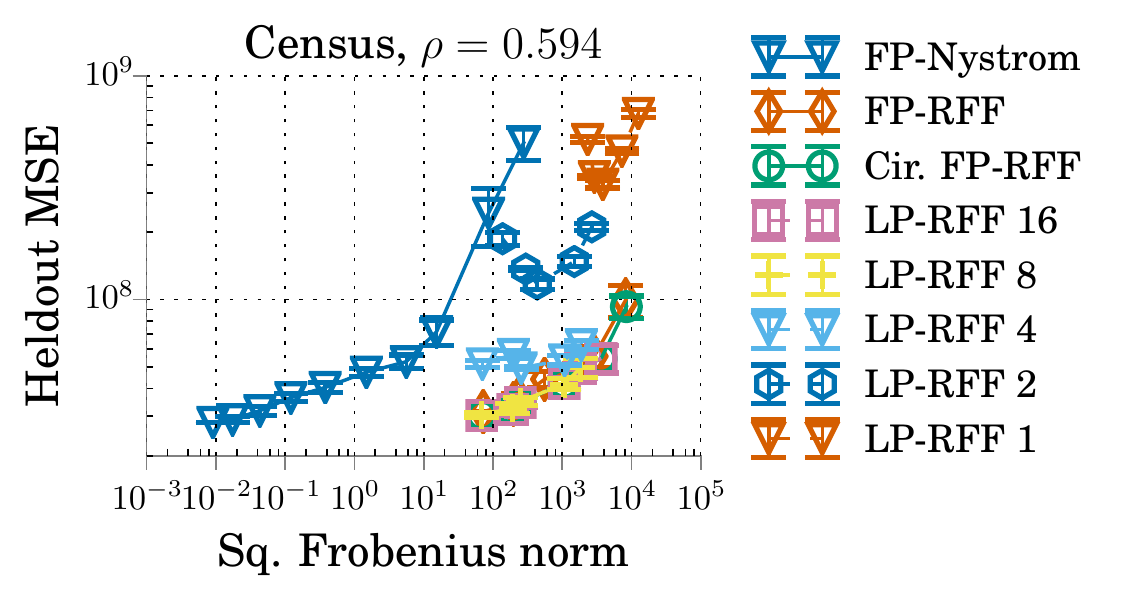} &
		\includegraphics[height=0.25\linewidth]{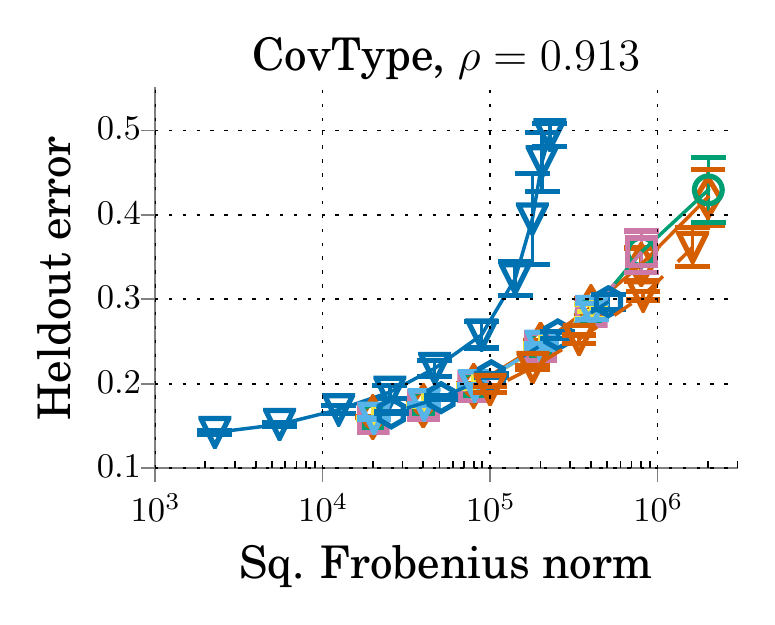} \\ [-0.5em]
		\includegraphics[height=0.25\linewidth]{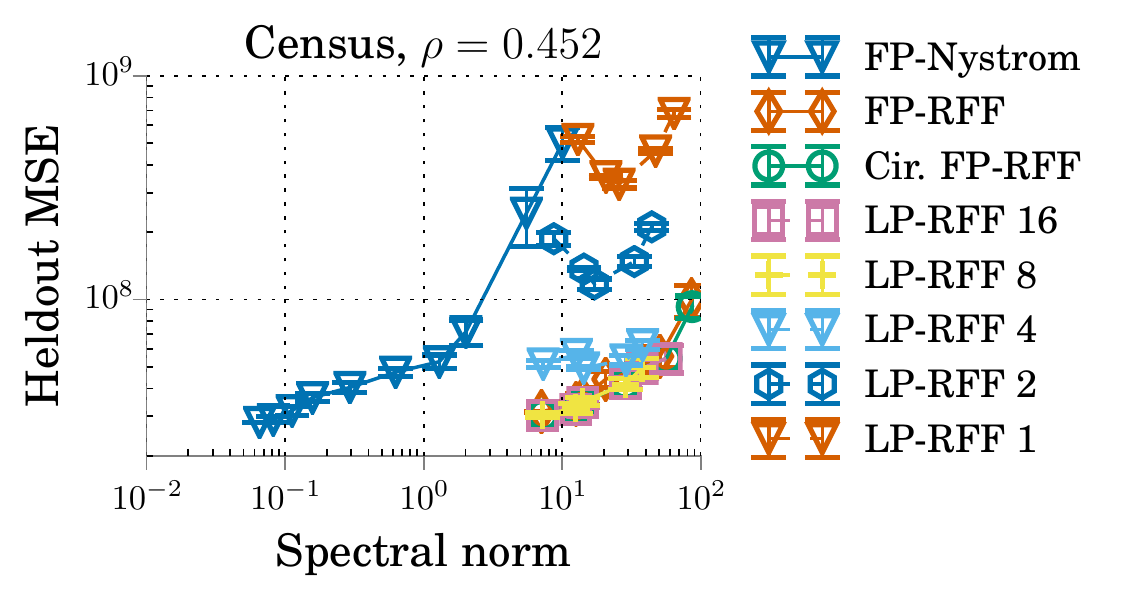} & 
		\includegraphics[height=0.25\linewidth]{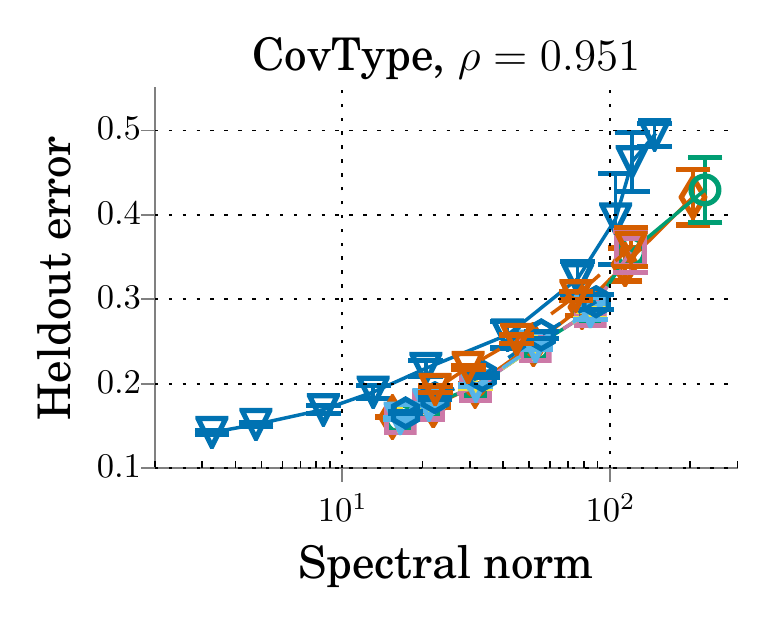}\\ [-0.5em]
		\includegraphics[height=0.25\linewidth]{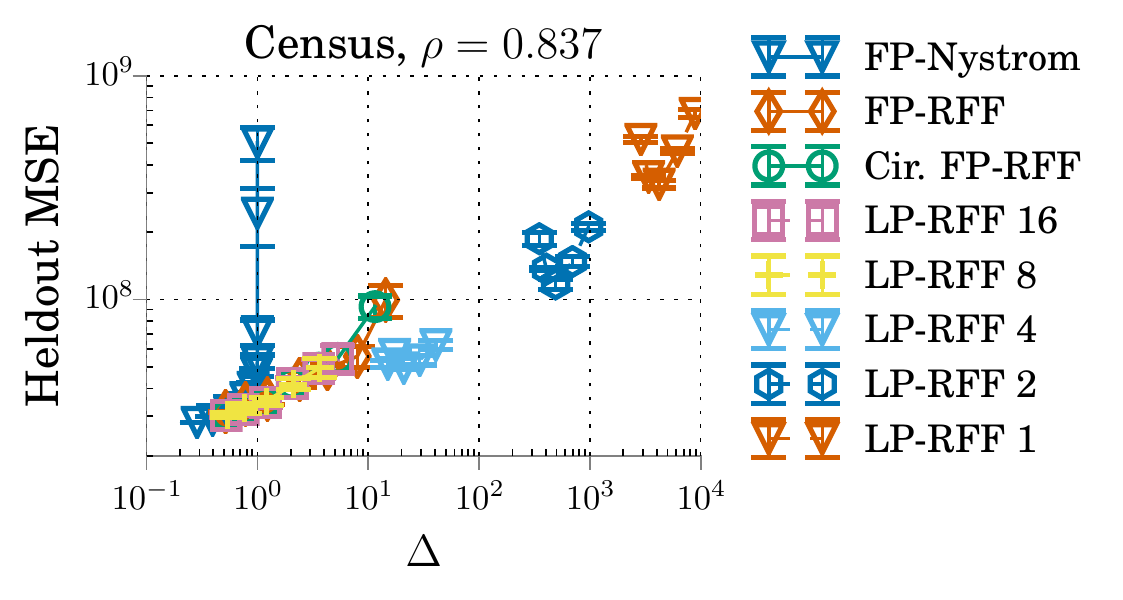} & 
		\includegraphics[height=0.25\linewidth]{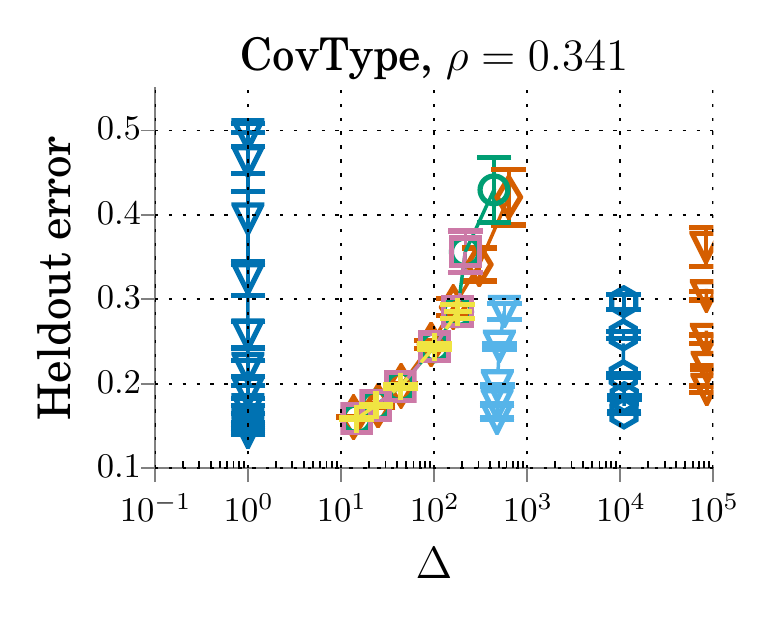}\\ [-0.5em]
		\includegraphics[height=0.25\linewidth]{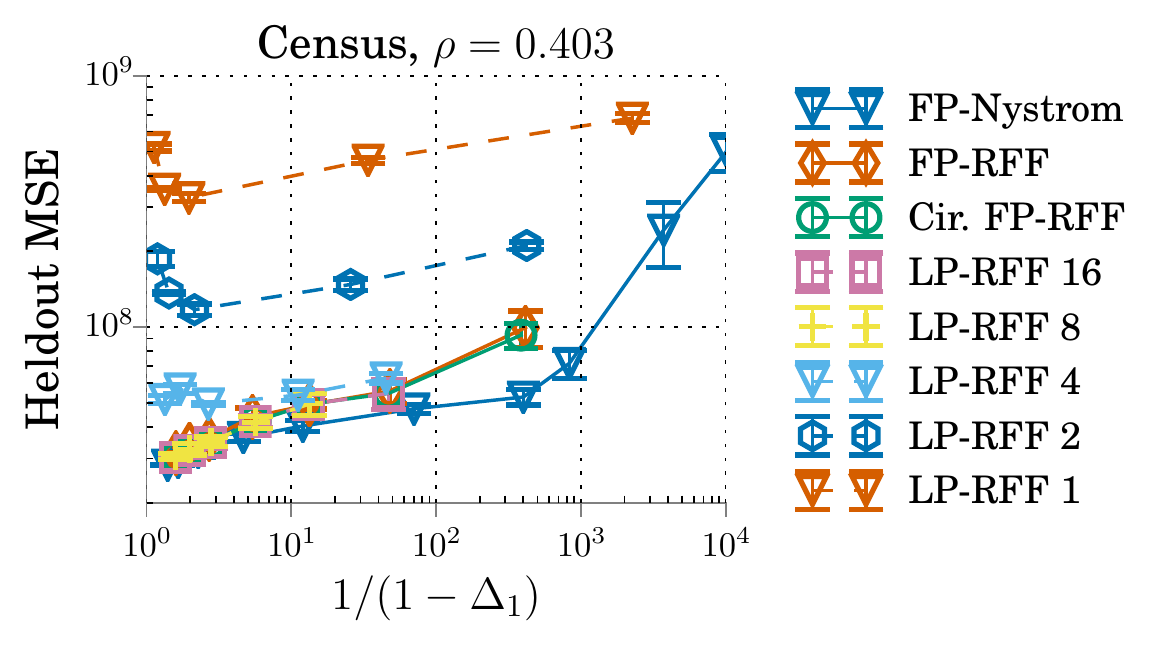} &
		\includegraphics[height=0.25\linewidth]{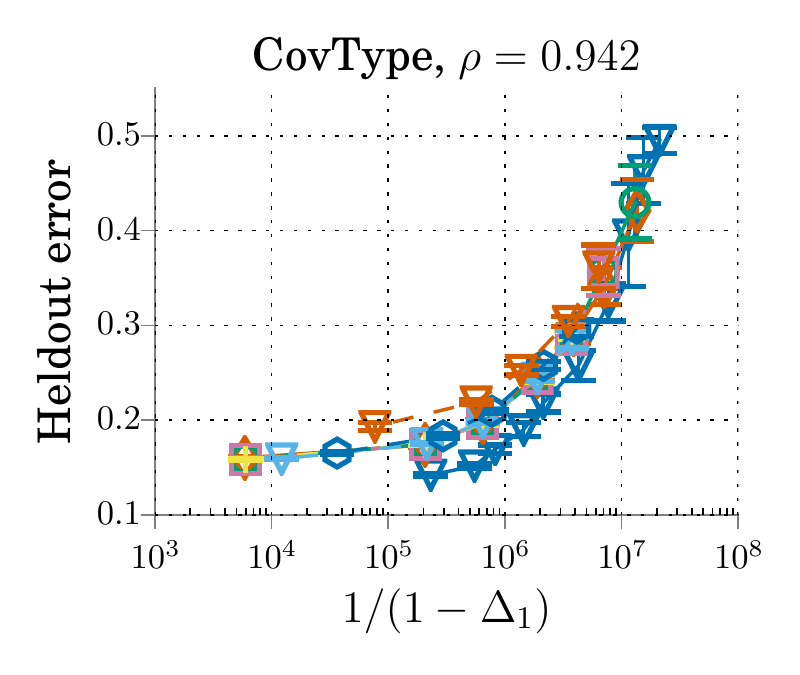} \\ [-0.5em]
		\includegraphics[height=0.25\linewidth]{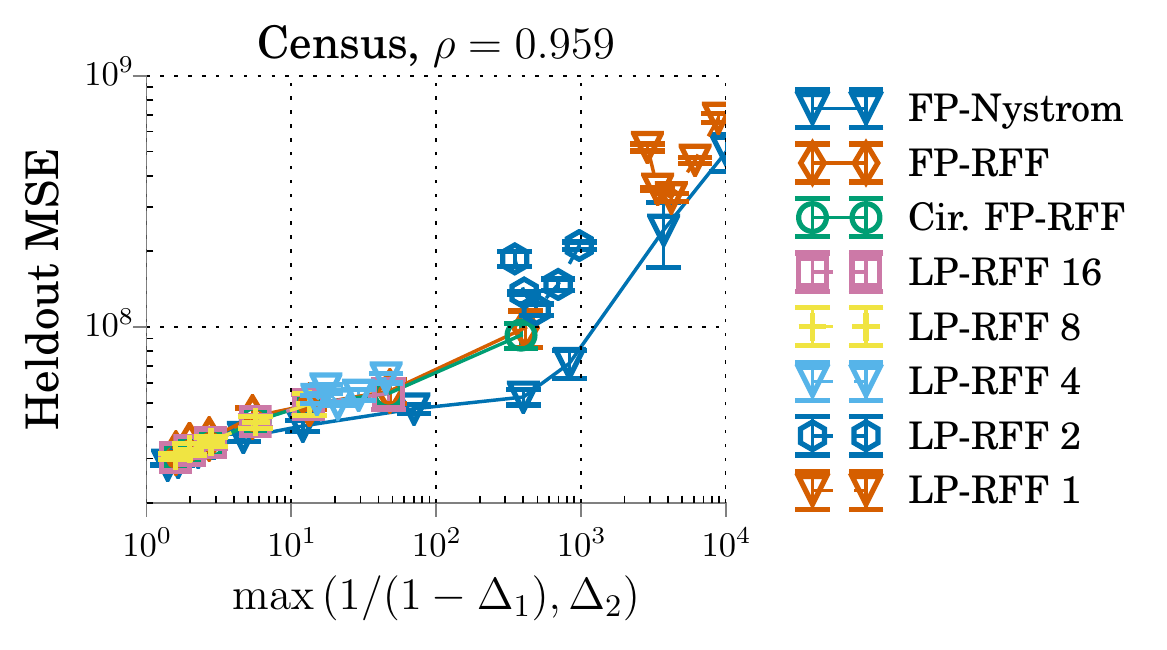}		&
		\includegraphics[height=0.25\linewidth]{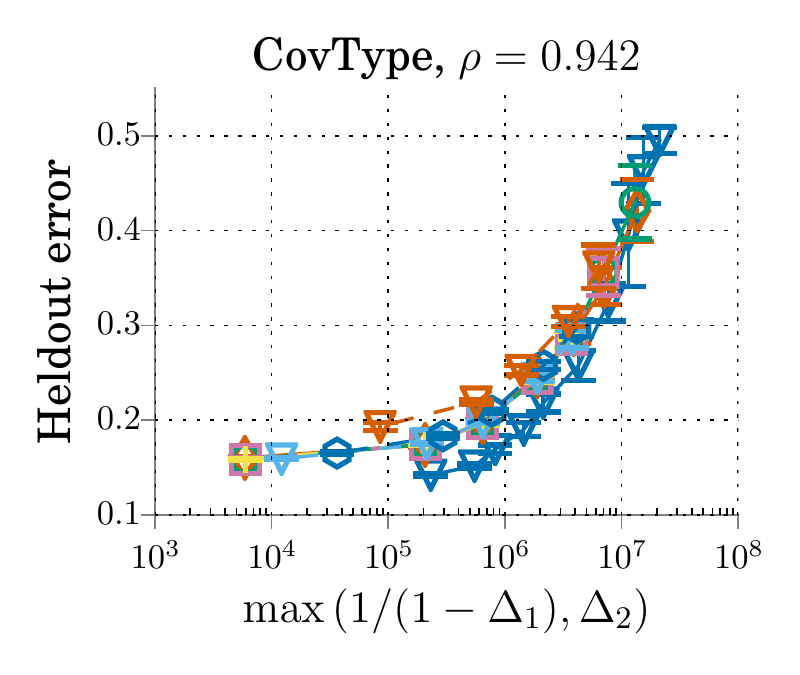} \\ [-0.5em]
	\end{tabular}
\caption{Generalization performance vs. different kernel approximation metrics on Census and CovType.
	The metric $\max\left(1/(1-\Delta_1),\Delta_2 \right)$ is able to incorporate the influence of both $\Delta_1$ and $\Delta_2$ on performance for LP-RFFs, aligning well with generalization performance on both Census (Spearman rank correlation coefficient $\rho=0.959$) and CovType ($\rho=0.942$). $1/(1-\Delta_1)$, on the other hand, fails to align well on Census ($\rho=0.403$), but does align on CovType ($\rho=0.942$).
	Note that although we plot the average performance across five random seeds for each experimental setting (error bars indicate standard deviation), when we compute the $\rho$ values we treat each experimental result independently (without averaging).
	}
	\label{fig:delta_max_perf_app}
\end{figure}

\subsection{Other Experimental Results}
\label{app:other_results}
\begin{figure}
	\centering
	\includegraphics[width=0.7\linewidth]{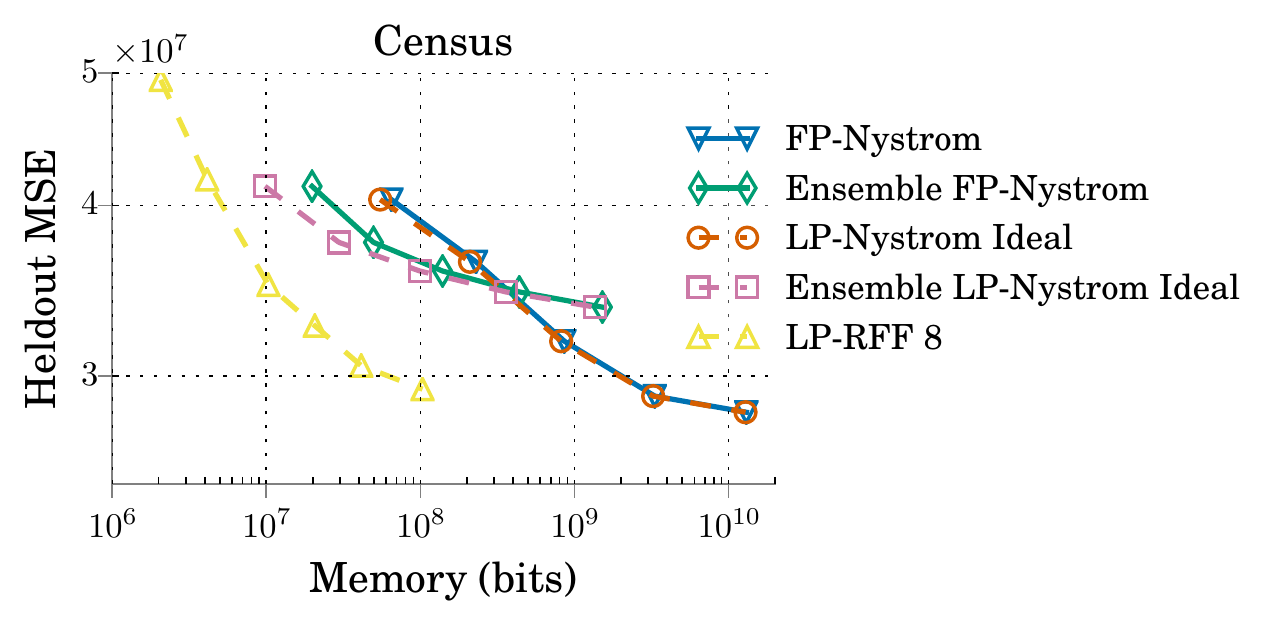}
	\captionof{figure}{We plot the heldout performance (MSE) for the full-precision \Nystrom method, the ensemble \Nystrom method, and 8-bit LP-RFFs.  We also show the best possible performance for the \Nystrom methods, assuming only the kernel approximation features are quantized (denoted ``ideal''); we compute this by plotting the full-precision results but without counting the memory occupied by the features.  The LP-RFF method significantly outperforms the ``ideal'' \Nystrom methods as a function of memory.}
	\label{fig:lpnystrom_ideal}
\end{figure}
\begin{figure}
	\centering
	\begin{tabular}{c c}
		\includegraphics[width=0.35\linewidth]{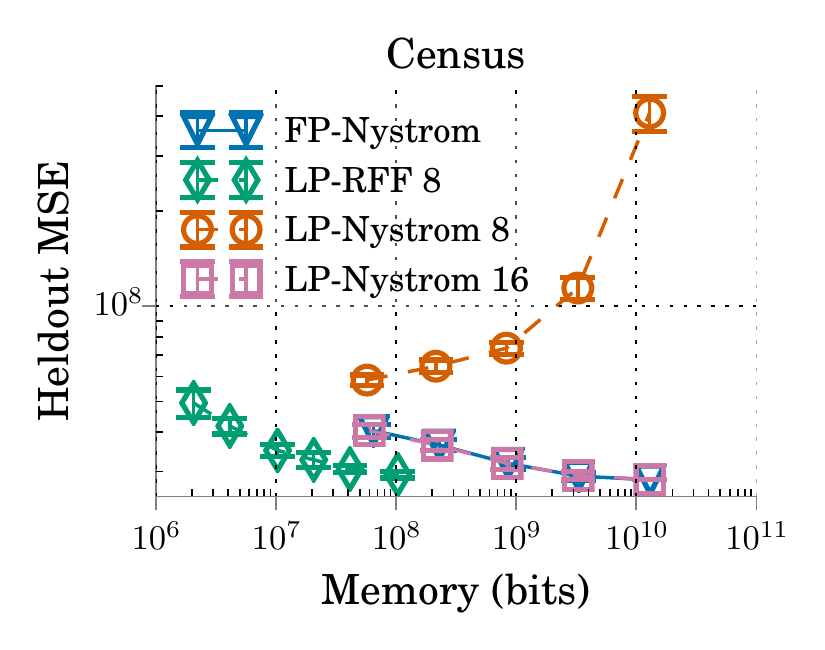} &
		\includegraphics[width=0.35\linewidth]{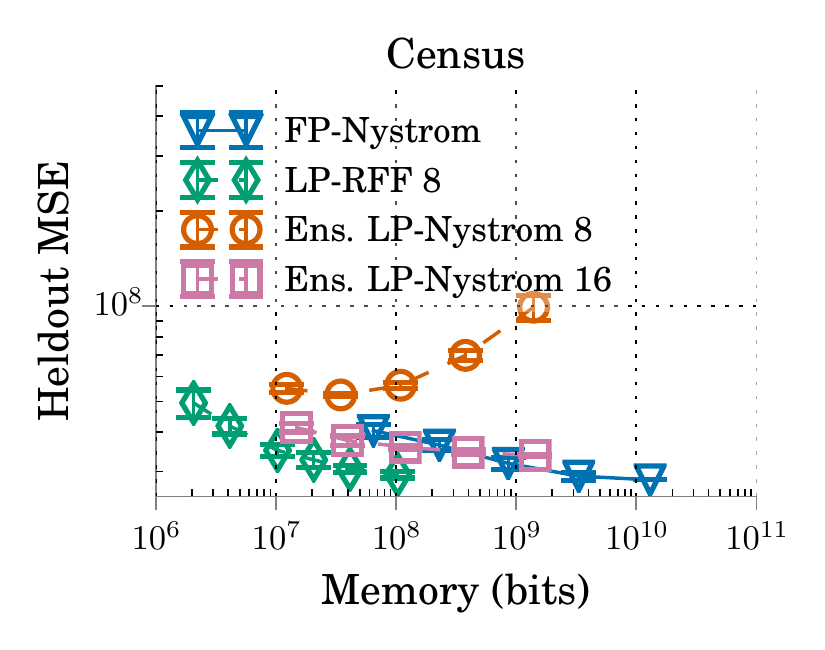} \\
		(a) LP-\Nystrom & (b) Ensemble LP-\Nystrom
	\end{tabular}
	\caption{The generalization performance of low-precision \Nystrom and low-precision ensemble \NystromNS, using a uniform quantization scheme.}
	\label{fig:lpnystrom_unif_quant}
\end{figure}

\subsubsection{Low-Precision \NystromNS}

We encounter two main obstacles to attaining strong performance with the \Nystrom method under a memory budget:
(1) For the standard \Nystrom method, because of the large projection matrix ($32m^2$ space), there are inherent limits to the memory savings attainable by only quantizing the features, \textit{regardless of the quantization scheme used}.
(2) While the ensemble \Nystrom method \citep{ensemble09} is a well-known method which can dramatically reduce the space occupied by the the \Nystrom projection matrix, it does not attain meaningful gains in performance under a memory budget.
To demonstrate the first issue empirically, we plot in Figure~\ref{fig:lpnystrom_ideal} the performance of the full-precision \Nystrom methods \textit{without} counting the memory from the feature mini-batches (denoted ``ideal'').
This is the best possible performance under any quantization scheme for these \Nystrom methods, assuming only the features are quantized.
LP-RFFs still outperform these ``ideal'' methods for fixed memory budgets.
To demonstrate the second issue, in Figure~\ref{fig:lpnystrom_ideal} we also plot the generalization performance of the ensemble method vs.\ the standard \Nystrom method, as a function of memory.
We can see that the ensemble method does not attain meaningful gains over the standard \Nystrom method, as a function of memory.

Lastly, we mention that quantizing \Nystrom features is more challenging due to their larger dynamic range.
For \NystromNS, all we know is that each feature is between $-1$ and $1$, whereas for RFFs, each feature is between $-\sqrt{2/m}$ and $\sqrt{2/m}$.
In Figure~\ref{fig:lpnystrom_unif_quant} we plot our results quantizing \Nystrom features with the following simple scheme:
for each feature, we find the maximum and minimum value on the training set, and then uniformly quantize this interval using $b$ bits.
We observe that performance degrades significantly with less than or equal to 8 bits compared to full-precision \NystromNS.
Although using the ensemble method reduces the dynamic range by a factor of $\sqrt{r}$ with $r$ blocks (we use $r=10$, a common setting in \citep{kumar12}),
and also saves space on the projection matrix, these strengths do not result in significantly better performance for fixed memory budgets.

\subsubsection{Low-Precision Training for LP-RFFs}
\label{sec:halp}
As discussed in Section~\ref{subsec:memory_utils}, there are a number of ways to reduce the memory occupied by the model parameters, including 
(1) using a low-rank decomposition of the parameter matrix \citep{sainath2013low}, 
(2) using structured matrices \citep{structured15}, and 
(3) using low-precision \citep{halp18}.
Though these methods are orthogonal to our LP-RFF method, we now present results using a low-precision parameterization of the model, and show that we can attain similar performance to full-precision training.
We use a training algorithm called LM-HALP (linear model high-accuracy low-precision) \citep{halp18}.
By parameterizing the model in low precision, this algorithm eliminates the need of casting the LP-RFFs back into full precision in order to multiply them with the model parameters.
This approach also allows for these matrix multiplications to be implemented using fast low-precision matrix operations, and reduces the memory occupied by the model during training.

LM-HALP is based on the stochastic variance-reduced gradient (SVRG) algorithm.
The model is parameterized using a low-precision fixed-point representation during training.
In LM-HALP, all of the matrix multiplications involved in the stochastic model updates are done using low-precision fixed-point operations;
however, the periodic computation of the full gradient is calculated in full precision (this is embarrassingly parallelizable).
Importantly, even though most of training is done in low precision, the final model returned by this training algorithm is a full-precision model.
We can further simplify the LM-HALP algorithm by replacing the SVRG updates with SGD updates, thus eliminating the need for calculating and storing the full gradient.
In Figure~\ref{fig:halp} we present our results using LM-HALP on TIMIT, our largest and most challenging dataset;
we use 8-bit LM-HALP (SVRG and SGD) on top of 8-bit LP-RFFs, and compare to full-precision SGD training. For LM-HALP based training, we perform the bit centering and rescaling operation after each training epoch. Under this setting, we show that when the number of features is at least $\num[group-separator={,}]{10000}$, the performance of both versions of HALP closely matches that of full-precision training.

\begin{figure}
	\centering
		\includegraphics[width=0.4\linewidth]{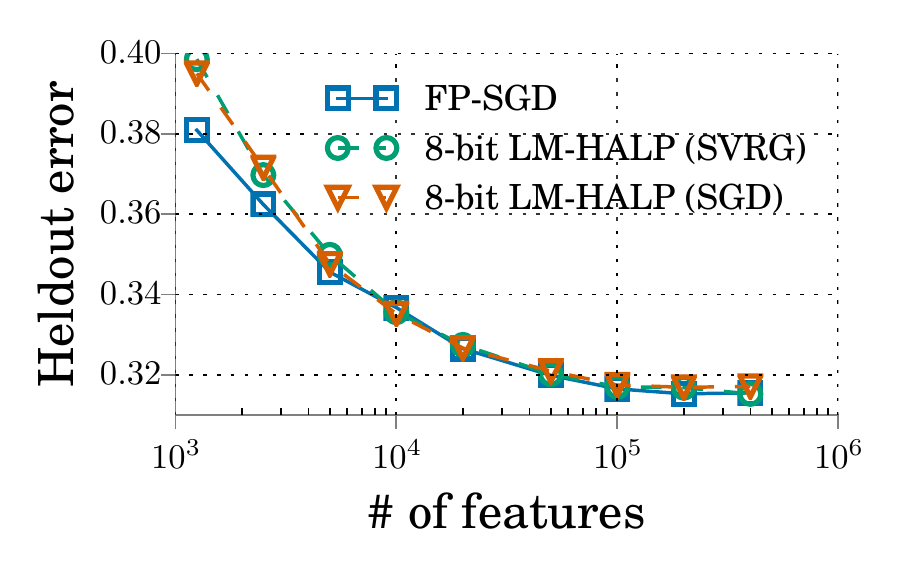}
		\vspace{-1em}
\captionof{figure}{Low-precision vs. full-precision training algorithms on TIMIT using 8-bit LP-RFFs.}	
\label{fig:halp}
\end{figure}

\subsubsection{Double Sampling}
We perform some initial experiments using the double sampling method of \citet{zipml17}.
In particular, we use a different random quantization of the LP-RFFs on the ``forward pass'' of our algorithm than in the ``backward pass.''
In our initial experiments with double sampling, as shown in Figure~\ref{fig:double_sampling}, we did not observe noticeable improvements in performance.
These experiments were on the YearPred dataset with the Gaussian kernel.
We leave a more extended investigation of these gradient bias reduction methods for future work.

\begin{figure}
	\centering
	\includegraphics[width=0.6\linewidth]{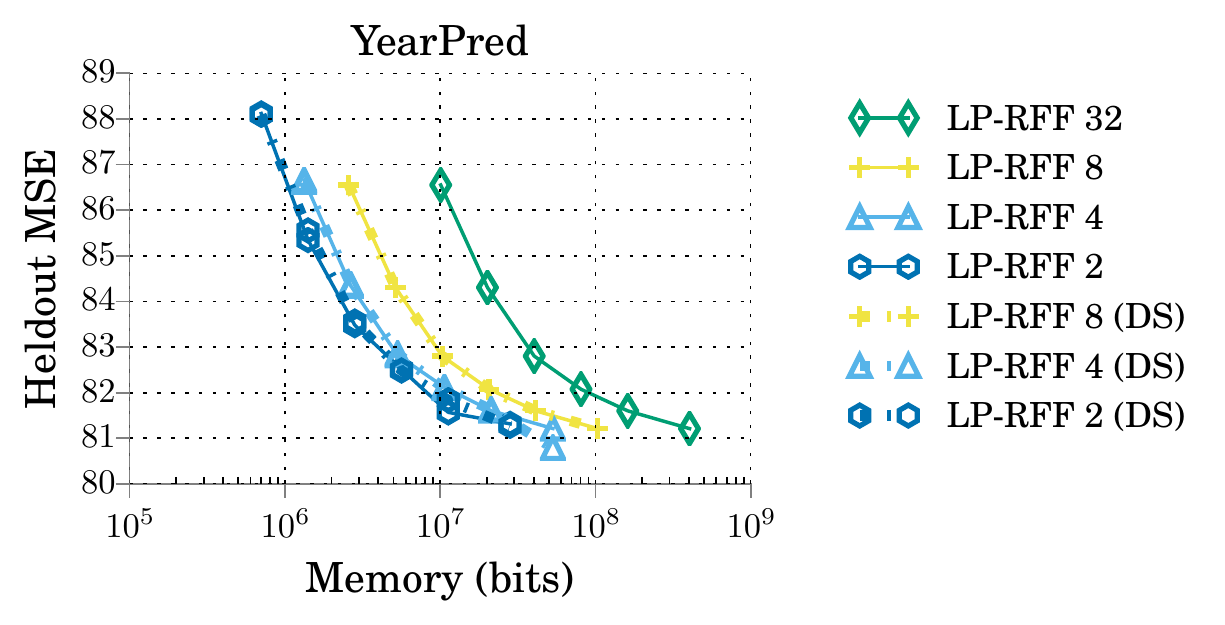}
	\captionof{figure}{Comparison of LP-RFFs with and without double sampling on the YearPred dataset.}
	\label{fig:double_sampling}
\end{figure}

\section{\uppercase{Extended related work}}
\label{sec:relwork_full}
\ifcamera
\paragraph{Generalization Performance of Kernel Approximation Methods}
From a theoretical perspective, there has been a lot of work analyzing the generalization performance of kernel approximation methods \citep{rahimi08kitchen,cortes10,sutherland15,rudi17,avron17,li18}.
The work most relevant to ours is that of \citet{avron17}, which defines $\Delta$-spectral approximation and bounds the generalization performance of kernel approximation methods in terms of $\Delta$.
This approach differs from works which evaluate kernel approximation methods in terms of the Frobenius or spectral norms of their kernel approximation matrices \citep{cortes10,gittens16,qmc,sutherland15,yu16,dao17}.
Our work shows the promise of \citeauthor{avron17}'s approach, and builds upon it.

\paragraph{Large-Scale Kernel Experiments}
On the topic of scaling kernel methods to large datasets, there have been a few notable recent papers.
\citet{block16} propose a distributed block coordinate descent method for solving large-scale least squares problems using the \Nystrom method or RFFs.
The recent work of \citet{may2017} uses a single GPU to train large RFF models on speech recognition datasets, %
showing comparable performance to fully-connected deep neural networks.
That work was limited by the number of features that could fit on a single GPU,
and thus our proposed method could help scale these results.

\else
Our work is not the first to attempt to minimize the memory footprint of kernel approximation methods.
For RFFs, there has been work on using structured random projections \citep{fastfood,yu15,sphereRKS}, and feature selection \citep{sparseRKS, may2016} to reduce memory utilization.
Our work is orthogonal to these, because LP-RFFs can be used in conjunction with both.
For \NystromNS, there has been extensive work on improving the choice of landmark points, and reducing the memory footprint in other ways \citep{kmeans08,ensemble09,fastpred14,meka14,musco17}.
In our work, we study the effect of \textit{quantization} on generalization performance for RFFs under memory constraints.
We focus on RFFs because the feature generation component (landmark points and projection matrix) for \Nystrom is very memory-intensive, and thus limits the relative memory savings attainable via feature quantization.
For our initial experiments quantizing \Nystrom features, see Appendix~\ref{app:other_results}.

From a theoretical perspective, there has been a lot of work analyzing the generalization performance of kernel approximation methods \citep{rahimi08kitchen,cortes10,sutherland15,rudi17,avron17}.
The work most relevant to ours is that of \citet{avron17}, which defines $\Delta$-spectral approximation and bounds the generalization performance of kernel approximation methods in terms of $\Delta$.
This approach differs from works which evaluate kernel approximation methods in terms of the Frobenius or spectral norms of their kernel approximation matrices \citep{cortes10,gittens16,qmc,sutherland15,yu16,dao17}.
Our work shows the promise of \citeauthor{avron17}'s approach, and builds upon it.

On the topic of scaling kernel methods to large datasets, there have been a few notable recent papers.
\citet{block16} propose a distributed block coordinate descent method for solving large-scale least squares problems using the \Nystrom method or RFFs.
The recent work of \citet{may2017} uses a single GPU to train large RFF models on speech recognition datasets, %
showing comparable performance to fully-connected deep neural networks.
That work was limited by the number of features that could fit on a single GPU,
and thus our proposed method could help scale these results.

There has been much recent interest in the topic of low-precision for accelerating training and/or inference of machine learning models, as well as for model compression \citep{gupta15,hogwild15,hubara16,halp18,desa17,han15}. 
There have also been many advances in hardware support for low-precision \citep{tpu17,brainwave17}.
These improvements in hardware could dramatically speed up the training time of our method.

This work is inspired by the experiments comparing \Nystrom and RFFs under a memory budget in the PhD dissertation of one of the first authors \citep{maythesis}.
The current work provides a principled understanding of these prior results by showing that the relative performance of these methods can largely be explained in terms of our $(\Delta_1,\Delta_2)$ measure of kernel approximation error.
Based on this understanding, we propose LP-RFFs as a way of attaining improved generalization performance under a memory budget.
\fi

\end{document}